\documentclass[11pt]{article}

\usepackage[final]{acl}
\usepackage{tabularx}
\usepackage{textgreek}
\usepackage{capt-of}
\usepackage{booktabs}
\usepackage{multirow}
\usepackage{amsthm}

\usepackage{dblfloatfix}
\usepackage{graphicx}
\usepackage{booktabs}
\usepackage{siunitx}
\usepackage{float}
\usepackage{makecell}
\usepackage{xstring} 
\usepackage{adjustbox}
\usepackage{cuted}
\usepackage{tabularx}   % 用于自适应宽度的表格
\usepackage{amsmath}

\usepackage{cleveref}
\usepackage{subcaption}   % 用于子图
\usepackage{caption}      % 可选：用于调整主图标题样式
\usepackage[table,xcdraw]{xcolor}
\usepackage{colortbl}
\usepackage{array}
\usepackage{times}
\usepackage{latexsym}
\usepackage{amsfonts} 
\usepackage[T1]{fontenc}
\usepackage[utf8]{inputenc}
\usepackage{enumitem}
\usepackage{microtype}

\usepackage{inconsolata}
\usepackage{amssymb}
\title{Parallelism and Generation Order in Masked Diffusion Language Models: Limits Today, Potential Tomorrow}

\author{
\textbf{Yangyang Zhong}$^{1,2}$, \textbf{Yanmei Gu}$^{2}$, \textbf{Zhengqing Zang}$^{1,2}$, \textbf{Xiaomeng Li}$^{2}$, \\
\textbf{Yuqi Ding}$^{2,3}$, \textbf{Xibei Jia}$^{1,2}$, \textbf{Yuting Shen}$^{2,4}$, \textbf{Zhenzhong Lan}$^{2,5}$, \textbf{Liwang Zhu}$^{2}$, \\
\textbf{Weiping Liu}$^{2}$, \textbf{Junlin Zhou}$^{2}$, \textbf{Haisheng Liu}$^{2}$, \textbf{Zhong Xin Yu}$^{2}$, \textbf{Pengxin Luo}$^{1}$, \\
\textbf{Donglian Qi}$^{1}$, \textbf{Yunfeng Yan}$^{1\dagger}$, \textbf{Junbo Zhao}$^{2,1\dagger}$ \\[4pt]
$^{1}$Zhejiang University, $^{2}$Ant Group, $^{4}$Shanghai Jiao Tong University \\
$^{3}$University of Chinese Academy of Social Sciences, $^{5}$Westlake University \\[2pt]
\texttt{\{yyff, j.zhao\}@zju.edu.cn} 
}

\begin{document}
\maketitle
\begin{abstract}
Masked Diffusion Language Models (MDLMs) promise parallel token generation and arbitrary-order decoding, yet it remains unclear to what extent current models truly realize these capabilities. We characterize MDLM behavior along two dimensions—parallelism strength and generation order—using {Average Finalization Parallelism (AFP)} and Kendall's $\tau$. We evaluate eight mainstream MDLMs (up to 100B parameters) on 58 benchmarks spanning knowledge, reasoning, and programming. The results show that MDLMs still lag behind comparably sized autoregressive models, mainly because parallel probabilistic modeling weakens inter-token dependencies. Meanwhile, MDLMs exhibit adaptive decoding behavior: their parallelism and generation order vary significantly with the task domain, the stage of reasoning, and whether the output is correct. On tasks that require ``backward information'' (e.g., Sudoku), MDLMs adopt a solution order that tends to fill easier Sudoku blanks first, highlighting their advantages. Finally, we provide theoretical motivation and design insights supporting a \emph{Generate-then-Edit} paradigm, which mitigates dependency loss while retaining the efficiency of parallel decoding.
\end{abstract}

\section{Introduction}
Autoregressive (AR) language models dominate modern natural language processing (NLP) due to their strong likelihood-based training objectives and reliable left-to-right decoding. 
However, the strictly sequential nature of AR decoding entails two fundamental limitations: 
(i) high inference latency and constrained generation throughput; and 
(ii) for tasks requiring global constraints or non-monotonic planning, a fixed chronological order may not constitute the most natural solution path. 
As an important branch of discrete diffusion models, masked diffusion language models (MDLMs) address these gaps by iteratively denoising masked sequences, enabling parallel token prediction within a single step and, in principle, allowing for more flexible, non-monotonic generation orders \cite{li2025survey,sahoo2024simple}. 
For KV-cache reuse and efficiency, most practical systems adopt a block-diffusion architecture—executing autoregressively across blocks while employing diffusion within each block \cite{arriola2025block}. 
Recent high-performing MDLMs (e.g., LLaDA~2.0\cite{llada2.0}, Trado\cite{tracerl}, DiRL\cite{dirl}, SDAR\cite{sdar}, and OpenPangu-Diffusion\cite{openpangu}) are typically initialized from the architecture and weights of strong AR models and then transferred to the diffusion paradigm by modifying the decoding mechanism, achieving competitive results.

Despite rapid progress, it remains unclear whether MDLMs truly exploit their latent non-monotonic potential during inference, or how they navigate the trade-off between parallelism and generation quality. 
Existing studies primarily focus on inference strategies or parameter tuning, reporting their impact on a limited set of benchmarks \citep{Fast-dllm,ye2024beyond,Parallelbench,feng2025theoretical}, yet a systematic characterization of the model's decoding trajectories and dynamic behaviors is missing. 
While some recent works have begun examining decoding order \citep{chen2025beyond,diffucoder}, they often fail to explicitly decouple ``generation order'' from ``parallelism intensity''; for instance, \citep{diffucoder} conflates order metrics with variations in parallelism. 
Furthermore, speed evaluations frequently rely on hardware-dependent metrics such as throughput (tokens/s), which are sensitive to implementation details and hinder fair cross-model comparisons. 
Consequently, the community still lacks unified, interpretable, and hardware-agnostic metrics to disentangle the parallelism and non-monotonicity of MDLMs, limiting our understanding of their quality degradation and potential strengths.

Motivated by these gaps, we treat \textbf{parallelism} and \textbf{generation order} as two defining degrees of freedom of MDLMs and conduct a large-scale reality check via mechanism-level analyses. Our contributions are four-fold: 
\textbf{(1) Large-scale Evaluation:} We unify the deployment of recent MDLMs (e.g., LLaDA~2.0) and strong AR baselines on \textbf{58 benchmarks} spanning Knowledge, Math, Reasoning, and Coding. We identify a consistent accuracy gap and provide a theoretical account through the lens of \textit{parallel factorization}: the conditional-independence approximation in parallel decoding induces an unavoidable lower bound on quality. 
\textbf{(2) Parallelism and Order Metrics:} To operationalize these properties, we propose \textbf{Average Finalization Parallelism (AFP)} and adopt Kendall's $\tau$ to quantify the alignment between token finalization and surface order. Analyzing a 100B-scale MDLM, we reveal an \textbf{adaptive trade-off} where the model accelerates on structure-heavy spans but decelerates at semantic pivots. Notably, correct predictions demonstrate higher parallelism, suggesting a \textit{scaling dividend} where enhanced model capability naturally accelerates inference. 
\textbf{(3) Uncovering Non-monotonic Potential:} Through Sudoku-variant benchmarks, we amplify the advantages of MDLMs in parallel and non-monotonic solving, exhibiting solution paths qualitatively distinct from AR decoding and suggesting significant any-order probabilistic modeling potential. 
\textbf{(4) Mitigating Parallel Factorization Loss:} We provide a theoretical perspective suggesting that a two-stage generate-then-edit paradigm could mitigate the dependency loss induced by parallel factorization.

\begin{figure*}[t] % [t] 表示放置在页面顶部 (Top)
    \centering
    \includegraphics[width=\textwidth]{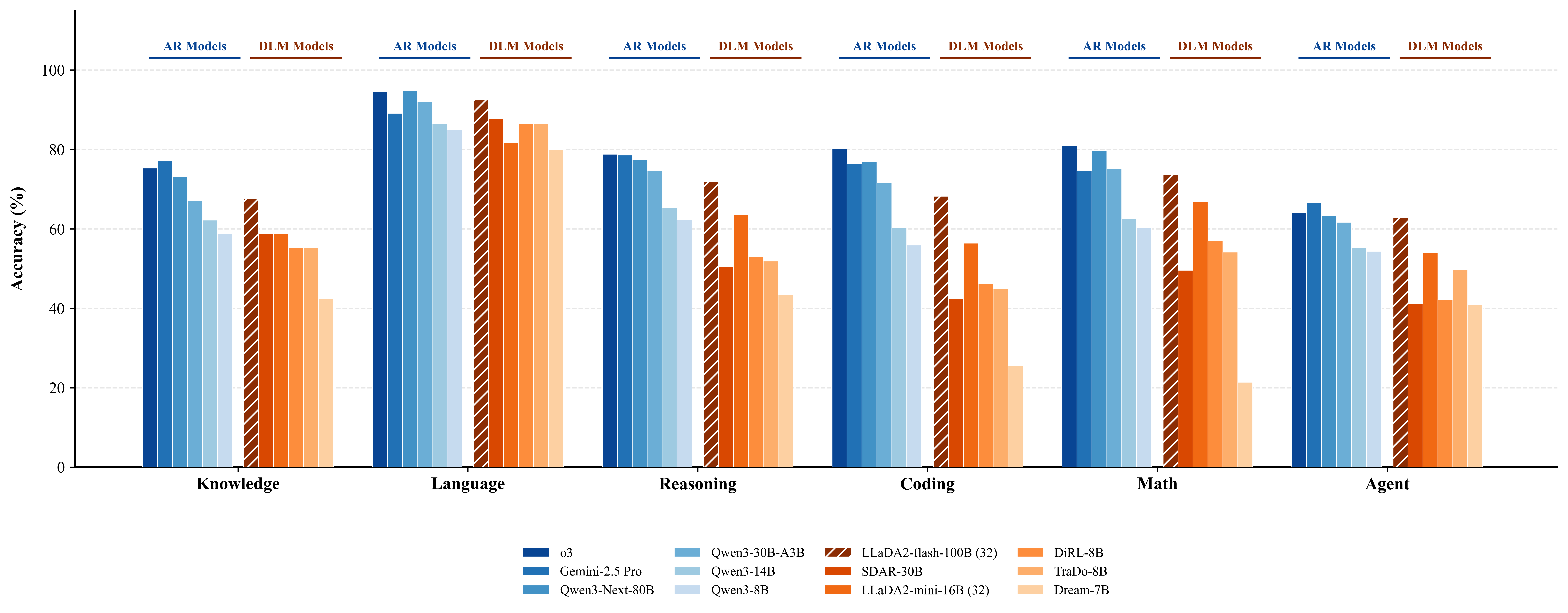}
    \caption{\textbf{Overall performance comparison between AR and DLM models across six evaluation dimensions.} }
    \label{fig:58arvsdlm}
\end{figure*}

\begin{figure}[h]
    \centering
    \includegraphics[width=\columnwidth]{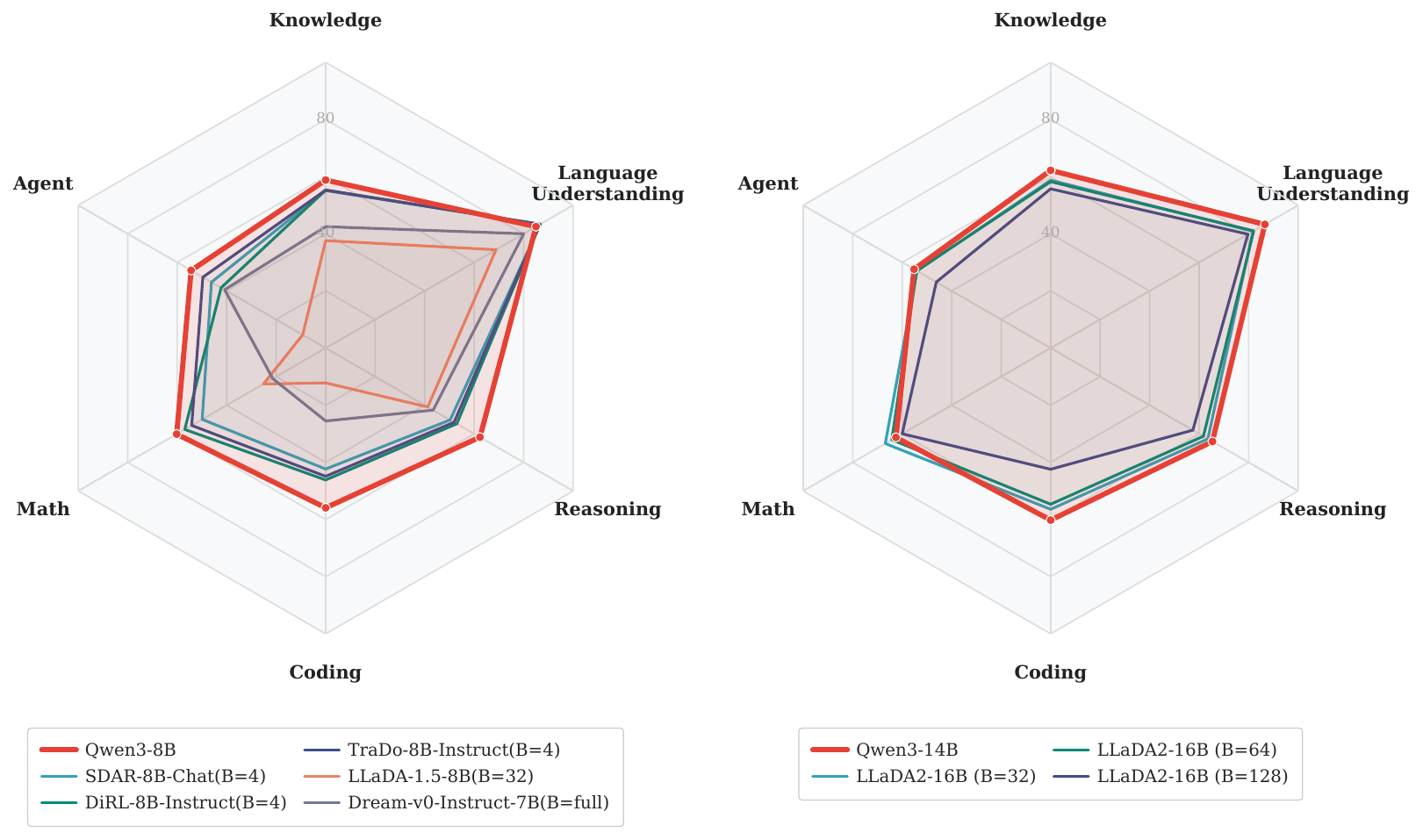} 
    \caption{\textbf{Comparison at the same parameter scale.} Smaller block size $B$ yields superior performance.}
    \label{fig:14v16v8}
\end{figure}

\section{Related Work}

\subsection{Masked Diffusion Language Models}
Masked Diffusion Language Models (MDLMs) are a representative class of discrete diffusion generative models~\citep{lisan1,lisan2}. In the forward process, noise is injected by progressively replacing clean tokens with a special mask symbol \texttt{[MASK]} in a discrete space. The reverse process then iteratively denoises a masked sequence to recover the original text. Recently, MDLMs have been scaled to the regime of large pretrained language models. On the one hand, several works study \emph{native} training of discrete diffusion language models; for example, the LLaDA 1.0 line~\citep{llada,llada1.5} introduces diffusion-oriented masking strategies and decoding schedules. On the other hand, a growing trend adapts strong autoregressive (AR) models to the diffusion paradigm by modifying their causal decoding while largely preserving pretrained weights (e.g., LLaDA-MoE~\cite{lladamoe}, Dream-7B~\cite{Dream}, LLaDA 2.0~\cite{llada2.0}, SDAR~\cite{sdar}, Trado~\cite{tracerl}, DiRL~\cite{dirl}, and OpenPangu-Diffusion~\cite{openpangu}). In our study, we further unify the deployment of these recent SOTA models and conduct an in-depth analysis.

\subsection{Analyses of Diffusion LM Mechanisms}
Prior work has analyzed dLLM decoding and parallelization.\cite{Fast-dllm} link quality drops in parallel decoding to the conditional-independence assumption and propose threshold-/factor-based unmasking, but evaluate on limited leaderboards and use hardware-dependent throughput (tokens/s), which hinders fair cross-setting comparison. \cite{diffucoder} study \emph{decoding order} via AR-ness, yet it depends on a hyperparameter $k$ and is sensitive to decoding setups; moreover, decoding order is often conflated with \emph{parallelism} (tokens per step). \cite{ye2024beyond} highlight dLLM benefits for reasoning/planning, but provide limited instance-level decoding dynamics. In contrast, we study 58 benchmarks, explicitly separate decoding order from parallelism, use robust order statistics, and measure parallelism by average tokens per step to avoid hardware confounds.

\section{Preliminaries}
\label{sec:preliminaryes}
In this section, we review inference in autoregressive and masked diffusion language models.

\vspace{1mm}
\noindent\textbf{Auto-regressive Modeling.} Let $x_{1:T} = (x_1,\dots,x_T)$ be a token sequence. Autoregressive language models factorize its probability as
\begin{equation}
    p_{\mathrm{AR}}(x_{1:T})
    = \prod_{t=1}^{T} p_\theta(x_t \mid x_{<t}),
    \label{eq:ar}
\end{equation}
where $x_{<t} = (x_1,\dots,x_{t-1})$ and $\theta$ are model parameters.
Inference is strictly left-to-right: at each step $t$, the model predicts $x_t$ conditioned on the full prefix $x_{<t}$.

\vspace{1mm}
\noindent\textbf{Masked-diffusion Inference.} This paradigm underpins recent state-of-the-art models such as LLaDA \cite{llada2.0,llada,llada1.5,lladamoe} and Dream-7b \cite{Dream}. Let $\tilde{x}\sim q(\cdot\mid x)$ denote a masked variant of $x$, where positions in a masked index set
$M(\tilde{x})=\{i:\tilde{x}_i=[m]\}$ are replaced by $[m]$.
In one parallel denoising iteration, we choose an update set $K\subseteq M(\tilde{x})$ and predict all
tokens in $K$ in a single forward pass using the (mean-field) approximation
\begin{equation}
\begin{aligned}
p_\theta(x_K \mid \tilde{x})
&\approx
\prod_{i\in K} p_\theta(x_i \mid \tilde{x}), \\
&\text{s.t. } \tilde{x}_i=[m],\ \forall i\in K.
\end{aligned}
\label{eq:md_meanfield}
\end{equation}
where $x_K=\{x_i\}_{i\in K}$.
Eq.~\eqref{eq:md_meanfield} makes explicit that within-step parallelism assumes conditional independence
among $\{x_i\}_{i\in K}$ given $\tilde{x}$, i.e., tokens updated in the same step cannot condition on each other.

To enable variable-length generation and efficient KV caching, these models adopt a blockwise framework that applies the above inference mechanism autoregressively. Specifically, a sequence $x$ is partitioned into $B$ contiguous blocks $\{x^{(1)}, \dots, x^{(B)}\}$. The generation follows a block-level autoregressive factorization: 
\begin{equation}
\log p_\theta(x) = \sum_{b=1}^B \log p_\theta(x^{(b)} \mid x^{(<b)}).
\end{equation}
For each block $b$, the distribution $p_\theta(x^{(b)} \mid x^{(<b)})$ is modeled by the masked diffusion process described in Eq.~\eqref{eq:md_meanfield}, where the denoising is conditioned on the fixed, previously generated blocks $x^{(<b)}$.

% 使用 figure* 环境来跨越双栏
\begin{figure*}[t]
    \centering
    
    % 第一张图
    \begin{subfigure}[b]{0.24\textwidth}
        \centering
        \includegraphics[width=\textwidth]{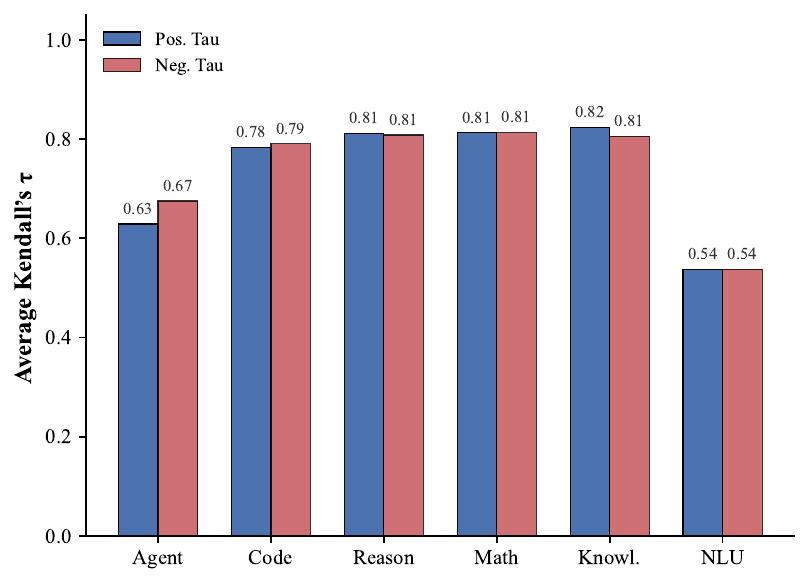} % 替换为你的文件名
        \caption{Kendall’s τ (normal)}
        \label{fig:case1}
    \end{subfigure}
    \hfill % 在子图之间插入弹性间距
    % 第二张图
    \begin{subfigure}[b]{0.24\textwidth}
        \centering
        \includegraphics[width=\textwidth]{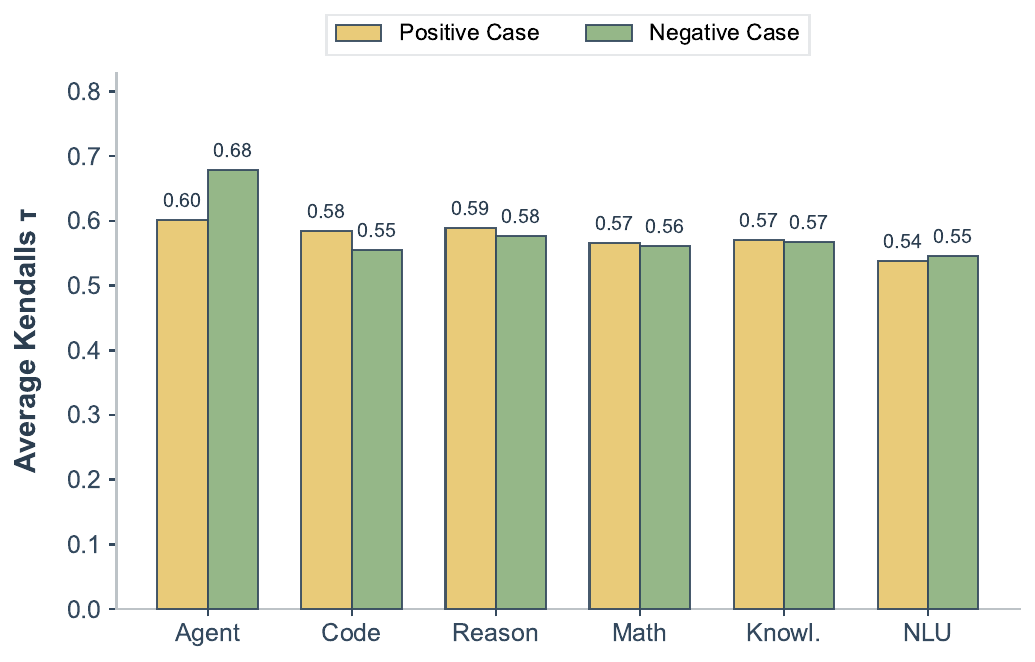}
        \caption{Kendall’s τ (repetitive)}
        \label{fig:case2}
    \end{subfigure}
    \hfill
    % 第三张图
    \begin{subfigure}[b]{0.24\textwidth}
        \centering
        \includegraphics[width=\textwidth]{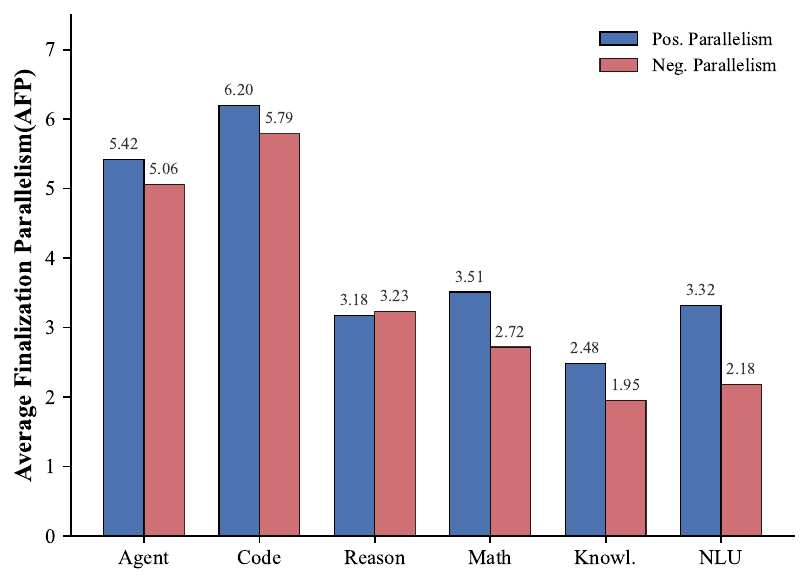}
        \caption{AFP (normal)}
        \label{fig:case3}
    \end{subfigure}
    \hfill
    % 第四张图
    \begin{subfigure}[b]{0.24\textwidth}
        \centering
        \includegraphics[width=\textwidth]{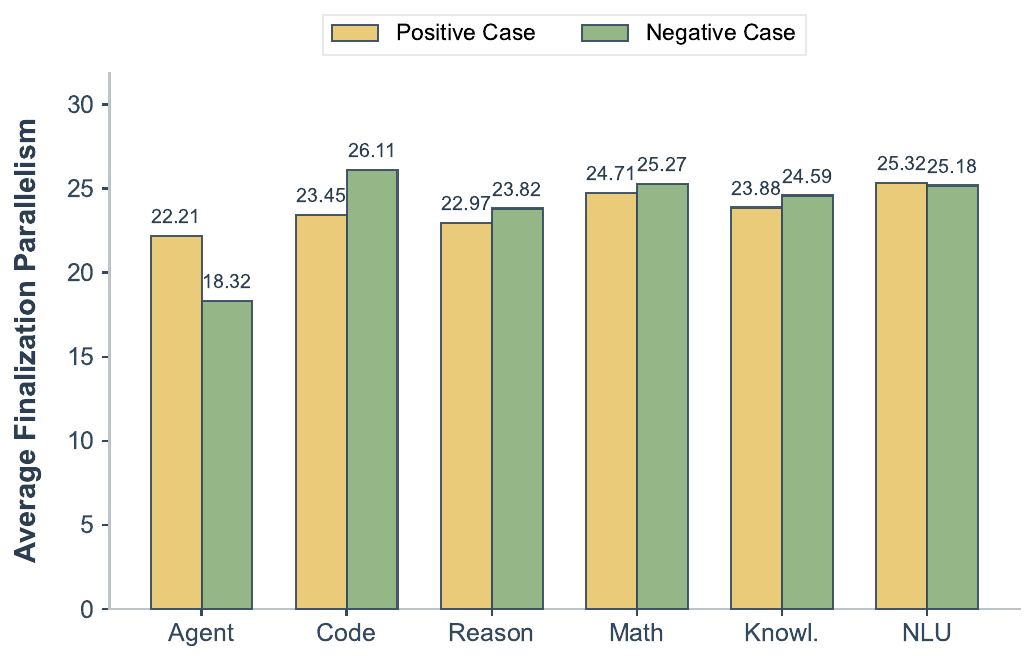}
        \caption{AFP (repetitive)}
        \label{fig:case4}
    \end{subfigure}

    \caption{\textbf{Comparison of intra-chunk parallelism and sequential ordering patterns across different task domains.} Repeating samples are analyzed separately because, despite their limited count, each fills the entire 32k context window and disproportionately impacts overall statistical values.}
    \label{fig:four_plots}
\end{figure*}

\section{Reality Check: MDLMs Still Trail AR LMs at Scale}

Despite the potential of MDLMs to challenge the AR paradigm, existing studies often rely on small-scale datasets, leaving their performance at scale largely unexplored. To provide a transparent and systematic evaluation, we conduct a large-scale benchmark of eight representative MDLMs against state-of-the-art AR models. Using a unified platform and 58 diverse benchmarks, this section serves as a "reality check" to quantify the current progress of the field and identify the remaining performance gaps between these two architectural paradigms.

\subsection{Experimental Protocol: Models, 58 Benchmarks, and Unified Inference}
\noindent\textbf{Models and Benchmarks.} 
We conduct a large-scale evaluation of 8 representative state-of-the-art MDLMs alongside frontier Autoregressive (AR) models, including OpenAI o3 and Gemini-2.5 Pro. To ensure a systematic assessment, we curate a comprehensive suite of \textbf{58 benchmarks} spanning six core dimensions: \textit{Knowledge, Mathematics, Reasoning, Language Understanding, Agentic capabilities, and Coding}. We report mean accuracy over all test instances for each benchmark (one inference per instance). Due to space constraints, exhaustive model specifications and benchmark descriptions are provided in Appendix~\ref{DEP}.

\vspace{1.5mm}
\noindent\textbf{Implementation and Fairness.} 
All experiments are performed on a distributed cluster equipped with \textbf{512 NVIDIA GPUs} using a unified inference pipeline. Unless otherwise specified in the Appendix, each model is evaluated using its optimal default configurations. Notably, as we employ a standardized and more stringent evaluation protocol (e.g., unified prompt templates and rigorous answer-matching logic), the absolute scores for certain models may be lower than those reported in their respective technical reports. However, this controlled environment eliminates confounding variables and ensures a rigorous \textit{apples-to-apples} comparison across disparate model architectures. Detailed implementation settings are deferred to Appendix~\ref{DEP}.

\subsection{Empirical Trade-off: More Parallelism, Lower Accuracy}
Detailed experimental results across 58 benchmarks are reported in Tables \cref{tab:code-benchmarks,tab:knowledge-language-benchmarks,tab:math-benchmarks,tab:reasoning-agent-benchmarks}. Aggregated evaluations across six capability dimensions (as illustrated in Figure \ref{fig:58arvsdlm}) indicate that a performance gap persists between current masked diffusion language models (MDLMs) and their autoregressive (AR) counterparts. A comparative analysis of models at a similar parameter scale in Figure \ref{fig:14v16v8} reveals that MDLMs achieving near-AR performance typically employ restricted block sizes. Notably, to attain peak precision, OpenPangu-7B-Diffusion utilizes a block size of 4 and constrains each step to decode only the token with the highest confidence. These findings collectively suggest an empirical conclusion: in contemporary architectures, excessive decoding parallelism serves as a primary bottleneck that compromises predictive accuracy.

\subsection{Why This Gap Persists: The Parallel Factorization Limit}
\label{sec:Theoretical Support}
The fundamental limitation of Masked Discrete Diffusion Models (MDLMs) stems from the \textbf{conditional independence assumption} \cite{tidar,Parallelbench}. This assumption necessitates factorizing the complex joint distribution of multiple tokens into a product of independent marginals during parallel decoding, thereby introducing significant and unavoidable approximation errors.

\begin{figure*}[h] % [t] 表示放置在页面顶部 (Top)
    \centering
    \includegraphics[width=\textwidth]{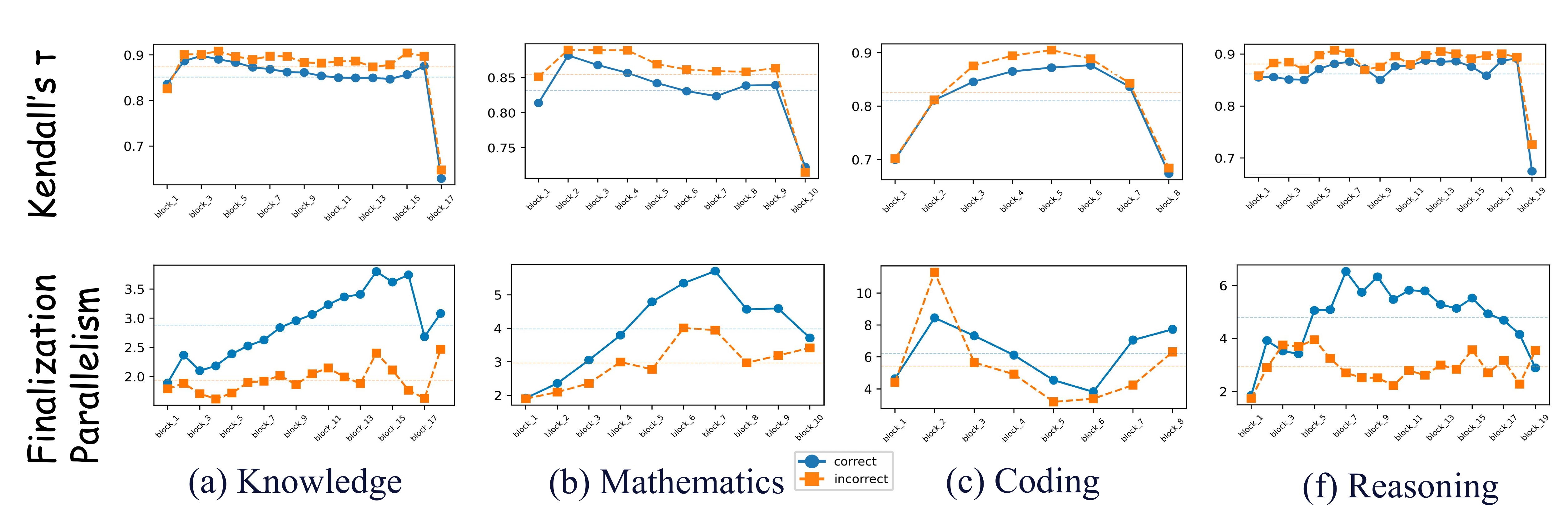} 
\caption{\textbf{Comparison of inter-chunk parallelism and sequential ordering patterns across different task domains.} Refer to Appendix~\ref{app:tau_block} for more groupings of $\tau$, and Appendix~\ref{app:fap_block} for parallel AFP.}
    \label{fig:Trajectories}
\end{figure*}

\noindent\textbf{The Factorization Gap.} Autoregressive (AR) models strictly adhere to the probabilistic chain rule, $P_{AR}(x_{1:n}|c) = \prod_{i=1}^n P(x_i | x_{<i}, c)$, to ensure global coherence. In contrast, an MDLM with a block size $B$ approximates the joint distribution of a token block $Y$ as the product of independent marginal distributions:
\begin{equation}
\label{ea:mdlm}
P_{\text{MDLM}}(Y | X) = \prod_{y_i \in Y} P_{\theta}(y_i | X, S_{<t})
\end{equation}
where $S_{<t}$ denotes tokens finalized in previous steps. When tokens exhibit strong syntactic, semantic, or logical coupling, this independence assumption leads to a pronounced structural bias. Following the framework of \cite{Parallelbench}, even if $P_{\theta}$ achieves the optimal marginal distribution at each position, the KL divergence between this factorized form and the true data distribution $P_{data}(Y|X)$ is lower-bounded by the \textbf{Conditional Total Correlation (CTC)}:

\begin{align}
\min_{\theta} D_{KL}(P_{data} \| P_{\theta})
&\ge C(Y|X) \nonumber\\
&= \sum_{y_i \in Y} H(y_i|X) - H(Y|X)
\end{align}
Here, $C(Y|X)$ quantifies the statistical dependency strength among tokens in $Y$ given the context $X$. Stronger dependencies result in a higher theoretical floor for quality loss, represented by the CTC bound.

\noindent\textbf{Elucidating Empirical Observations.} The aforementioned framework clarifies the performance decay observed in Section 3.2: as the block size $B$ (parallelism granularity) increases, the complexity of inter-token dependencies within $Y$ grows, forcing the model to factorize under tighter intra-group constraints. This leads to a monotonic increase in the CTC-induced error bound. Our analysis suggests that the performance gap between MDLMs and AR models in high-parallelism settings is not merely a consequence of insufficient model capacity, but rather an intrinsic architectural limitation in representing high-dependency structures. This further explains why state-of-the-art methods (e.g., Trado, SDAR, openPangu~\cite{openpangu}) often adopt highly conservative block sizes; they effectively sacrifice their defining architectural advantage—parallelism—to circumvent the theoretical precision ceiling imposed by independent factorization.

\section{Decoding Dynamics of MDLMs: Disentangling Parallelism and Generation Order}
\label{Decoding Dynamics of MDLMs}

While the preceding analysis identifies the current performance {limits} of MDLMs, we argue that their latent {potential} resides in the unique flexibility of their decoding trajectories. 
To characterize this behavior, we shift our perspective from macro-scale accuracy to a {mechanistic deconstruction} of the two defining dimensions in our title: {Parallelism} and {Generation Order}. 
Focusing on the state-of-the-art 100B-parameter LLaDA2-flash, we operationalize these properties via Average Finalization Parallelism (AFP, Eq.~\ref{eq:afp}) and Kendall’s $\tau$ (Eq.~\ref{eq:tau}) to conduct an in-depth analysis of the {decoding dynamics} of frontier diffusion models across diverse benchmarks.

\begin{figure*}[t]
    \centering
    % 第一行：3 张子图
    \begin{subfigure}[b]{0.3\textwidth}
        \centering
        \includegraphics[width=\linewidth]{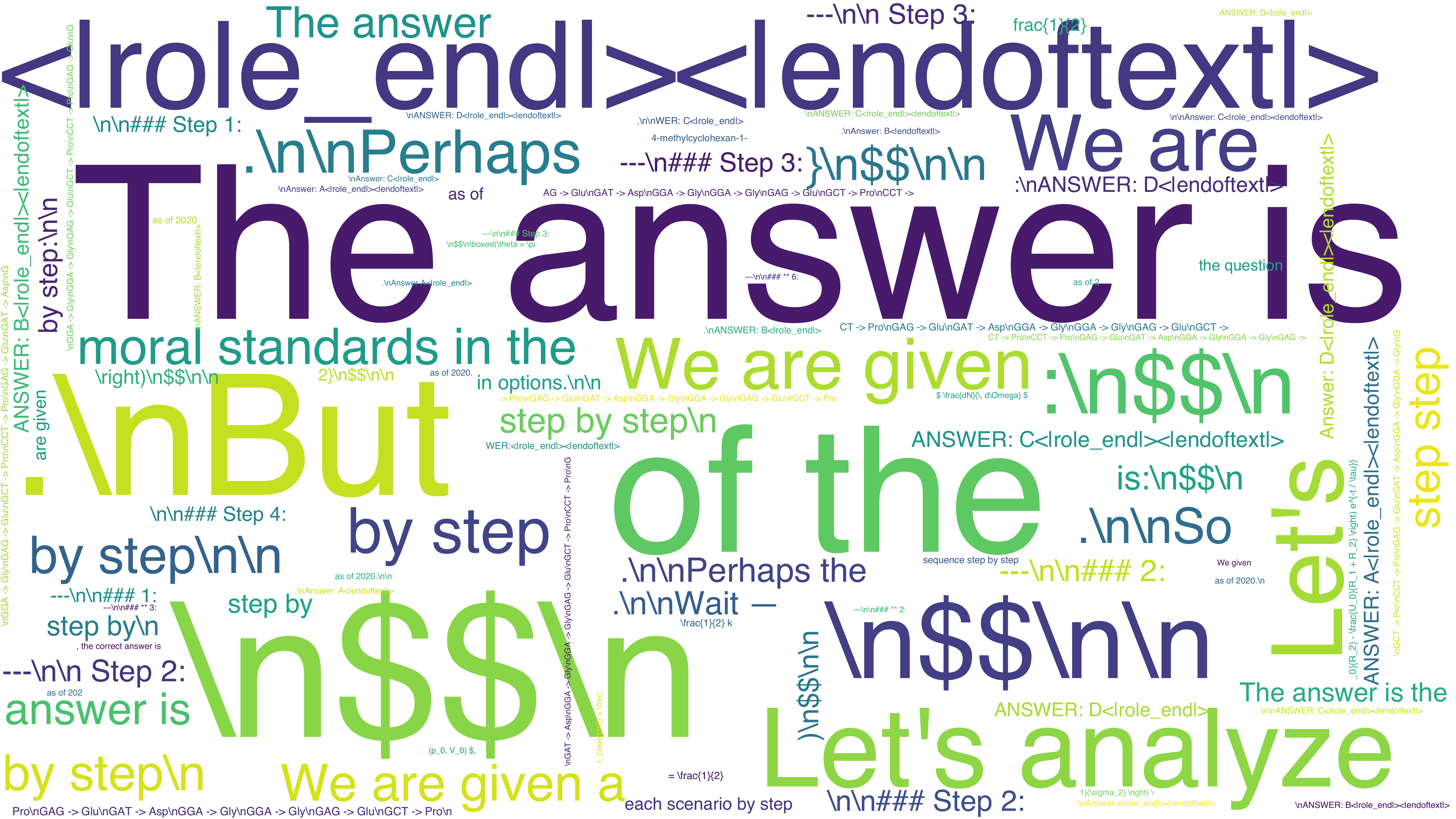} % 替换为你的图片路径
        \caption{Knowledge}
        \label{fig:sub1}
    \end{subfigure}
    \hfill % 水平填充，使子图分散
    \begin{subfigure}[b]{0.3\textwidth}
        \centering
        \includegraphics[width=\linewidth]{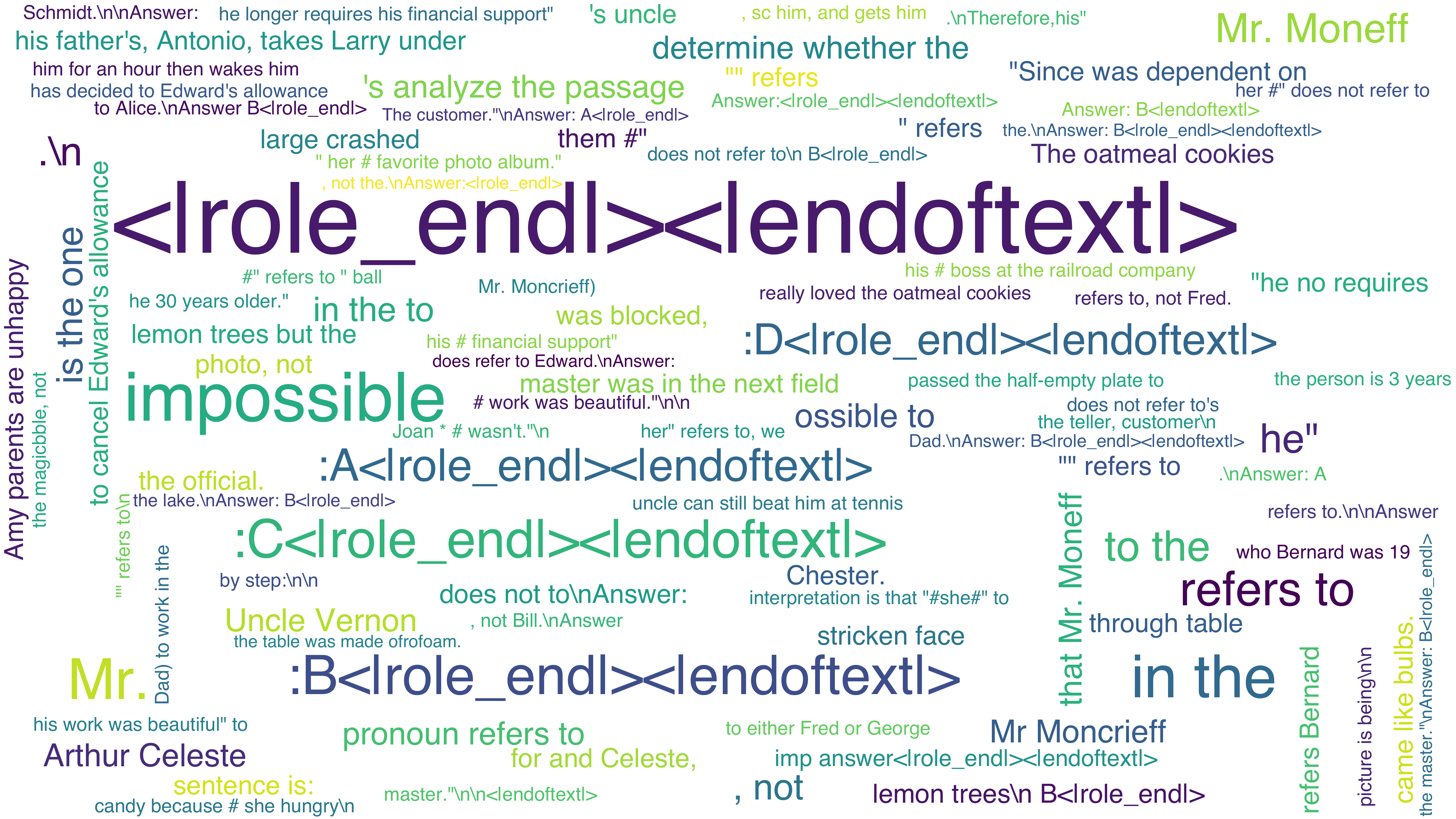}
        \caption{Natural Language Understand}
        \label{fig:sub2}
    \end{subfigure}
    \hfill
    \begin{subfigure}[b]{0.3\textwidth}
        \centering
        \includegraphics[width=\linewidth]{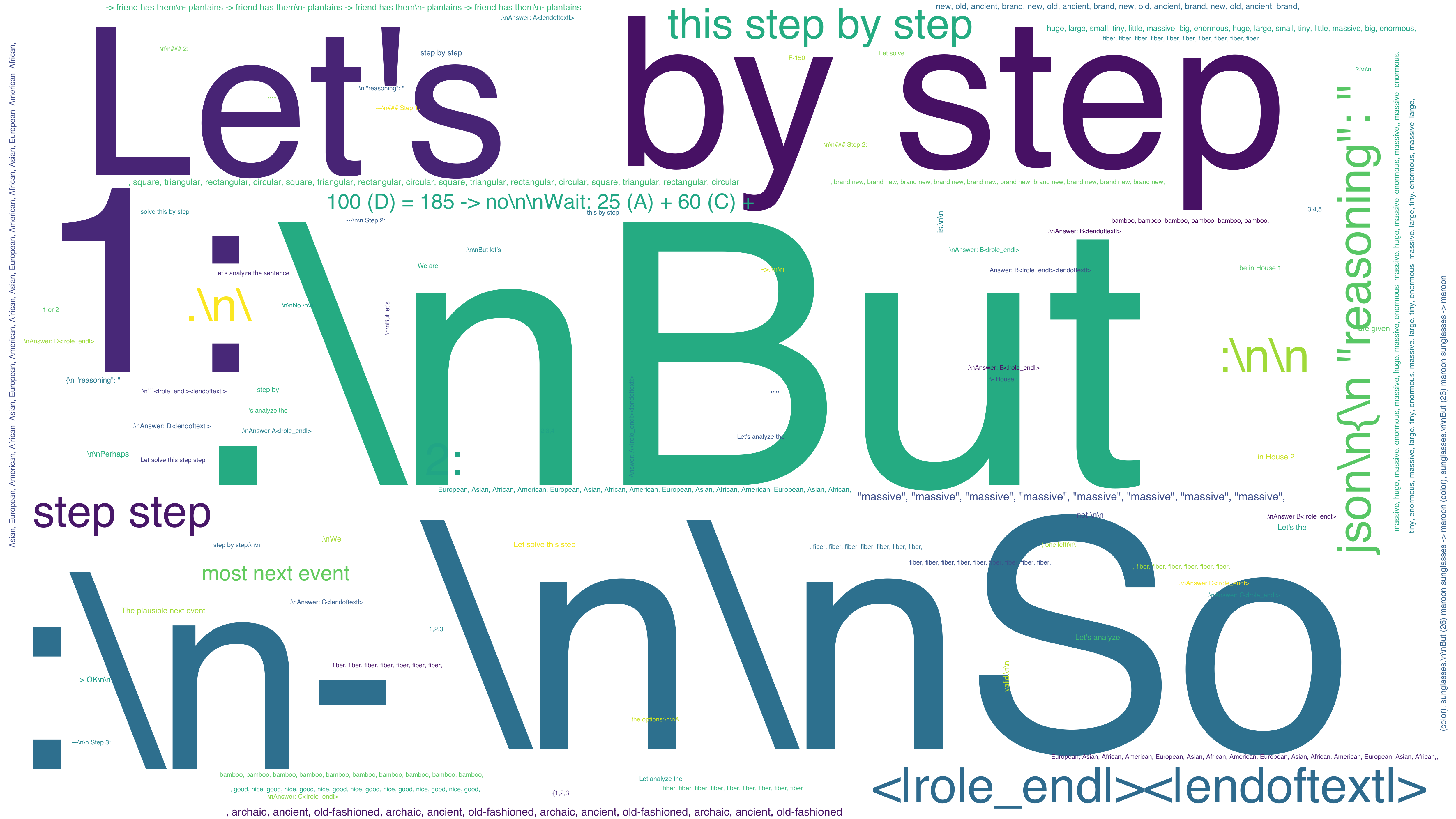}
        \caption{Reasoning}
        \label{fig:sub3}
    \end{subfigure}

    \vspace{1em} % 行间距

    % 第二行：3 张子图
    \begin{subfigure}[b]{0.3\textwidth}
        \centering
        \includegraphics[width=\linewidth]{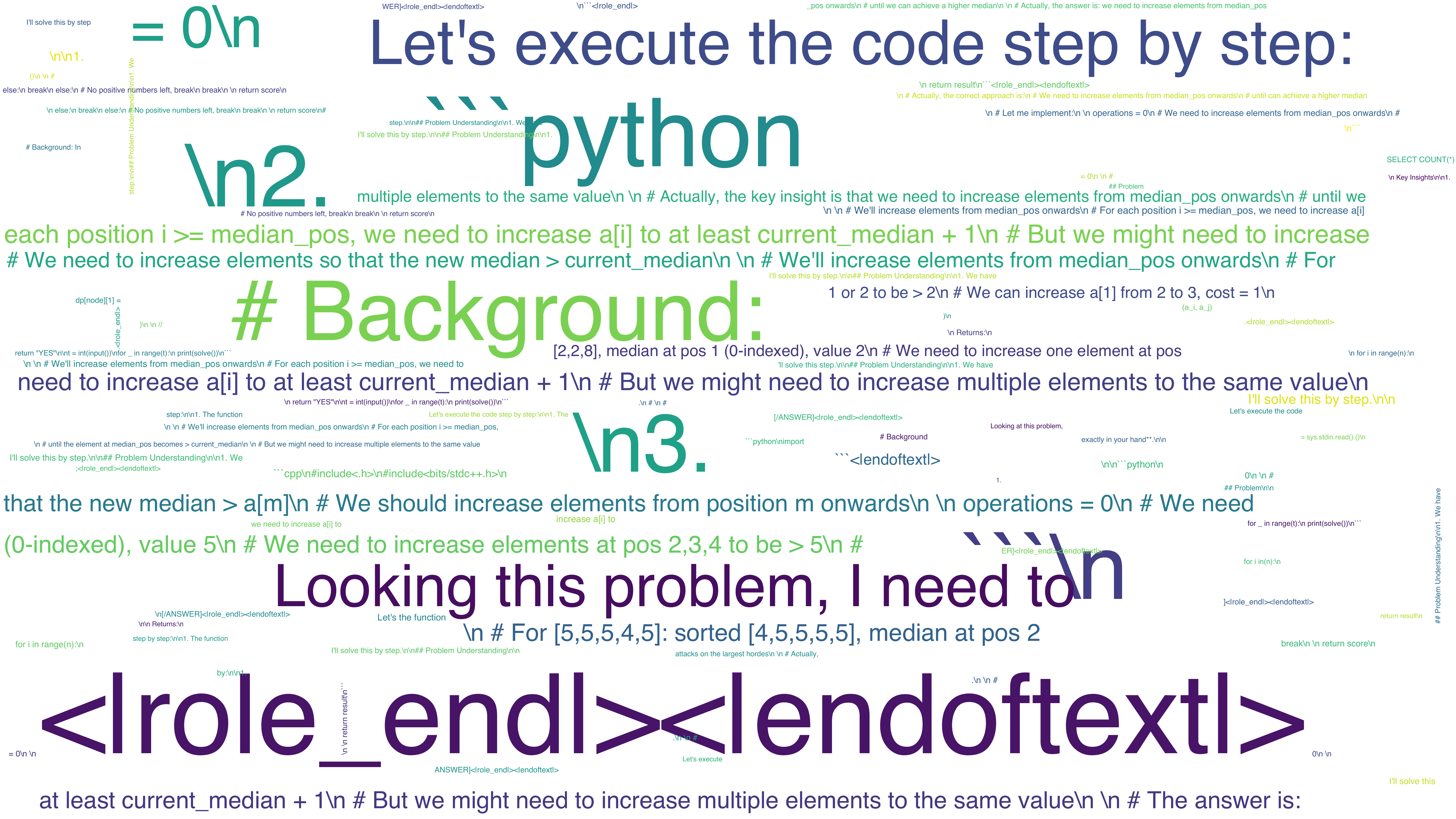}
        \caption{Coding}
        \label{fig:sub4}
    \end{subfigure}
    \hfill
    \begin{subfigure}[b]{0.3\textwidth}
        \centering
        \includegraphics[width=\linewidth]{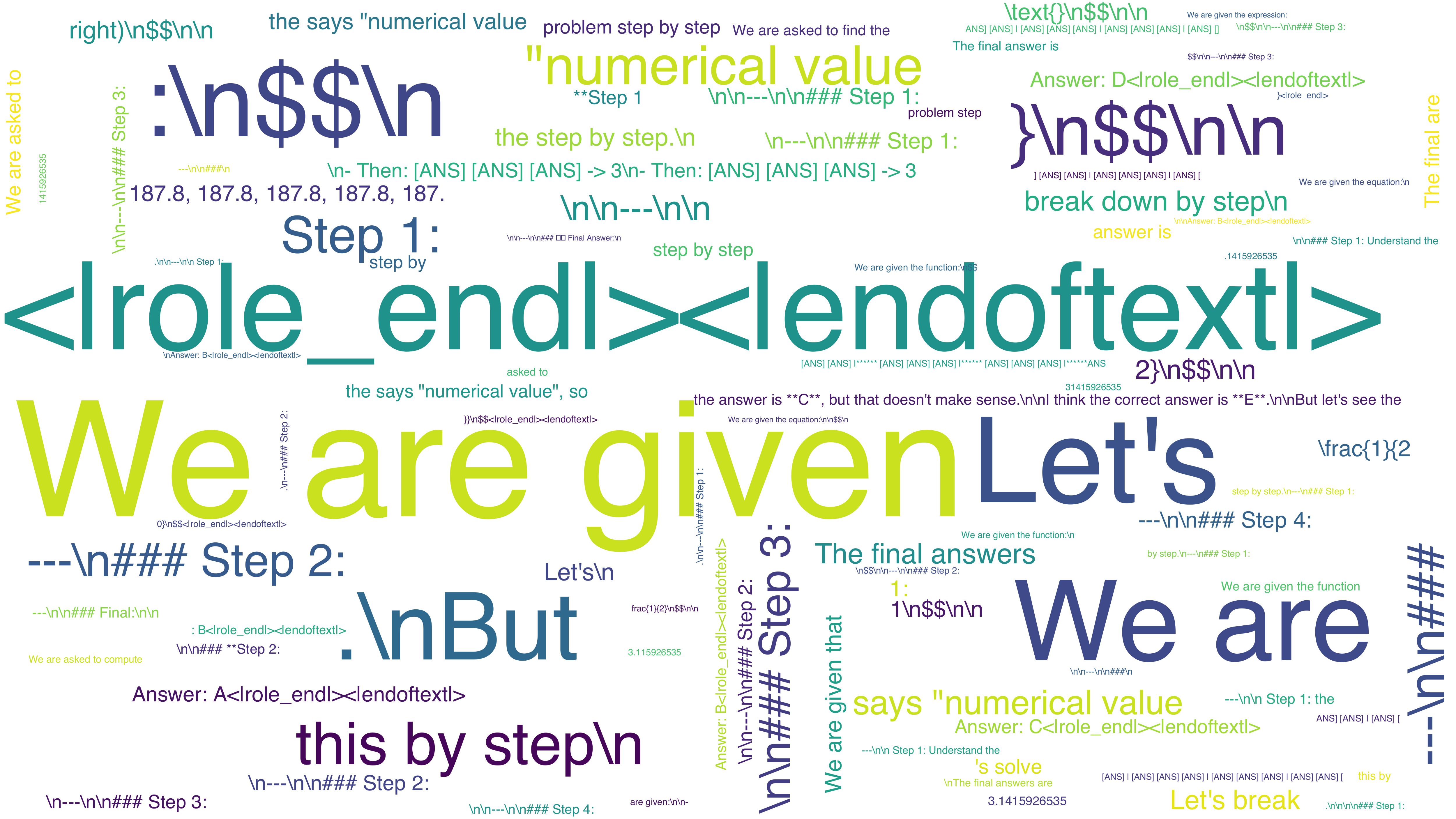}
        \caption{Math}
        \label{fig:sub5}
    \end{subfigure}
    \hfill
    \begin{subfigure}[b]{0.3\textwidth}
        \centering
        \includegraphics[width=\linewidth]{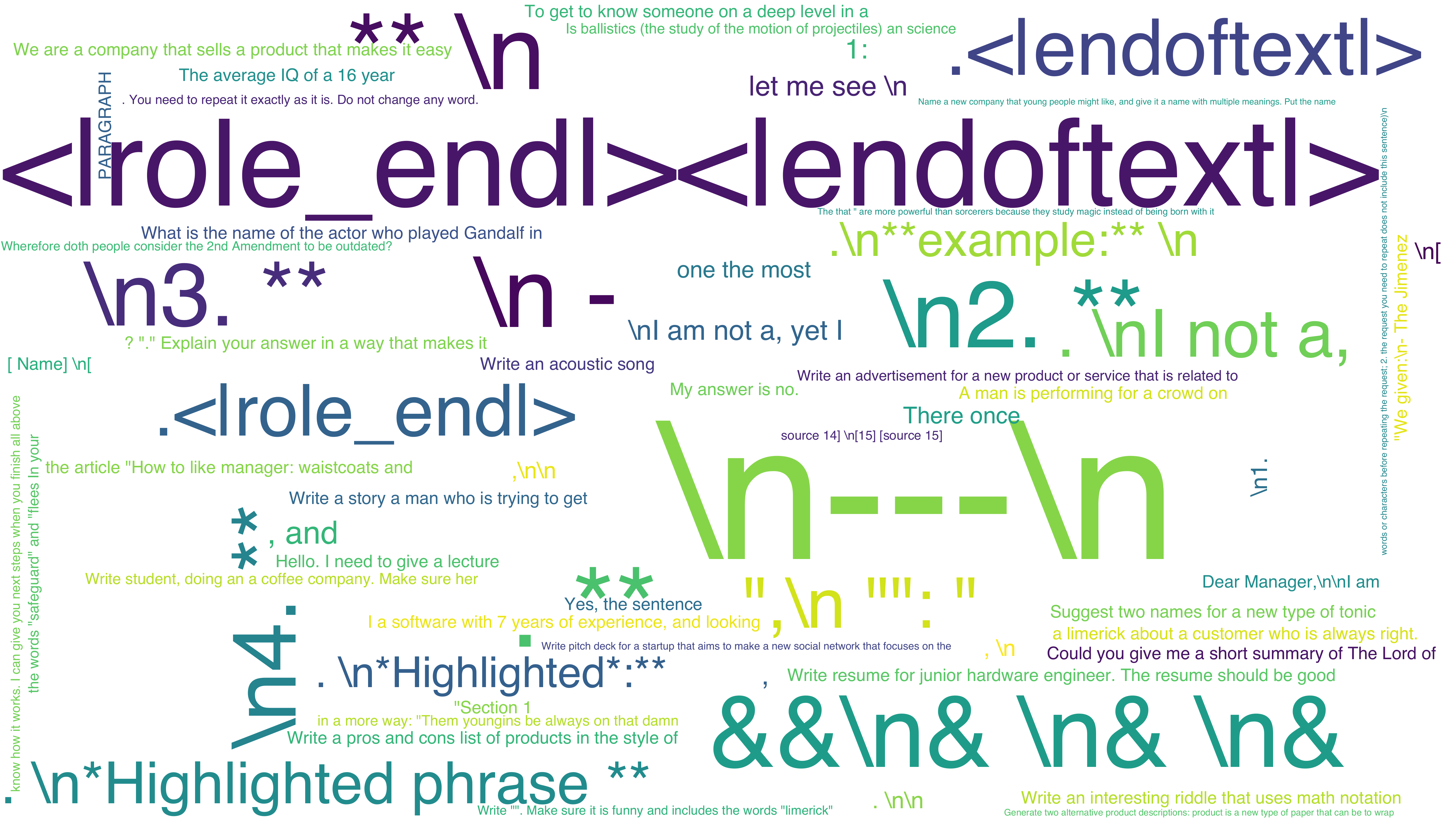}
        \caption{Agent}
        \label{fig:sub6}
    \end{subfigure}

    \caption{\textbf{Word clouds of the most frequent parallel decoding combinations across six benchmarks.} High-frequency co-occurrences are dominated by fixed phrases and specific special token combinations.}
    \label{fig:all word clouds}
\end{figure*}

\begin{figure*}[t] % [t] 表示放置在页面顶部 (Top)
    \centering
    \includegraphics[width=\textwidth]{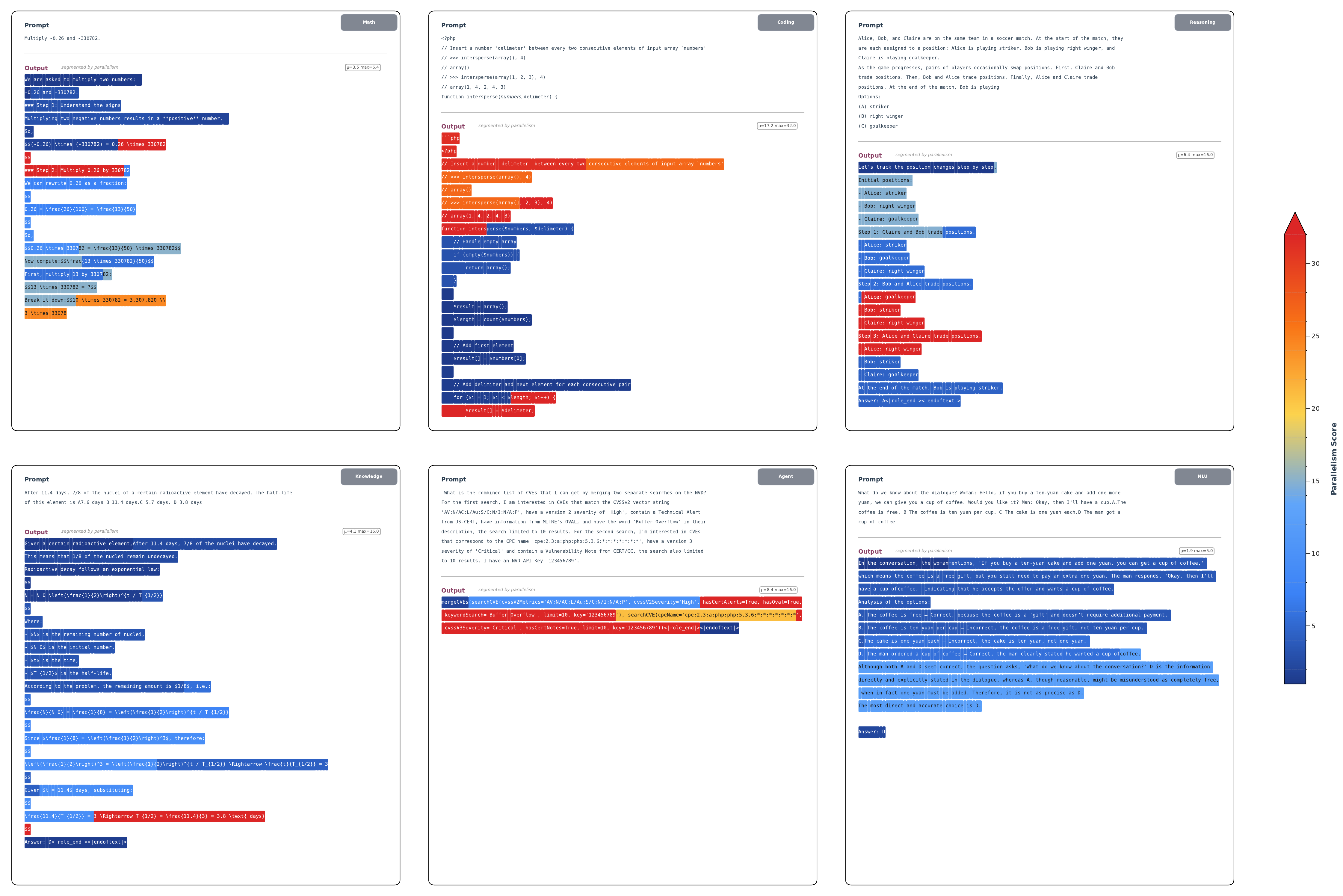} 
    \caption{\textbf{Parallelism visualization plot.} Tokens are organized into blocks, each overlaid with a background color. The intensity of the color corresponds to the average number of tokens decoded per step within that block.}
    \label{fig:Multi_Parallel_Vis}
\end{figure*}

\subsection{Metrics for Decoding Dynamics: AFP and Kendall’s τ}
\label{sec:Metrics_tau_afp}
To quantitatively analyze how MDLMs decode tokens, we introduce two key metrics.

\noindent\textbf{Average Finalization Parallelism(AFP).}
For a generated sequence of length $n$, let $c_i \in \{1,\dots,T\}$ denote the decoding step at which token $i$ is finalized. We define the number of
\emph{effective decoding steps} as
$T_{\mathrm{eff}} = \bigl|\{\, t \mid \exists\, i \text{ s.t. } c_i = t \,\}\bigr|$,
i.e., the number of distinct steps that finalize at least one token.
The \emph{Average Finalization Parallelism} (AFP) is then

\begin{equation}
\label{eq:afp}
    \mathrm{AFP}
    = \frac{n}{T_{\mathrm{eff}}}
    = \frac{n}{\bigl|\{\, t \mid \exists\, i \text{ s.t. } c_i = t \,\}\bigr|}.
\end{equation}
Autoregressive decoding yields $\mathrm{AFP} \approx 1$, while a fully
parallel decoder that finalizes all tokens in a single step achieves
$\mathrm{AFP} = n$.

\vspace{1mm}
\noindent\textbf{Kendall’s $\tau$.} 
To quantify the alignment between finalization and surface order, we adopt Kendall’s $\tau$ \cite{abdi2007kendall}, a rank correlation coefficient from statistics. For a sequence of length $n$, it measures the correlation between token indices $i$ and their finalization steps $c_i$. A pair of indices $(i, j)$ with $i < j$ is concordant if $c_i < c_j$ and discordant if $c_i > c_j$. The metric is defined as:
\begin{equation}
\label{eq:tau}
    \tau = \frac{\mathcal{C} - \mathcal{D}}{n(n-1)/2},
\end{equation}
where $\mathcal{C}$ and $\mathcal{D}$ are the numbers of concordant and discordant pairs, respectively. While AR decoding yields $\tau = 1$ (strictly monotonic), $\tau < 1$ indicates a non-monotonic generation order where future tokens may be finalized before their predecessors.

\subsection{Macro-Level Patterns: How Task Domains Shape Decoding Trajectories}

We employ AFP and Kendall’s $\tau$ to characterize MDLM dynamics, analyzing both \textit{intra-block} distributions and \textit{inter-block} evolution. 

\paragraph{Intra-block Patterns.} We observed significant disparities in the behavioral patterns of samples exhibiting the repetition phenomenon~\cite{fudu}; consequently, these instances are analyzed separately in the subsequent sections. As shown in Fig.~\ref{fig:four_plots}, non-repetitive samples exhibit clear domain-specific traits: 
(1) \textbf{Sequential Dependency}: Logic-intensive tasks (Math, Reasoning, Code) show $\tau > 0.8$, strictly following a left-to-right causal chain. Conversely, Agentic and NLU tasks yield lower $\tau$, favoring global semantic aggregation over linear derivation (Fig.~\ref{fig:case1}). 
(2) \textbf{Difficulty-awareness}: High AFP in Code and Agent domains suggests MDLMs effectively capture structural parallelism. Crucially, correct samples consistently exhibit higher AFP than incorrect ones, indicating that high-confidence parallel decoding aligns with factual correctness (Fig.~\ref{fig:case3}). 
(3) \textbf{Pathological Repetition}: Repetitive outputs display an extreme imbalance: ultra-high parallelism (AFP 5--7$\times$ normal) alongside near-zero $\tau$ (\cref{fig:case2,fig:case4}). This reflects \textbf{entropy collapse}, where the model is trapped in a low-entropy state, triggering aggressive parallel filling.

\paragraph{Inter-block Patterns.} To mitigate length bias, we group samples by total block count. Trajectories within the same domain remain remarkably consistent (Fig.~\ref{fig:Trajectories}; Appendix.~\ref{app:Inter-block Patterns}).
(1) \textbf{$\tau$ Evolution}: Reasoning tasks maintain high-linearity $\tau$; coding exhibits an ``arch-shaped'' trajectory. A universal $\tau$ drop in the final block across all domains reveals a ``generate-EOS-then-backfill'' strategy. 
(2) \textbf{AFP Evolution}: In Knowledge and Math, AFP for correct samples climbs steadily as reasoning progresses, while stagnating for incorrect ones. This suggests the model captures the \textbf{uncertainty decay} driven by contextual accumulation.

\subsection{Micro-Level Semantics: Probing Parallelism and POS-based Order}

\noindent\textbf{Semantic Parallel Composition.} Building upon the statistical analysis of 58 evaluation tasks, we further investigate high-frequency token combinations across varying degrees of parallelism within six capability dimensions (as illustrated in Figure~\ref{fig:all word clouds}). Empirical observations reveal that these parallel-generated token sequences share a distinct characteristic: they are highly formulaic with minimal semantic information gain. These combinations primarily consist of newlines, punctuation, JSON delimiters, Markdown/mathematical symbols, and high-frequency transitional markers. Such patterns suggest that parallel decoding tends to produce generic formatting and discourse structures. This phenomenon aligns closely with the theoretical framework established in Section~\ref{sec:Theoretical Support}, which posits that highly templated token sequences exhibit near-zero Conditional Total Correlation, i.e., $C(Y|X) \approx 0$. In these instances, the predictive signal for each token $y_i \in Y$ within a block originates almost entirely from the global context $X$, rather than local interactions within the block $Y$. For example, the header of a \texttt{for} loop inherently determines the subsequent newline and indentation; since these tokens do not depend on each other's local state, they can be generated in parallel.

\vspace{1.0mm}
\noindent\textbf{Semantic Generation Order.} 
To elucidate the underlying mechanism of generation order, we conduct a case study on the Knowledge Benchmark. We utilize Stanza \cite{Stanza} to perform Part-of-Speech (POS) tagging on model outputs and compute the average decoding step for various POS categories (results presented in Table~\ref{tab:pos_global_avg_step}). The results reveal a clear hierarchical pattern: the model first constructs a "structural backbone" and subsequently fills in "fine-grained details." This non-linear generation logic suggests that MDLMs do not simply construct text in a strictly left-to-right fashion. Instead, they prioritize the deployment of structural anchors by identifying global certainty, followed by the incremental refinement of complex modifiers. Further details are provided in Appendix~\ref{app:pos}.

% \begin{figure*}[t] % [t] 表示放置在页面顶部 (Top)
%     \centering
%     \includegraphics[width=\textwidth]{tau_case.pdf} 
%     \caption{xxxxxxxxx}
%     \label{fig:tau_case}
% \end{figure*}

\subsection{Case Studies: Grounding Dynamics in Semantic Anchors}

\noindent\textbf{Parallelism Cases.} Figure \ref{fig:Multi_Parallel_Vis} provides a qualitative validation of the aggregate statistical trends through representative case studies. 
First, we observe a distinct inverse correlation between parallel intensity and semantic novelty: parallelism peaks during the generation of deterministic content, such as numerical recapitulation, logical restatements, or formulaic templates. 
Second, in the {Knowledge and Math} domains, parallel intensity scales as the reasoning trajectory converges, manifesting a "clarification effect" where decoding accelerates once the solution path becomes certain. 
Finally, {Coding} tasks exhibit an initial surge in parallelism driven by structural boilerplate, which subsequently tapers off as the model transitions to complex core logic. 
These fine-grained observations are highly consistent with the statistical trajectories visualized in \cref{fig:Trajectories}. 
\begin{figure*}[t] % 添加了星号*，注意双栏图片通常只能放在顶部 [t]
    \centering
    \includegraphics[width=\textwidth]{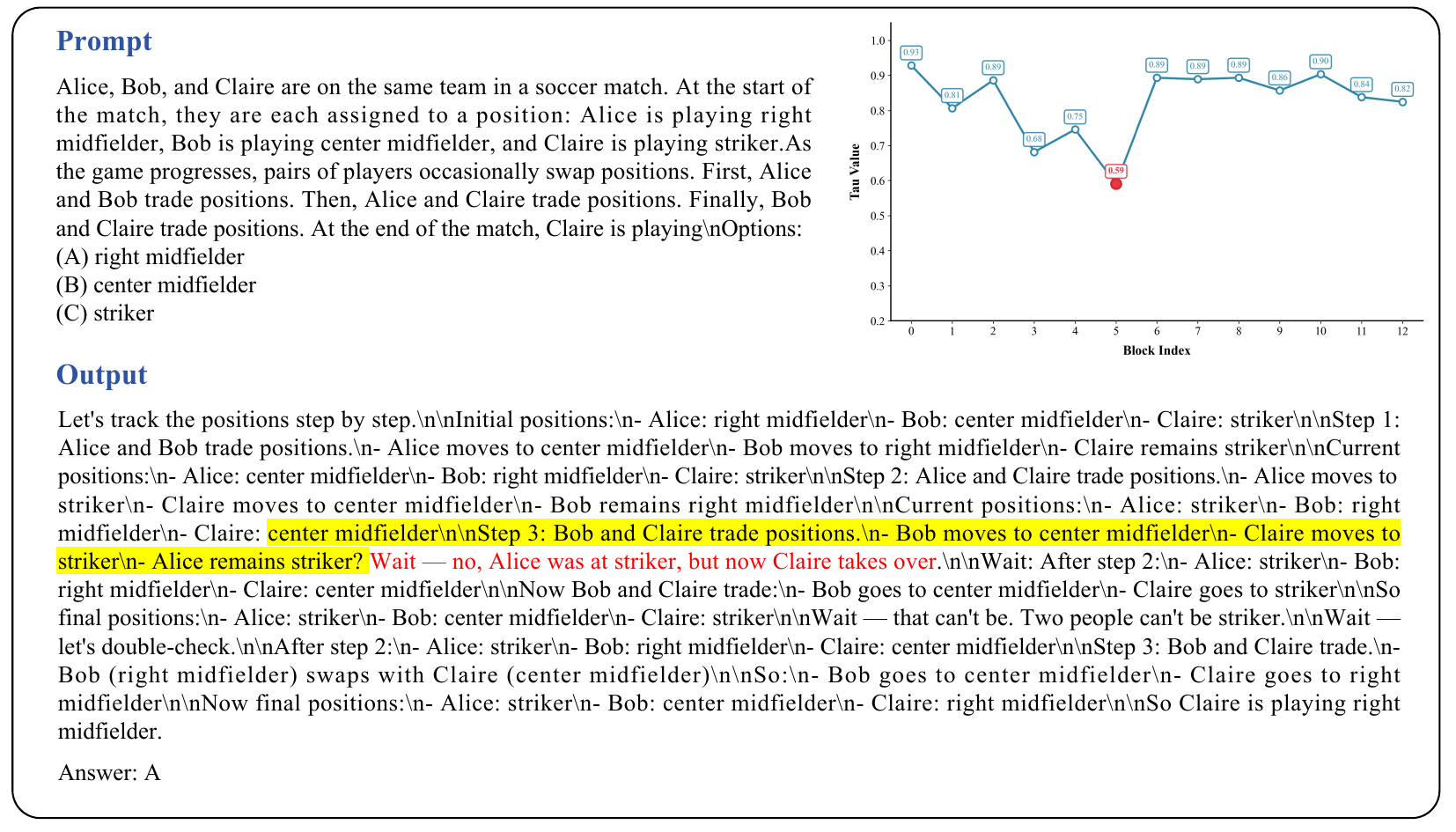} % 宽度改为 \textwidth
    \caption{Case study. Highlighted blocks represent the local minima of τ. More cases in App. \ref{app:order_disruption}.}
    \label{fig:tau_case}
\end{figure*}

% \vspace{1.5mm}
\noindent\textbf{Order Disruptions at Semantic Pivots.} 
Beyond the high global sequentiality shown in \cref{fig:Trajectories,fig:four_plots}, we observe localized "plunges" in Kendall’s $\tau$ (\cref{fig:tau_case}). These minima typically align with \textit{semantic transitions} (e.g., shifting from chain-of-thought reasoning to the final answer) or \textit{structural boundaries} (e.g., paragraph breaks or code-to-comment transitions). At these junctions, the model momentarily diverges from a strict left-to-right trajectory. This suggests that MDLMs treat transition markers as \textit{structural anchors} that are generated out-of-order or in parallel to "pre-plan" the layout of the subsequent thought block before fully finalizing the current one. Detailed case analyses are provided in \cref{app:order_disruption}.

\section{Discussion: Unlocking the Potential of MDLMs}
Driven by the findings in Section~\ref{Decoding Dynamics of MDLMs}, we identify three key potentials for MDLMs: parallelism, generation order, and iterative editing.

\paragraph{1) Parallelism Potential: Acceleration.}
While MDLMs parallelize structured content effectively (Sec.~\ref{sec:Metrics_tau_afp}), current speedup is limited by inference frameworks. Future diffusion-specific optimizations (e.g., \texttt{dinfer}~\cite{dinfer}) will be essential to realizing their speed advantage in structured and long-form generation.

\paragraph{2) Sequential Potential: Non-Sequential Dependency.}
Unlike the unidirectional dependency of AR models, MDLMs support "any-order" inference, enabling them to handle tasks with non-linear dependency structures. 
\textbf{Evidence and Analysis:} In our Sudoku experiments (Appendix~\ref{app:Sudoku}), as well as in comparative studies on chemical molecule design and protein design, MDLMs exhibit causal modeling capabilities that surpass those of AR models. The models spontaneously acquire an “easy-to-hard” strategy: they first fill in high-certainty “anchor points,” and then leverage bidirectional attention to resolve global uncertainty.
\textbf{Research Direction:} Current uniform masking strategies fail to account for the internal causal hierarchy of data, leading to an excessively large search space and numerous invalid paths during training. A key future challenge lies in introducing heuristic biases to optimize decoding path modeling by mining reverse causality.

\paragraph{3) Editing Potential: A Theoretical Path to Bridging the Accuracy Gap.}
To mitigate the loss of causal information stemming from the conditional independence assumption in parallel decoding, we explore the theoretical potential of MDLMs for parallel editing. Proofs provided in Appendix~\ref{app:theory_proof} demonstrate that a "generate-then-edit" paradigm not only effectively recovers accuracy degradation but also redefines the boundaries of parallel generation. This paradigm enables the model to engage in bold probabilistic exploration during the initial stages, followed by efficient parallel editing to rectify biases in batches. This "fast sampling followed by parallel refinement" strategy theoretically allows the model to achieve superior generation quality compared to single-pass inference within a shorter equivalent time. This signifies the transition of MDLMs from mere samplers to self-evolving reasoning systems. Furthermore, we implemented an adaptation of editing capabilities on LLaDA2-mini-16B, the preliminary experimental validation of which is detailed in Appendix~\ref{app:editing_preliminary_results}.

\section{Conclusion}
To the best of our knowledge, this work presents the first systematic and large-scale evaluation of state-of-the-art Masked Diffusion Language Models (MDLMs). By benchmarking across 58 diverse tasks and performing an in-depth mechanistic analysis of the 100B LLaDA model, we reveal the fundamental decoding dynamics that govern parallel and arbitrary-order generation. Our findings, grounded in both extensive empirical evidence and theoretical proofs, illuminate the limits of current parallel factorization while highlighting the immense potential of MDLMs for non-linear causal modeling and efficient iterative editing. This study provides a foundational blueprint for the development of next-generation non-autoregressive language models that are both logically flexible and computationally efficient.

\section*{Limitations}
Potential Risks and Misinterpretation. We evaluate all models under a unified local pipeline (prompting/decoding/scoring) to control confounders and enable fair apples-to-apples relative comparisons across architectures. However, these controlled scores should not be treated as leaderboard-equivalent or directly compared to numbers from original reports, since prompts, decoding constraints, and answer-matching can differ and may lead to different absolute results. In addition, hardware-specific optimization can affect outcomes: e.g., OpenPangu is primarily optimized for Huawei Ascend, and our NVIDIA environment (without Ascend-specific kernels) may yield different practical behavior, so such differences should not be over-interpreted as purely modeling effects.

\section*{Acknowledgments}
We thank Ant Group for providing the computing resources and technical support. ChatGPT was used for language polishing.

\bibliography{custom}

\appendix
% \onecolumn
\clearpage

\section{Appendix}

\label{sec:appendix}

\subsection{Details of Experimental Protocol}
\label{DEP}
\noindent\textbf{Model Selection and Configurations.} 
We evaluate 8 state-of-the-art MDLMs, including LLaDA, Trado ~\cite{tracerl}, and SDAR~\cite{sdar}, with scales reaching 100B parameters (e.g., LLaDA2-flash~\cite{llada2.0}). These models span both block-wise and global diffusion decoding paradigms. For clarity, the inference block size for each model is explicitly documented in our experimental result tables to facilitate a direct comparison of parallelism granularity. As baselines, we benchmark against frontier AR models, including the open-source Qwen3 series ~\cite{qwen3} and proprietary models such as Gemini-2.5 Pro~\cite{gemini25} and OpenAI o3.

\vspace{1mm}
\noindent\textbf{Benchmarks.} To fill the void in systematic assessments for MDLMs, we curate \textbf{58 benchmarks} categorized into six dimensions: 
(i) {Knowledge} (e.g., MMLU); 
(ii) {Mathematics} (e.g., GSM8K, MATH); 
(iii) {Reasoning} (e.g., BBH, GPQA); 
(iv) {Language Understanding} (e.g., Hellaswag); 
(v) {Agentic} (e.g., BFCL); 
(vi) {Coding} (e.g., HumanEval). 
This suite quantifies the performance gap between diffusion and autoregressive paradigms across diverse domains. See Appendix \ref{BD} for a complete list.

% \vspace{1mm}
\noindent\textbf{Implementation Details.} 
Our large-scale evaluation is conducted on a distributed cluster of {512 NVIDIA GPUs}. To maintain the architectural integrity and default configurations of each model, we perform inference using the {original implementations} provided in their respective official open-source repositories. For OpenPangu, we observe that it is natively optimized for the {Huawei Ascend} architecture and lacks optimized kernels for NVIDIA hardware. Due to the resulting {excessive computational overhead}, evaluating OpenPangu on the full 58-benchmark suite was infeasible; we therefore report its performance on a {representative subset}.

\noindent\textbf{Unified evaluation protocol.} All results are obtained by running each model locally. Except for the parameter configurations specified in the table, all other settings follow the best-performing configurations reported in each model's technical report, and results are generated under a unified evaluation pipeline. While this ensures a fair, apples-to-apples comparison across models and benchmarks, the absolute scores of some models may be lower than those reported in their original papers or public leaderboards due to differences in evaluation settings (e.g., prompts, and our stricter answer-matching/content scoring). Therefore, we recommend interpreting our results primarily as relative comparisons under a controlled protocol, rather than as leaderboard-equivalent numbers.

\subsection{Details of Benchmarks}
\label{BD}
This section provides detailed descriptions of the 58 benchmarks used in our large-scale evaluation. Benchmarks are grouped into six categories: Coding, Mathematics, Knowledge, Language Understanding, Reasoning, and Agentic Agentic Capabilities.

\subsection*{Coding}
To comprehensively assess coding capabilities, our evaluation spans from foundational synthesis to complex, domain-specific reasoning. 
We use \textbf{CruxEval-O} to evaluate a model's ability to predict program outputs \citep{gu2024cruxeval}; \textbf{mbpp} to assess the synthesis of basic Python programs from natural language \citep{austin2021program}; \textbf{MultiPL-E} to test multilingual code generation capabilities \citep{cassano2022multipl}; and \textbf{humaneval} for standard code generation from docstrings \citep{chen2021evaluatinglargelanguagemodels}. Furthermore, we assess performance on more complex tasks with \textbf{livecodebench\_v6}, using novel problems from programming contests \citep{feng2024livecodebench}; \textbf{Bigcodebench-Full}, which challenges models to invoke function calls as tools \citep{benallal2024bigcode}; \textbf{LCBench} for learning curve analysis \citep{zimmer2021auto}; the complex text-to-SQL benchmark \textbf{spider} \citep{yu2018spider}; the challenging algorithmic problems from \textbf{CodeForces} \citep{li2022alphacode}; the scientific-focused \textbf{SciCode} benchmark \citep{tian2024scicode}; and \textbf{Aider} for testing interactive code editing skills \citep{aider-git}.
%-------------------------------------------------------

\begin{table*}[t]
\begingroup
\setlength{\arrayrulewidth}{0.6pt}
\newcommand{\Vsep}{\hspace{5pt}\vrule width \arrayrulewidth\hspace{5pt}}

\centering
\small
\setlength{\tabcolsep}{3.2pt}
\renewcommand{\arraystretch}{1.12}
\caption{\textbf{Performance comparison across code benchmarks.} "Block Size 4 (max)" refers to the parallel computation of 4-position logits where only the single most confident token is decoded. Other configurations utilize threshold-based adaptive decoding, accepting all tokens that exceed a specific confidence level. }
\label{tab:code-benchmarks}

\newlength{\cw}
\setlength{\cw}{1.25cm} % adjust if needed
\newcommand{\bh}[1]{%
  \makecell[c]{\parbox[t]{\cw}{\centering\arraybackslash #1}}%
}
\newcommand{\hbreak}{\allowbreak}

\newcommand{\ShortSep}{%
  \noalign{\vskip 1pt}%
  \omit\hfil
    \hspace{1pt}%
    \leaders\hrule height 0.6pt\hskip 4.8cm % <-- 横线长度
    \hspace{2pt}\hfil
  & \omit\hfil
  & \multispan{12}\omit\hfil
  \cr
  \noalign{\vskip 1pt}%
}

\resizebox{\linewidth}{!}{%
\begin{tabular}{@{\hspace{5pt}}p{4.8cm} @{\Vsep} c @{\Vsep\hspace{3pt}}
                *{12}{S[table-format=2.2,table-column-width=\cw]} @{\hspace{5pt}}}

\specialrule{0.9pt}{0pt}{2pt}
\multicolumn{1}{c}{\textbf{Model}} & \multicolumn{1}{c}{\textbf{}} & \multicolumn{12}{c}{\textbf{Coding}} \\
\specialrule{0.6pt}{2pt}{2pt}
\textbf{Model Name} & \textbf{Block Size} &
\multicolumn{1}{c}{\bh{CruxEval-O}} &
\multicolumn{1}{c}{\bh{mbpp}} &
\multicolumn{1}{c}{\bh{MultiPL-E}} &
\multicolumn{1}{c}{\bh{humaneval}} &
\multicolumn{1}{c}{\bh{openai\_\hbreak humaneval}} &
\multicolumn{1}{c}{\bh{livecode\_\hbreak benchv6}} &
\multicolumn{1}{c}{\bh{Bigcode\hbreak bench}} &
\multicolumn{1}{c}{\bh{LCBench}} &
\multicolumn{1}{c}{\bh{spider}} &
\multicolumn{1}{c}{\bh{Code\hbreak Forces}} &
\multicolumn{1}{c}{\bh{Sci\hbreak Code}} &
\multicolumn{1}{c}{\bh{Aider}} \\
\specialrule{0.6pt}{2pt}{2pt}

qwen3-8B  & {/} & 74.25 & 79.80 & 67.14 & 79.73 & 84.76 & 26.60 & 36.14 & 48.73 & 72.94 & 26.14 & 17.71 & 56.95 \\
qwen3-14B & {/} & 79.25 & 83.26 & 69.07 & 84.91 & 87.80 & 29.46 & 38.51 & 55.61 & 76.94 & 27.39 & 21.96 & 68.42 \\
qwen3-30B-A3B-Instruct-2507 & {/} & 90.31 & 86.86 & 75.64 & 89.41 & 93.29 & 47.36 & 41.14 & 79.22 & 81.60 & 73.56 & 26.22 & 74.25 \\
qwen3-Next-80B-A3B-Instruct & {/} & 94.38 & 94.82 & 78.87 & 90.93 & 93.90 & 61.23 & 44.91 & 90.75 & 84.12 & 72.68 & 32.64 & 84.77 \\
o3 & {/} & 99.38 & 98.59 & 71.23 & 89.94 & 98.17 & 73.79 & 40.35 & 94.97 & 77.69 & 86.49 & 40.97 & 90.23 \\
Gemini-2.5 Pro & {/} & 94.38 & 98.33 & 48.26 & 91.16 & 25.00 & 72.08 & 89.97 & 89.97 & 85.10 & 84.77 & 44.10 & 93.98 \\

TraDo-8B-Instruct & 4 & 58.56 & 79.10 & 34.11 & 44.82 & 81.10 & 20.93 & 34.30 & 40.19 & 72.33 & 15.22 & 10.42 & 48.68 \\
TraDo-4B-Instruct & 4 & 52.81 & 68.15 & 21.18 & 27.29 & 78.05 & 14.10 & 30.96 & 26.46 & 71.63 & 6.00  & 11.81 & 32.71 \\
DiRL-8B-Instruct  & 4 & 62.50 & 74.94 & 48.64 & 68.14 & 73.78 & 15.97 & 34.30 & 32.83 & 70.98 & 17.21 & 10.94 & 43.80 \\
SDAR-30B-A3B-Chat & 4 & 60.62 & 75.18 & 23.04 & 12.27 & 85.37 & 21.37 & 38.60 & 47.06 & 76.99 & 12.77 & 9.11  & 45.49 \\
SDAR-4B-Chat      & 4 & 51.06 & 69.09 & 17.65 & 14.10 & 71.95 & 10.79 & 29.30 & 24.83 & 69.59 & 5.62  & 9.81  & 30.83 \\
SDAR-8B-Chat      & 4 & 56.81 & 74.88 & 31.48 & 55.11 & 77.44 & 12.17 & 34.47 & 32.41 & 71.03 & 7.69  & 12.33 & 42.29 \\
Dream-v0-Instruct-7B & full & 40.00 & 52.55 & 25.74 & 16.16 & 52.44 & 5.40  & 18.95 & 8.63  & 47.28 & 9.14  & \multicolumn{1}{c}{--} & 5.26 \\

LLaDA-8B-Instruct & 32 & 20.50 & 21.52 & 11.61 & 16.01 & 19.51 & 1.60 & 2.02 & 3.06 & 28.32 & 1.17 & \multicolumn{1}{c}{--} & 4.51 \\
LLaDA-1.5         & 32 & 21.69 & 22.10 & 11.71 & 14.79 & 23.17 & 1.76 & 2.19 & 3.29 & 28.60 & 1.17 & \multicolumn{1}{c}{--} & 3.57 \\
openPangu-7B-Diffusion-Base & 32 &
42.00 & \multicolumn{1}{c}{--} & \multicolumn{1}{c}{--} & \multicolumn{1}{c}{--} & 18.90 &
\multicolumn{1}{c}{--} & \multicolumn{1}{c}{--} & 18.05 &
\multicolumn{1}{c}{--} & \multicolumn{1}{c}{--} & \multicolumn{1}{c}{--} & 0.00 \\

openPangu-7B-Diffusion-Base & 4(max) &
41.50 & \multicolumn{1}{c}{--} & \multicolumn{1}{c}{--} & \multicolumn{1}{c}{--} & 67.68 &
\multicolumn{1}{c}{--} & \multicolumn{1}{c}{--} & 21.26 &
\multicolumn{1}{c}{--} & \multicolumn{1}{c}{--} & \multicolumn{1}{c}{--} & \multicolumn{1}{c}{--} \\

% openPangu-7B-Diffusion-Thinking & 32 &
% 28.94 & \multicolumn{1}{c}{--} & \multicolumn{1}{c}{--} & \multicolumn{1}{c}{--} & 6.10 &
% \multicolumn{1}{c}{--} & \multicolumn{1}{c}{--} & \multicolumn{1}{c}{--} &
% \multicolumn{1}{c}{--} & \multicolumn{1}{c}{--} & \multicolumn{1}{c}{--} & 16.92 \\

\multirow{3}{*}{LLaDA2-mini-16B}
 & 32  & 72.62 & 82.17 & 70.25 & 81.25 & 85.37 & 31.72 & 33.07 & 67.81 & 76.99 & 22.84 & 13.19 & 39.85 \\
 & 64  & 69.44 & 81.06 & 69.32 & 81.71 & 85.98 & 28.52 & 31.14 & 64.44 & 74.57 & 17.31 & 12.50 & 39.85 \\
 & 128 & 31.31 & 65.46 & 63.70 & 67.30 & 71.95 & 16.85 & 27.46 & 55.44 & 61.76 & 23.31 & 4.17  & 20.30 \\

\multirow{3}{*}{LLaDA2-flash-100B}
 & 32  & 85.44 & 89.14 & 76.01 & 87.65 & 94.51 & 42.51 & 40.88 & 82.91 & 82.49 & 47.72 & 24.65 & 65.04 \\
 & 64  & 79.00 & 86.62 & 74.77 & 85.29 & 91.46 & 42.73 & 39.82 & 78.49 & 81.37 & 51.58 & 21.88 & 63.53 \\
 & 128 & 42.19 & 42.10 & 72.92 & 67.68 & 71.95 & 18.56 & 25.00 & 68.55 & 46.86 & 47.35 & 7.90  & 32.52 \\

\specialrule{0.9pt}{2pt}{0pt}
\end{tabular}%
}
\endgroup
\end{table*}

%------------------------------------------------------
\subsection*{Mathematics}

To evaluate mathematical capabilities, we employ a diverse set of benchmarks. We begin with foundational multi-step reasoning using \textbf{GSM8K} \cite{gsm8k} and its augmented version, \textbf{GSM\_Plus} \cite{gsm-plus}. We then assess advanced problem-solving with the challenging \textbf{MATH} competition dataset \cite{math-dataset}, a curated subset (\textbf{math500} \cite{math500}), and its Chinese-language parallel, \textbf{CMATH} \cite{cmath}, supplemented by \textbf{GKMathUnion}. To probe the limits of model reasoning, we use elite-level competition problems from \textbf{OlympiadBench} \cite{olympiadbench}, recent editions of the \textbf{AIME (24/25)} \cite{aime} and \textbf{HMMT25} \cite{hmmt}, and the exceptionally difficult \textbf{HARDMath2} \cite{hardmath2}. Performance on higher education material is measured using \textbf{College\_Math} \cite{college-math}, the proof-based \textbf{UGMathBench} \cite{xu2025ugmathbench}, and the theorem-application test \textbf{TheoremQA} \cite{theoremqa}. Finally, for a holistic view, we include composite and dynamic benchmarks like \textbf{MathBench} \cite{mathbench}, the \textbf{Minerva\_Math} problem set \cite{minerva}, and the community-driven \textbf{Livemathbench} \cite{liu2025your}.
%-------------------------------------------------------
\begin{table*}[h]
\begingroup
\setlength{\arrayrulewidth}{0.6pt}
\newcommand{\Vsep}{\hspace{5pt}\vrule width \arrayrulewidth\hspace{5pt}}

\centering
\small
\setlength{\tabcolsep}{3.2pt}
\renewcommand{\arraystretch}{1.12}
\caption{\textbf{Performance comparison across mathematics benchmarks.} "Block Size 4 (max)" refers to the parallel computation of 4-position logits where only the single most confident token is decoded. Other configurations utilize threshold-based adaptive decoding, accepting all tokens that exceed a specific confidence level. }
\label{tab:math-benchmarks}
%\newlength{\cw}
\setlength{\cw}{1.25cm} % adjust if needed
\newcommand{\bh}[1]{%
  \makecell[c]{\parbox[t]{\cw}{\centering\arraybackslash #1}}%
}
\newcommand{\hbreak}{\allowbreak}

\newcommand{\ShortSep}{%
  \noalign{\vskip 1pt}%
  \omit\hfil
    \hspace{1pt}%
    \leaders\hrule height 0.6pt\hskip 4.8cm % <-- 横线长度
    \hspace{2pt}\hfil
  & \omit\hfil
  & \multispan{12}\omit\hfil
  \cr
  \noalign{\vskip 1pt}%
}

\resizebox{\linewidth}{!}{%
\begin{tabular}{@{\hspace{5pt}}p{4.8cm} @{\Vsep} c @{\Vsep\hspace{3pt}}
                *{17}{S[table-format=2.2,table-column-width=\cw]} @{\hspace{5pt}}}

\specialrule{0.9pt}{0pt}{2pt}
\multicolumn{1}{c}{\textbf{Model}} & \multicolumn{1}{c}{\textbf{}} & \multicolumn{17}{c}{\textbf{Mathematics}} \\
\specialrule{0.6pt}{2pt}{2pt}
\textbf{Model Name} & \textbf{Block Size} &
\multicolumn{1}{c}{\bh{GSM8K}} &
\multicolumn{1}{c}{\bh{math}} &
\multicolumn{1}{c}{\bh{Olympiad\hbreak Bench}} &
\multicolumn{1}{c}{\bh{CMATH}} &
\multicolumn{1}{c}{\bh{GSM\_\hbreak Plus}} &
\multicolumn{1}{c}{\bh{College\_\hbreak Math}} &
\multicolumn{1}{c}{\bh{Math\hbreak Bench}} &
\multicolumn{1}{c}{\bh{Minerva\_\hbreak Math}} &
\multicolumn{1}{c}{\bh{math500}} &
\multicolumn{1}{c}{\bh{GKMath\hbreak Union}} &
\multicolumn{1}{c}{\bh{Live\hbreak mathbench}} &
\multicolumn{1}{c}{\bh{AIME\hbreak 24}} &
\multicolumn{1}{c}{\bh{AIME\hbreak 25}} &
\multicolumn{1}{c}{\bh{HARD\hbreak Math2}} &
\multicolumn{1}{c}{\bh{UGMath\hbreak Bench}} &
\multicolumn{1}{c}{\bh{Theorem\hbreak QA}} &
\multicolumn{1}{c}{\bh{HMMT\hbreak 25}} \\
\specialrule{0.6pt}{2pt}{2pt}

qwen3-8B  & {/} & 93.73 & 86.24 & 55.07 & 95.54 & 86.06 & 83.07 & 83.65 & 32.05 & 86.20 & 77.46 & 63.39 & 30.10 & 22.29 & 5.57  & 58.77 & 53.00 & 11.35 \\
qwen3-14B & {/} & 95.11 & 88.54 & 58.63 & 95.42 & 87.32 & 84.81 & 88.16 & 35.78 & 88.85 & 79.84 & 66.63 & 32.50 & 26.25 & 1.78  & 62.15 & 59.72 & 11.41 \\
qwen3-30B-A3B-Instruct-2507 & {/} & 96.51 & 96.54 & 77.67 & 96.61 & 89.64 & 89.39 & 94.58 & 35.78 & 98.00 & 91.11 & 85.52 & 76.09 & 61.30 & 4.38  & 70.64 & 72.00 & 43.65 \\
qwen3-Next-80B-A3B-Instruct & {/} & 96.42 & 91.52 & 80.37 & 97.04 & 96.17 & 89.82 & 95.79 & 45.83 & 91.53 & 91.83 & 88.07 & 87.73 & 68.64 & 38.27 & 72.80 & 75.76 & 48.43 \\
o3 & {/} & 95.45 & 96.30 & 82.81 & 94.72 & 88.83 & 89.67 & 93.00 & 19.79 & 94.05 & 91.87 & 82.48 & 88.59 & 85.62 & 44.31 & 74.55 & 76.81 & 76.95 \\
Gemini-2.5 Pro & {/} & 92.99 & 85.78 & 53.04 & 95.72 & 87.30 & 90.13 & 95.31 & 33.88 & 88.75 & 92.62 & 89.28 & 87.81 & 84.27 & 43.36 & 74.47 & 75.94 & 0.00 \\

TraDo-8B-Instruct & 4 & 92.51 & 78.88 & 44.59 & 93.58 & 83.73 & 77.79 & 81.61 & 38.60 & 81.20 & 63.33 & 45.90 & 16.88 & 16.46 & 4.27 & 49.25 & 48.50 & 3.33 \\
TraDo-4B-Instruct & 4 & 91.76 & 76.74 & 41.04 & 93.26 & 82.95 & 75.30 & 80.16 & 34.99 & 76.35 & 58.41 & 40.30 & 13.54 & 6.67  & 0.47 & 48.48 & 42.38 & 6.67 \\
DiRL-8B-Instruct  & 4 & 93.35 & 83.70 & 50.11 & 92.03 & 85.70 & 77.71 & 79.83 & 36.76 & 85.05 & 72.70 & 56.56 & 16.88 & 16.88 & 4.38 & 51.45 & 54.62 & 10.42 \\
SDAR-30B-A3B-Chat & 4 & 91.51 & 77.52 & 35.26 & 92.67 & 82.95 & 65.72 & 79.68 & 27.94 & 78.30 & 56.23 & 41.80 & 6.67  & 10.21 & 2.49 & 46.32 & 44.25 & 3.54 \\
SDAR-4B-Chat      & 4 & 91.26 & 70.46 & 34.59 & 91.26 & 80.63 & 66.47 & 77.12 & 22.73 & 70.25 & 52.54 & 34.02 & 6.46  & 6.67  & 0.47 & 38.66 & 33.88 & 3.33 \\
SDAR-8B-Chat      & 4 & 90.86 & 74.02 & 40.15 & 92.58 & 82.21 & 73.46 & 78.97 & 33.27 & 74.70 & 56.87 & 38.66 & 16.25 & 3.75  & 1.90 & 44.19 & 43.25 & 3.33 \\
Dream-v0-Instruct-7B & full & 9.61 & 36.08 & 11.93 & 63.07 & 4.87 & 68.24 & 46.41 & 1.29 & 36.60 & 10.60 & 19.26 & 0.89 & 0.68 & 0.83 & 15.77 & 25.56 & 0.05 \\

LLaDA-8B-Instruct & 32 & 74.17 & 32.72 & 7.33 & 58.24 & 61.61 & 24.20 & 38.53 & 7.90  & 25.80 & 25.04 & 8.27  & 0.62 & 0.16 & 0.47 & 14.84 & 13.63 & 0.16 \\
LLaDA-1.5         & 32 & 75.55 & 34.32 & 8.00 & 65.39 & 64.49 & 25.87 & 43.96 & 10.17 & 29.70 & 26.94 & 9.15  & 1.15 & 0.10 & 0.47 & 16.81 & 15.00 & 0.21 \\
openPangu-7B-Diffusion-Base & 32 &
26.52 & \multicolumn{1}{c}{--} & \multicolumn{1}{c}{--} & \multicolumn{1}{c}{--} &
\multicolumn{1}{c}{--} & \multicolumn{1}{c}{--} & \multicolumn{1}{c}{--} &
14.52 & 48.00 & \multicolumn{1}{c}{--} & \multicolumn{1}{c}{--} &
\multicolumn{1}{c}{--} & \multicolumn{1}{c}{--} & 0.24 &
\multicolumn{1}{c}{--} & \multicolumn{1}{c}{--} & \multicolumn{1}{c}{--} \\

openPangu-7B-Diffusion-Base & 4(max) &
\multicolumn{1}{c}{--} & \multicolumn{1}{c}{--} & \multicolumn{1}{c}{--} & \multicolumn{1}{c}{--} &
\multicolumn{1}{c}{--} & \multicolumn{1}{c}{--} & \multicolumn{1}{c}{--} &
3.98 & 36.00& \multicolumn{1}{c}{--} & \multicolumn{1}{c}{--} &
\multicolumn{1}{c}{--} & \multicolumn{1}{c}{--} &1.07 &
\multicolumn{1}{c}{--} & \multicolumn{1}{c}{--} & \multicolumn{1}{c}{--} \\

% openPangu-7B-Diffusion-Thinking & 32 &
% 82.26 & \multicolumn{1}{c}{--} & \multicolumn{1}{c}{--} & \multicolumn{1}{c}{--} &
% \multicolumn{1}{c}{--} & \multicolumn{1}{c}{--} & \multicolumn{1}{c}{--} &
% 4.04 & 48.40& \multicolumn{1}{c}{--} & \multicolumn{1}{c}{--} &
% \multicolumn{1}{c}{--} & \multicolumn{1}{c}{--} & 0.00 &
% \multicolumn{1}{c}{--} & \multicolumn{1}{c}{--} & \multicolumn{1}{c}{--} \\

\multirow{3}{*}{LLaDA2-mini-16B}
 & 32  & 94.24 & 93.18 & 67.56 & 95.54 & 86.29 & 85.42 & 87.61 & 30.70 & 93.70 & 87.18 & 77.32 & 46.67 & 40.83 & 3.79 & 58.09 & 63.75 & 24.38 \\
 & 64  & 94.29 & 92.96 & 66.96 & 95.36 & 86.31 & 85.95 & 86.01 & 32.54 & 93.10 & 85.67 & 74.04 & 53.12 & 36.67 & 2.84 & 53.97 & 63.38 & 18.54 \\
 & 128 & 89.14 & 85.98 & 63.19 & 89.75 & 80.61 & 83.07 & 78.81 & 31.80 & 89.80 & 82.86 & 69.67 & 48.54 & 38.33 & 1.42 & 35.16 & 59.38 & 21.46 \\

\multirow{3}{*}{LLaDA2-flash-100B}
 & 32  & 96.36 & 95.60 & 73.48 & 96.99 & 89.60 & 88.57 & 92.03 & 35.66 & 97.60 & 89.84 & 83.88 & 63.33 & 56.67 & 4.27 & 68.63 & 72.25 & 47.50 \\
 & 64  & 96.13 & 95.14 & 72.30 & 95.72 & 89.65 & 88.86 & 89.64 & 38.24 & 96.40 & 89.05 & 82.51 & 61.67 & 49.58 & 3.79 & 66.73 & 70.88 & 52.92 \\
 & 128 & 62.85 & 80.48 & 66.81 & 68.17 & 62.24 & 88.50 & 34.36 & 28.68 & 95.00 & 78.85 & 82.14 & 59.17 & 42.92 & 2.84 & 33.95 & 69.12 & 27.92 \\

\specialrule{0.9pt}{2pt}{0pt}
\end{tabular}%
}
\endgroup
\end{table*}

%-------------------------------------------------------
\subsection*{Knowledge and Language Understanding} 
Our assessment of knowledge and language understanding covers a wide range, from broad factual recall to deep linguistic reasoning. For general world knowledge, we use the Massive Multitask Language Understanding (\textbf{MMLU}) benchmark \cite{mmlu} and its more challenging expert-level variant, \textbf{MMLU-PRO} \cite{mmlu-pro}. We also assess knowledge in Chinese contexts through its counterpart, \textbf{CMMLU} \cite{cmmlu}, and the comprehensive \textbf{C-EVAL} suite \cite{ceval}. Performance on human-standardized exams is measured with \textbf{AGIEval} \cite{agieval}, which includes tasks from exams like the SAT, and \textbf{GAOKAO-Bench} \cite{gaokao-bench}, which is based on the Chinese college entrance exam. To test deep, specialized knowledge, we use the “Google-proof” \textbf{GPQA} benchmark \cite{gpqa}, the scientific reasoning challenge \textbf{ARC-c} \cite{arc}, the multi-disciplinary \textbf{scibench} \cite{scibench}, and the high-difficulty physics benchmark \textbf{PhyBench} \cite{phybench}. Open-domain question answering and fact retrieval are assessed using \textbf{TriviaQA} \cite{triviaqa} and the real-world query-based \textbf{NQ} (Natural Questions) \cite{natural-questions}. Finally, deeper linguistic comprehension is evaluated through the Classical Chinese understanding benchmark \textbf{C³} \cite{c3}, the coreference resolution task \textbf{WSC} (Winograd Schema Challenge) \cite{wsc}, and \textbf{SQuAD2.0} \cite{squad2}, which tests a model's ability to handle unanswerable questions.
%-------------------------------------------------------
\begin{table*}[h]
\begingroup
\setlength{\arrayrulewidth}{0.6pt}
\newcommand{\Vsep}{\hspace{5pt}\vrule width \arrayrulewidth\hspace{5pt}}

\centering
\small
\setlength{\tabcolsep}{3.2pt}
\renewcommand{\arraystretch}{1.12}
\caption{\textbf{Performance comparison across Knowledge and Language benchmarks.} "Block Size 4 (max)" refers to the parallel computation of 4-position logits where only the single most confident token is decoded. Other configurations utilize threshold-based adaptive decoding, accepting all tokens that exceed a specific confidence level. }
\label{tab:knowledge-language-benchmarks}

\newlength{\cwK}
\newlength{\cwL}
\setlength{\cwK}{1.15cm} 
\setlength{\cwL}{1.15cm} 

\newcommand{\bhK}[1]{\makecell[c]{\parbox[t]{\cwK}{\centering\arraybackslash #1}}}
\newcommand{\bhL}[1]{\makecell[c]{\parbox[t]{\cwL}{\centering\arraybackslash #1}}}
\newcommand{\hbreak}{\allowbreak}

\resizebox{\linewidth}{!}{%
\begin{tabular}{@{\hspace{5pt}} p{4.8cm} @{\Vsep} c @{\Vsep\hspace{3pt}}
                *{12}{S[table-format=2.2,table-column-width=\cwK]} @{\Vsep\hspace{3pt}}
                *{3}{S[table-format=2.2,table-column-width=\cwL]} @{\hspace{5pt}}}

\specialrule{0.9pt}{0pt}{2pt}
\multicolumn{1}{c}{\textbf{Model}} & \multicolumn{1}{c}{\textbf{}} &
\multicolumn{12}{c}{\textbf{Knowledge}} & \multicolumn{3}{c}{\textbf{Language}} \\
\specialrule{0.6pt}{2pt}{2pt}
\textbf{Model Name} & \textbf{Block Size} &
\multicolumn{1}{c}{\bhK{MMLU}} &
\multicolumn{1}{c}{\bhK{MMLU-\hbreak PRO}} &
\multicolumn{1}{c}{\bhK{GPQA}} &
\multicolumn{1}{c}{\bhK{AGIEval}} &
\multicolumn{1}{c}{\bhK{GAOKAO-\hbreak Bench}} &
\multicolumn{1}{c}{\bhK{Trivia\hbreak QA}} &
\multicolumn{1}{c}{\bhK{NQ}} &
\multicolumn{1}{c}{\bhK{ARC-c}} &
\multicolumn{1}{c}{\bhK{sci\hbreak bench}} &
\multicolumn{1}{c}{\bhK{Phy\hbreak bench}} &
\multicolumn{1}{c}{\bhK{CMMLU}} &
\multicolumn{1}{c}{\bhK{C-\hbreak EVAL}} &
\multicolumn{1}{c}{\bhL{C3}} &
\multicolumn{1}{c}{\bhL{WSC}} &
\multicolumn{1}{c}{\bhL{squad\hbreak 2.0}} \\
\specialrule{0.6pt}{2pt}{2pt}

qwen3-8B  & {/} & 82.35 & 66.12 & 47.92 & 71.19 & 86.13 & 52.47 & 30.91 & 92.71 & 3.51 & 11.89 & 79.09 & 81.44 & 92.05 & 77.70 & 85.25 \\
qwen3-14B & {/} & 84.67 & 70.49 & 51.42 & 77.21 & 90.17 & 61.19 & 34.07 & 93.69 & 2.97 & 13.49 & 82.37 & 84.64 & 94.08 & 74.70 & 90.89 \\
qwen3-30B-A3B-Instruct-2507 & {/} & 87.20 & 74.09 & 54.83 & 84.60 & 94.24 & 65.52 & 39.25 & 95.81 & 4.57 & 31.86 & 86.46 & 87.86 & 96.16 & 90.62 & 89.71 \\
qwen3-Next-80B-A3B-Instruct & {/} & 96.84 & 81.62 & 73.45 & 87.43 & 91.32 & 72.86 & 58.89 & 96.86 & 1.53 & 32.76 & 88.89 & 95.39 & 97.42 & 95.97 & 91.32 \\
o3 & {/} & 91.59 & 82.09 & 81.41 & 87.65 & 93.57 & 86.21 & 66.65 & 96.82 & 4.24 & 39.81 & 86.47 & 87.51 & 98.14 & 99.04 & 86.49 \\
Gemini-2.5 Pro & {/} & 92.87 & 85.59 & 84.34 & 88.77 & 96.91 & 84.29 & 63.27 & 97.20 & 0.81 & 52.63 & 89.00 & 89.67 & 98.14 & 99.94 & 69.38 \\

TraDo-8B-Instruct & 4 & 79.69 & 58.96 & 39.33 & 64.96 & 83.66 & 58.35 & 34.40 & 91.53 & 3.53 & 6.55  & 77.66 & 78.37 & 93.97 & 78.85 & 89.14 \\
TraDo-4B-Instruct & 4 & 75.81 & 53.62 & 35.26 & 59.53 & 77.67 & 48.44 & 28.73 & 89.83 & 2.69 & 8.23  & 73.14 & 74.59 & 91.07 & 57.69 & 87.80 \\
DiRL-8B-Instruct  & 4 & 81.08 & 44.87 & 44.07 & 61.18 & 82.08 & 56.57 & 33.07 & 90.17 & 2.30 & 11.81 & 76.95 & 79.44 & 93.48 & 78.85 & 87.41 \\
SDAR-30B-A3B-Chat & 4 & 80.92 & 59.82 & 36.84 & 69.29 & 86.81 & 65.23 & 39.53 & 94.83 & 3.01 & 7.34  & 81.71 & 81.16 & 96.00 & 78.85 & 88.23 \\
SDAR-4B-Chat      & 4 & 74.12 & 50.09 & 29.45 & 59.23 & 74.42 & 47.48 & 28.39 & 91.19 & 2.50 & 4.14  & 72.33 & 72.25 & 91.01 & 60.58 & 87.26 \\
SDAR-8B-Chat      & 4 & 77.72 & 55.17 & 36.96 & 64.08 & 82.41 & 57.33 & 33.30 & 92.20 & 2.92 & 5.74  & 76.89 & 76.91 & 93.42 & 77.88 & 89.02 \\
Dream-v0-Instruct-7B & full & 69.28 & 42.10 & 32.04 & 46.36 & 47.43 & 50.33 & 24.63 & 89.53 & 1.56 & 3.30  & 51.66 & 51.80 & 83.62 & 71.33 & 85.12 \\

LLaDA-8B-Instruct & 32 & 50.78 & 25.31 & 23.30 & 39.41 & 49.39 & 37.56 & 18.37 & 80.59 & 1.09 & 1.54 & 45.76 & 46.54 & 70.03 & 32.75 & 88.17 \\
LLaDA-1.5         & 32 & 56.00 & 30.19 & 24.53 & 42.15 & 54.30 & 37.48 & 19.14 & 79.45 & 0.40 & 2.23 & 52.42 & 52.98 & 80.49 & 37.56 & 88.48 \\
openPangu-7B-Diffusion-Base & 32 &
\multicolumn{1}{c}{--} & \multicolumn{1}{c}{--} & \multicolumn{1}{c}{--} & \multicolumn{1}{c}{--} & 39.28 &
\multicolumn{1}{c}{--} & \multicolumn{1}{c}{--} & 80.34 &
2.27 & 4.35 & \multicolumn{1}{c}{--} & 68.82 &
86.58 & 48.08 & 52.19 \\

openPangu-7B-Diffusion-Base & 4(max) &
\multicolumn{1}{c}{--} & \multicolumn{1}{c}{--} & \multicolumn{1}{c}{--} & \multicolumn{1}{c}{--} & \multicolumn{1}{c}{--} &
\multicolumn{1}{c}{--} & \multicolumn{1}{c}{--} & 76.27 &
3.31 & 5.59 & \multicolumn{1}{c}{--} & 72.24 &
83.34 & 51.92 & \multicolumn{1}{c}{--} \\

% openPangu-7B-Diffusion-Thinking & 32 &
% 75.67 & 31.16 & 41.60 & 49.88 & 52.70 &
% \multicolumn{1}{c}{--} & \multicolumn{1}{c}{--} & 87.46 &
% 2.14 & 2.69 & 69.07 & 67.69 &
% 86.36 & 76.92 & 13.24 \\

\multirow{3}{*}{LLaDA2-mini-16B}
 & 32  & 80.58 & 64.21 & 47.22 & 72.74 & 83.79 & 52.59 & 32.58 & 92.88 & 3.71 & 13.67 & 79.41 & 81.86 & 89.64 & 69.23 & 86.45 \\
 & 64  & 80.58 & 62.43 & 45.74 & 72.09 & 83.40 & 52.07 & 32.27 & 92.88 & 4.20 & 13.58 & 79.76 & 79.94 & 89.26 & 70.43 & 86.46 \\
 & 128 & 76.93 & 59.33 & 46.28 & 70.01 & 74.15 & 50.01 & 28.86 & 89.41 & 5.78 & 7.83  & 80.16 & 80.16 & 89.04 & 68.75 & 81.34 \\

\multirow{3}{*}{LLaDA2-flash-100B}
 & 32  & 87.91 & 74.84 & 62.25 & 82.01 & 93.29 & 66.81 & 44.79 & 95.93 & 3.74 & 27.71 & 85.05 & 85.93 & 94.19 & 93.27 & 89.88 \\
 & 64  & 86.75 & 73.89 & 65.59 & 80.72 & 92.71 & 66.74 & 44.57 & 94.58 & 4.06 & 20.53 & 83.53 & 84.62 & 93.21 & 93.27 & 89.75 \\
 & 128 & 42.13 & 21.99 & 46.53 & 46.31 & 57.26 & 65.10 & 40.75 & 31.86 & 7.57 & 22.02 & 26.99 & 34.98 & 50.58 & 24.04 & 76.85 \\

\specialrule{0.9pt}{2pt}{0pt}
\end{tabular}%
}
\endgroup
\end{table*}
%-------------------------------------------------------
\subsection*{Reasoning and Agentic Capabilities} 
We evaluate advanced reasoning and agentic abilities across a comprehensive set of benchmarks. For \textit{reasoning}, we include: multi-step abstract and compositional reasoning (\textbf{BBH} \cite{bbh}, \textbf{BBeH}); commonsense and physical reasoning in grounded scenarios (\textbf{HellaSwag} \cite{hellaswag}, \textbf{PIQA} \cite{piqa}); formal logical deduction via synthetic world models (\textbf{ProntoQA} \cite{prontoqa}), constraint-satisfaction puzzles (\textbf{ZebraLogic} \cite{zebralogic}), and open-ended bilingual logic (\textbf{AutoLogi} \cite{autologi}); expert-level academic reasoning resistant to web retrieval (\textbf{HLE} \cite{hle}); multistep soft reasoning over natural language narratives (\textbf{MuSR} \cite{sprague2023musr}); Chinese textual entailment (\textbf{OCNLI} \cite{ocnli}); discrete numerical reasoning over passages (\textbf{DROP} \cite{drop}); and knowledge-orthogonal rule application in out-of-distribution settings (\textbf{KorBench} \cite{ma2024kor}). For \textit{agentic} capabilities, we assess precise tool/API invocation correctness via \textbf{BFCL} \cite{bfcl}, and multi-agent coordination, workflow orchestration, and supervisor hierarchy management using \textbf{Nexus} \cite{nexus}.
%-------------------------------------------------------
\begin{table*}[h]
\begingroup
\setlength{\arrayrulewidth}{0.6pt}
\newcommand{\Vsep}{\hspace{5pt}\vrule width \arrayrulewidth\hspace{5pt}}

\centering
\small
\setlength{\tabcolsep}{3.2pt}
\renewcommand{\arraystretch}{1.12}
\caption{\textbf{Performance comparison across Reasoning and Agent benchmarks.} "Block Size 4 (max)" refers to the parallel computation of 4-position logits where only the single most confident token is decoded. Other configurations utilize threshold-based adaptive decoding, accepting all tokens that exceed a specific confidence level. }
\label{tab:reasoning-agent-benchmarks}

\newlength{\cwR}
\newlength{\cwA}
\setlength{\cwR}{1.15cm}
\setlength{\cwA}{1.15cm}

\newcommand{\bhR}[1]{\makecell[c]{\parbox[t]{\cwR}{\centering\arraybackslash #1}}}
\newcommand{\bhA}[1]{\makecell[c]{\parbox[t]{\cwA}{\centering\arraybackslash #1}}}
\newcommand{\hbreak}{\allowbreak}

\resizebox{\linewidth}{!}{%
\begin{tabular}{@{\hspace{5pt}} p{4.8cm} @{\Vsep} c @{\Vsep\hspace{3pt}}
                *{12}{S[table-format=2.2,table-column-width=\cwR]} @{\Vsep\hspace{3pt}}
                *{2}{S[table-format=2.2,table-column-width=\cwA]} @{\hspace{5pt}}}

\specialrule{0.9pt}{0pt}{2pt}
\multicolumn{1}{c}{\textbf{Model}} & \multicolumn{1}{c}{\textbf{}} &
\multicolumn{12}{c}{\textbf{Reasoning}} & \multicolumn{2}{c}{\textbf{Agent}} \\
\specialrule{0.6pt}{2pt}{2pt}
\textbf{Model Name} & \textbf{Block Size} &
\multicolumn{1}{c}{\bhR{ocnli}} &
\multicolumn{1}{c}{\bhR{drop}} &
\multicolumn{1}{c}{\bhR{kor\hbreak bench}} &
\multicolumn{1}{c}{\bhR{bbh}} &
\multicolumn{1}{c}{\bhR{bbeh}} &
\multicolumn{1}{c}{\bhR{hella\hbreak swag}} &
\multicolumn{1}{c}{\bhR{piqa}} &
\multicolumn{1}{c}{\bhR{Pronto\hbreak QA}} &
\multicolumn{1}{c}{\bhR{HLE}} &
\multicolumn{1}{c}{\bhR{Zebra\hbreak Logic}} &
\multicolumn{1}{c}{\bhR{Auto\hbreak Logi}} &
\multicolumn{1}{c}{\bhR{MuSR}} &
\multicolumn{1}{c}{\bhA{BFCL}} &
\multicolumn{1}{c}{\bhA{nexus}} \\
\specialrule{0.6pt}{2pt}{2pt}

qwen3-8B  & {/} & 60.88 & 84.39 & 58.80 & 78.71 & 17.87 & 79.19 & 88.25 & 93.19 & 3.85 & 37.45 & 76.49 & 69.24 & 70.27 & 38.46 \\
qwen3-14B & {/} & 67.22 & 84.78 & 62.48 & 83.09 & 16.96 & 86.52 & 88.68 & 88.06 & 3.20 & 48.98 & 82.13 & 72.41 & 67.23 & 43.17 \\
qwen3-30B-A3B-Instruct-2507 & {/} & 71.29 & 89.88 & 74.16 & 85.38 & 37.30 & 86.39 & 91.78 & 96.38 & 5.89 & 91.08 & 88.24 & 78.89 & 73.43 & 49.92 \\
qwen3-Next-80B-A3B-Instruct & {/} & 73.86 & 91.32 & 78.48 & 86.88 & 47.16 & 89.13 & 93.85 & 97.50 & 7.04 & 95.58 & 89.52 & 78.73 & 76.24 & 50.50 \\
o3 & {/} & 75.12 & 88.09 & 81.04 & 89.71 & 66.20 & 92.67 & 96.19 & 76.12 & 15.85 & 93.90 & 88.53 & 82.33 & 75.19 & 53.08 \\
Gemini-2.5 Pro & {/} & 73.29 & 87.35 & 81.36 & 88.66 & 66.75 & 92.86 & 96.35 & 96.88 & 7.83 & 91.90 & 77.61 & 82.31 & 75.16 & 58.18 \\

TraDo-8B-Instruct & 4 & 65.39 & 60.03 & 51.12 & 59.38 & 14.96 & 83.68 & 87.38 & 73.12 & 3.01 & 14.10 & 49.75 & 61.03 & 58.15 & 41.09 \\
TraDo-4B-Instruct & 4 & 62.98 & 59.77 & 47.36 & 54.98 & 14.25 & 81.77 & 83.13 & 61.12 & 4.12 & 9.22  & 40.91 & 55.61 & 41.69 & 34.90 \\
DiRL-8B-Instruct  & 4 & 65.80 & 45.52 & 51.84 & 61.80 & 16.31 & 77.72 & 86.83 & 93.50 & 3.99 & 13.60 & 61.00 & 58.21 & 45.98 & 38.53 \\
SDAR-30B-A3B-Chat & 4 & 65.76 & 76.92 & 50.40 & 55.15 & 15.04 & 88.93 & 91.29 & 48.12 & 3.85 & 11.10 & 41.84 & 57.97 & 46.10 & 36.27 \\
SDAR-4B-Chat      & 4 & 62.58 & 60.06 & 46.24 & 53.43 & 13.43 & 83.14 & 81.99 & 59.00 & 3.89 & 8.00  & 35.16 & 53.58 & 40.03 & 34.07 \\
SDAR-8B-Chat      & 4 & 64.17 & 65.09 & 49.04 & 52.18 & 13.51 & 83.16 & 86.83 & 72.62 & 3.20 & 11.12 & 43.70 & 59.15 & 53.18 & 39.12 \\
Dream-v0-Instruct-7B & full & 49.46 & 74.67 & 33.52 & 43.58 & 13.58 & 69.12 & 87.70 & 66.94 & 4.54 & 4.35  & 26.82 & 46.95 & 52.53 & 29.14 \\

LLaDA-8B-Instruct & 32 & 24.85 & 58.66 & 34.64 & 38.95 & 10.87 & 68.58 & 71.93 & 51.19 & 3.94 & \multicolumn{1}{c}{--} & 13.17 & 46.74 & \multicolumn{1}{c}{--} & 11.15 \\
LLaDA-1.5         & 32 & 35.73 & 63.79 & 35.12 & 42.69 & 10.97 & 67.66 & 73.45 & 56.44 & 4.73 & \multicolumn{1}{c}{--} & 13.43 & 49.79 & \multicolumn{1}{c}{--} & 9.27 \\
openPangu-7B-Diffusion-Base & 32 &
\multicolumn{1}{c}{--} & \multicolumn{1}{c}{--} & \multicolumn{1}{c}{--} & \multicolumn{1}{c}{--} &
\multicolumn{1}{c}{--} & \multicolumn{1}{c}{--} & 72.74 & \multicolumn{1}{c}{--} &
\multicolumn{1}{c}{--} & \multicolumn{1}{c}{--} & \multicolumn{1}{c}{--} & 49.49 &
\multicolumn{1}{c}{--} & \multicolumn{1}{c}{--} \\

openPangu-7B-Diffusion-Base & 4(max) &
\multicolumn{1}{c}{--} & \multicolumn{1}{c}{--} & \multicolumn{1}{c}{--} & \multicolumn{1}{c}{--} &
\multicolumn{1}{c}{--} & \multicolumn{1}{c}{--} & 69.21 & \multicolumn{1}{c}{--} &
\multicolumn{1}{c}{--} & \multicolumn{1}{c}{--} & \multicolumn{1}{c}{--} & 52.65 &
\multicolumn{1}{c}{--} & \multicolumn{1}{c}{--} \\

% openPangu-7B-Diffusion-Thinking & 32 &
% 38.17 & 9.33 & \multicolumn{1}{c}{--} & 45.87 & \multicolumn{1}{c}{--} & 59.75 &
% 82.81 & 24.25 & 2.41 & 37.65 & 13.01 & 54.57 &
% \multicolumn{1}{c}{--} & \multicolumn{1}{c}{--} \\

\multirow{3}{*}{LLaDA2-mini-16B}
 & 32  & 63.80 & 85.76 & 54.40 & 78.70 & 15.79 & 78.87 & 86.40 & 88.88 & 4.36 & 62.48 & 71.81 & 71.10 & 71.15 & 36.80 \\
 & 64  & 63.32 & 83.92 & 53.52 & 79.00 & 16.00 & 78.92 & 86.78 & 83.75 & 4.82 & 47.25 & 68.99 & 74.41 & 71.79 & 36.53 \\
 & 128 & 62.61 & 77.79 & 49.44 & 74.65 & 14.55 & 78.74 & 86.78 & 80.00 & 4.08 & 34.30 & 57.13 & 69.50 & 69.19 & 23.23 \\

\multirow{3}{*}{LLaDA2-flash-100B}
 & 32  & 71.86 & 89.45 & 67.76 & 87.00 & 28.27 & 85.08 & 92.93 & 97.00 & 4.91 & 82.05 & 79.53 & 78.23 & 74.90 & 50.82 \\
 & 64  & 68.31 & 85.51 & 68.24 & 84.87 & 27.40 & 83.60 & 91.19 & 97.00 & 4.17 & 74.60 & 81.27 & 78.64 & 74.54 & 49.05 \\
 & 128 & 6.10  & 30.07 & 56.08 & 33.88 & 21.20 & 58.72 & 52.45 & 53.50 & 4.36 & 50.75 & 72.18 & 54.03 & 64.52 & 27.70 \\

\specialrule{0.9pt}{2pt}{0pt}
\end{tabular}%
}
\endgroup
\end{table*}

\subsection{Details of Inter-block Patterns. }
\label{app:Inter-block Patterns}

\subsubsection{Dynamic variations of τ values across six domains for positive and negative samples.}
\label{app:tau_block}
See \cref{app:know_tau,app:agent_tau,app:math_tau,app:coding_tau,app:readsoning_tau,app:nlu_tau} for details.

\begin{figure*}
    \centering
    % height=\textheight 确保高度不超过正文区域
    % keepaspectratio 确保图片不会变形
    \includegraphics[width=\textwidth, height=0.9\textheight, keepaspectratio]{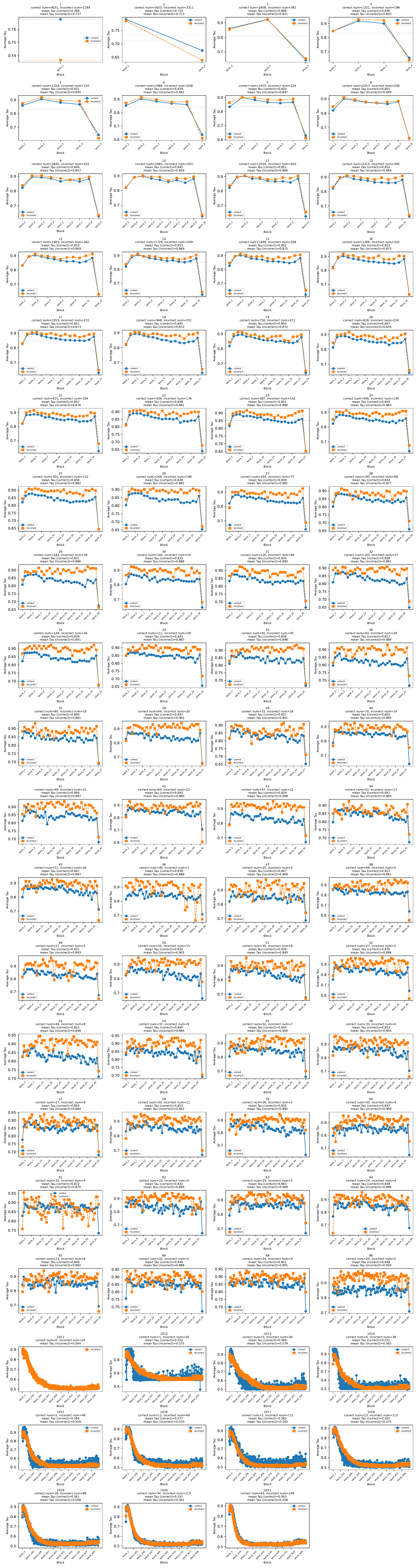} 
    \caption{Evolution of τ values for grouped samples across knowledge domains.}
    \label{app:know_tau}
\end{figure*}

\begin{figure*}
    \centering
    % height=\textheight 确保高度不超过正文区域
    % keepaspectratio 确保图片不会变形
    \includegraphics[width= \textwidth, height=0.9\textheight, keepaspectratio]{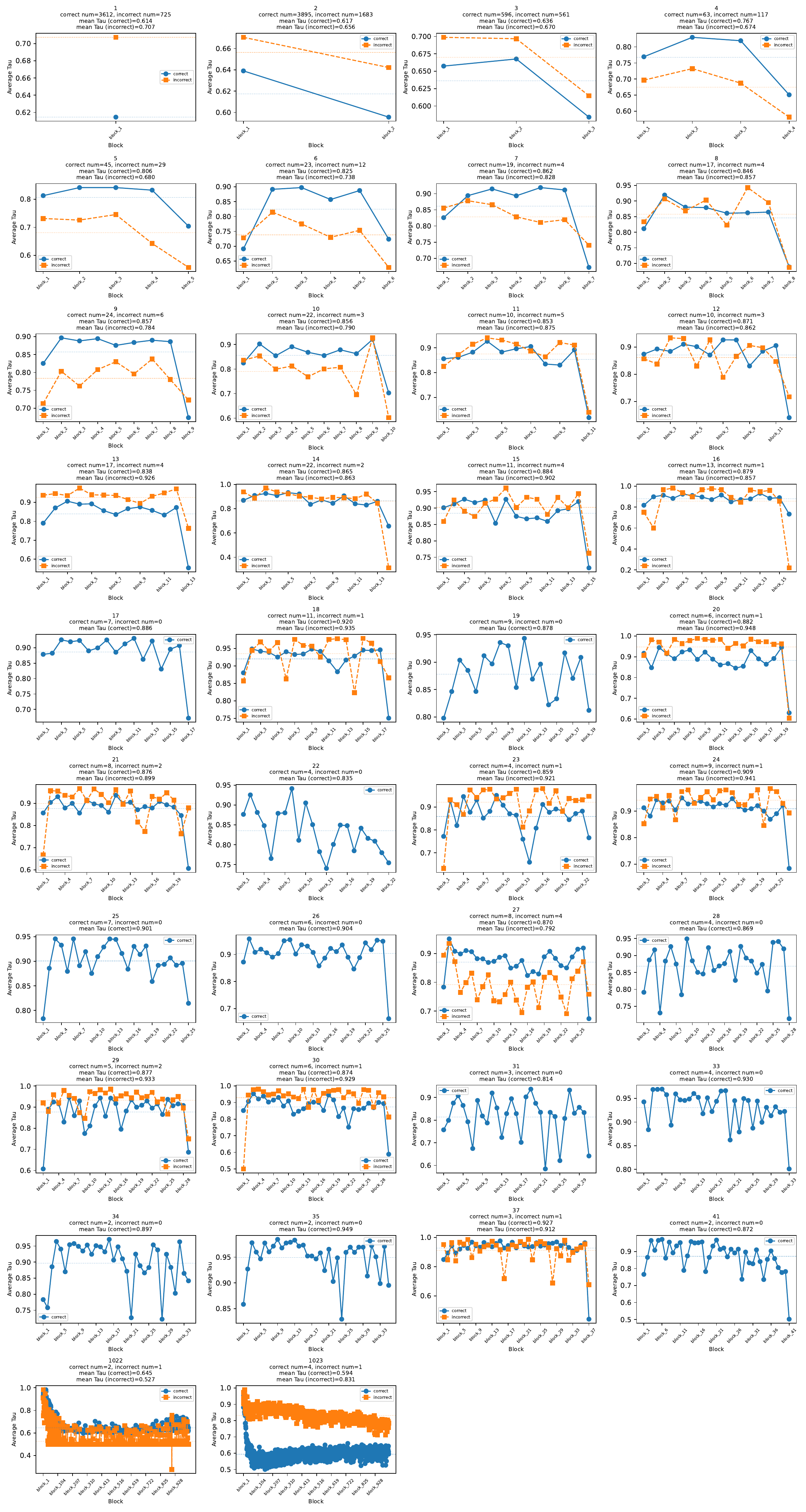} 
    \caption{Evolution of τ values for grouped samples across agent domains.}
    \label{app:agent_tau}
\end{figure*}

\begin{figure*}
    \centering
    % height=\textheight 确保高度不超过正文区域
    % keepaspectratio 确保图片不会变形
    \includegraphics[width=\textwidth, height=0.95\textheight, keepaspectratio]{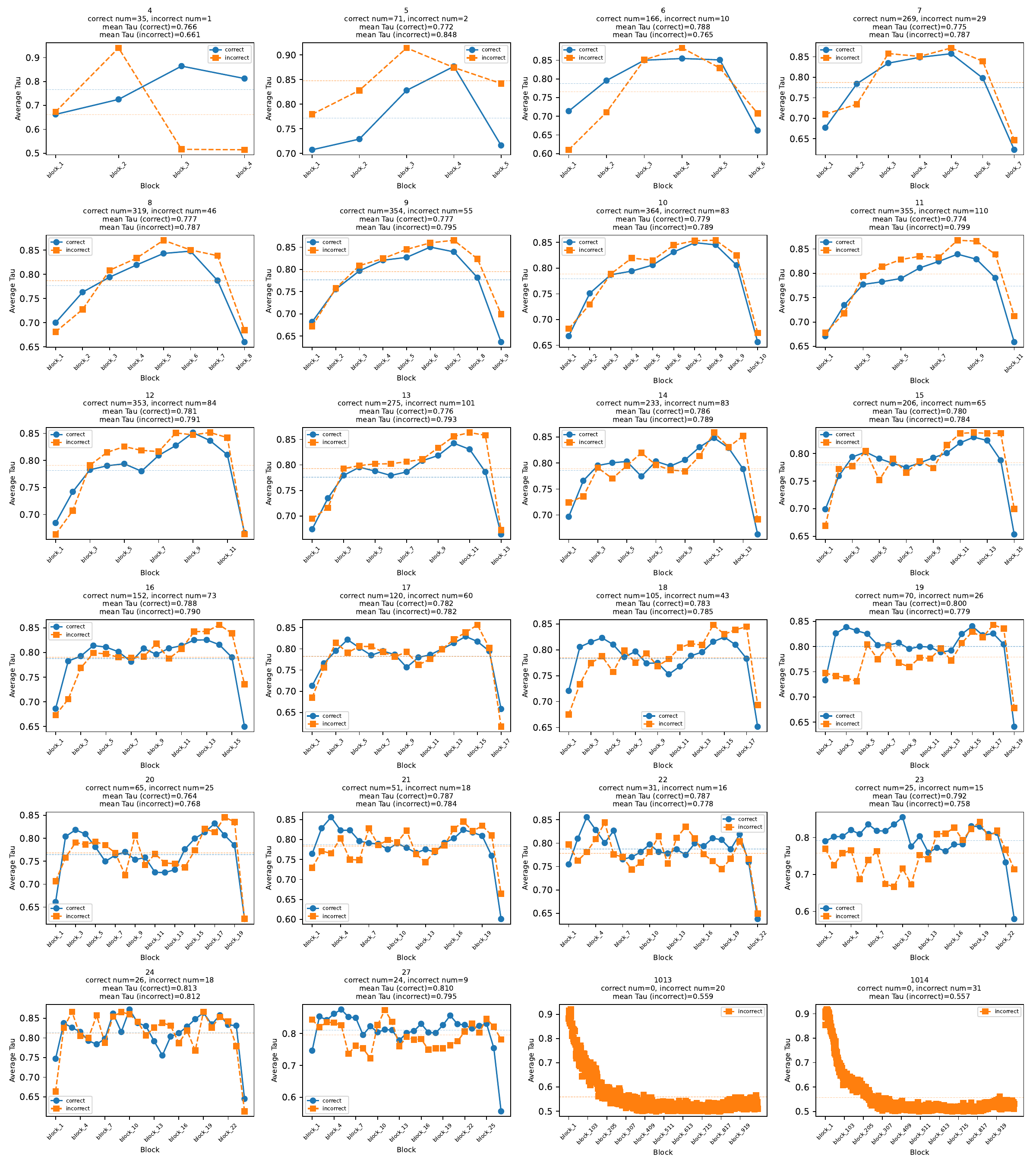} 
    \caption{Evolution of τ values for grouped samples across coding domains.}
    \label{app:coding_tau}
\end{figure*}

\begin{figure*}
    \centering
    % height=\textheight 确保高度不超过正文区域
    % keepaspectratio 确保图片不会变形
    \includegraphics[width=\textwidth, height=0.95\textheight, keepaspectratio]{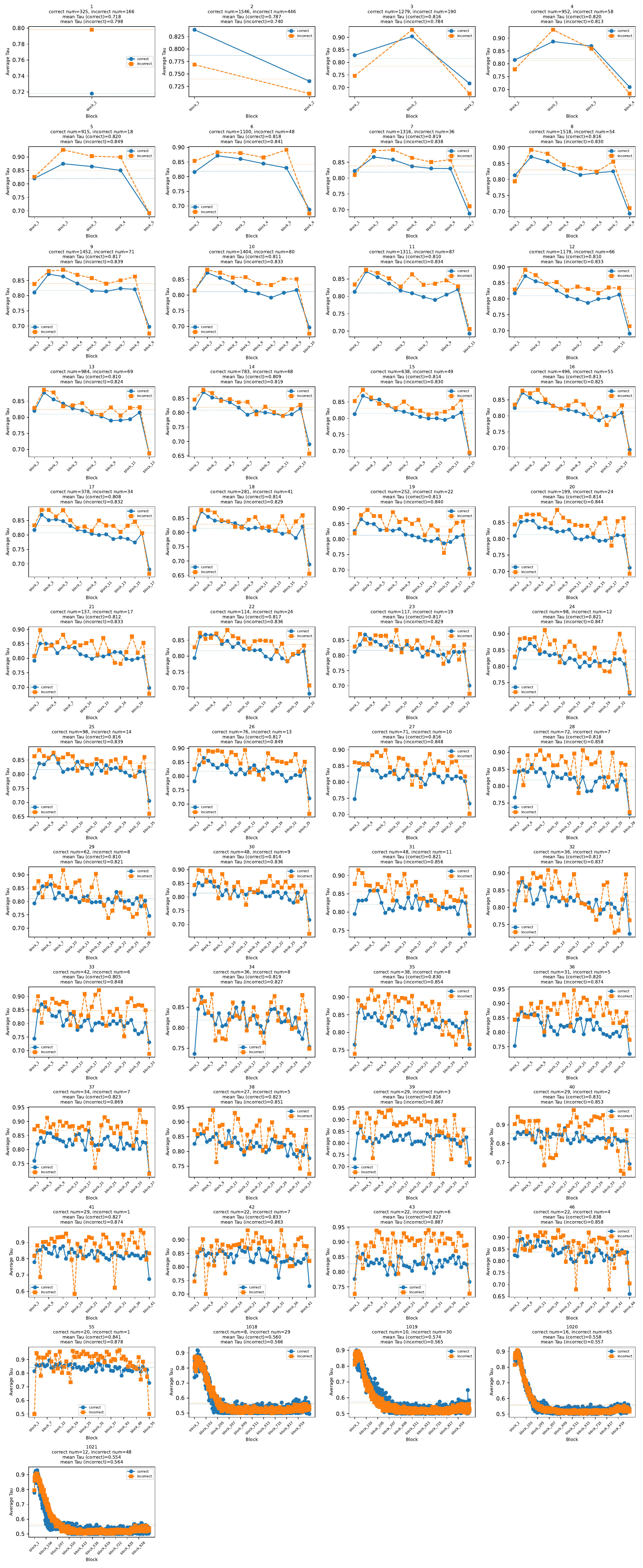} 
    \caption{Evolution of τ values for grouped samples across math domains.}
    \label{app:math_tau}
\end{figure*}

\begin{figure*}
    \centering
    % height=\textheight 确保高度不超过正文区域
    % keepaspectratio 确保图片不会变形
    \includegraphics[width=\textwidth, height=0.95\textheight, keepaspectratio]{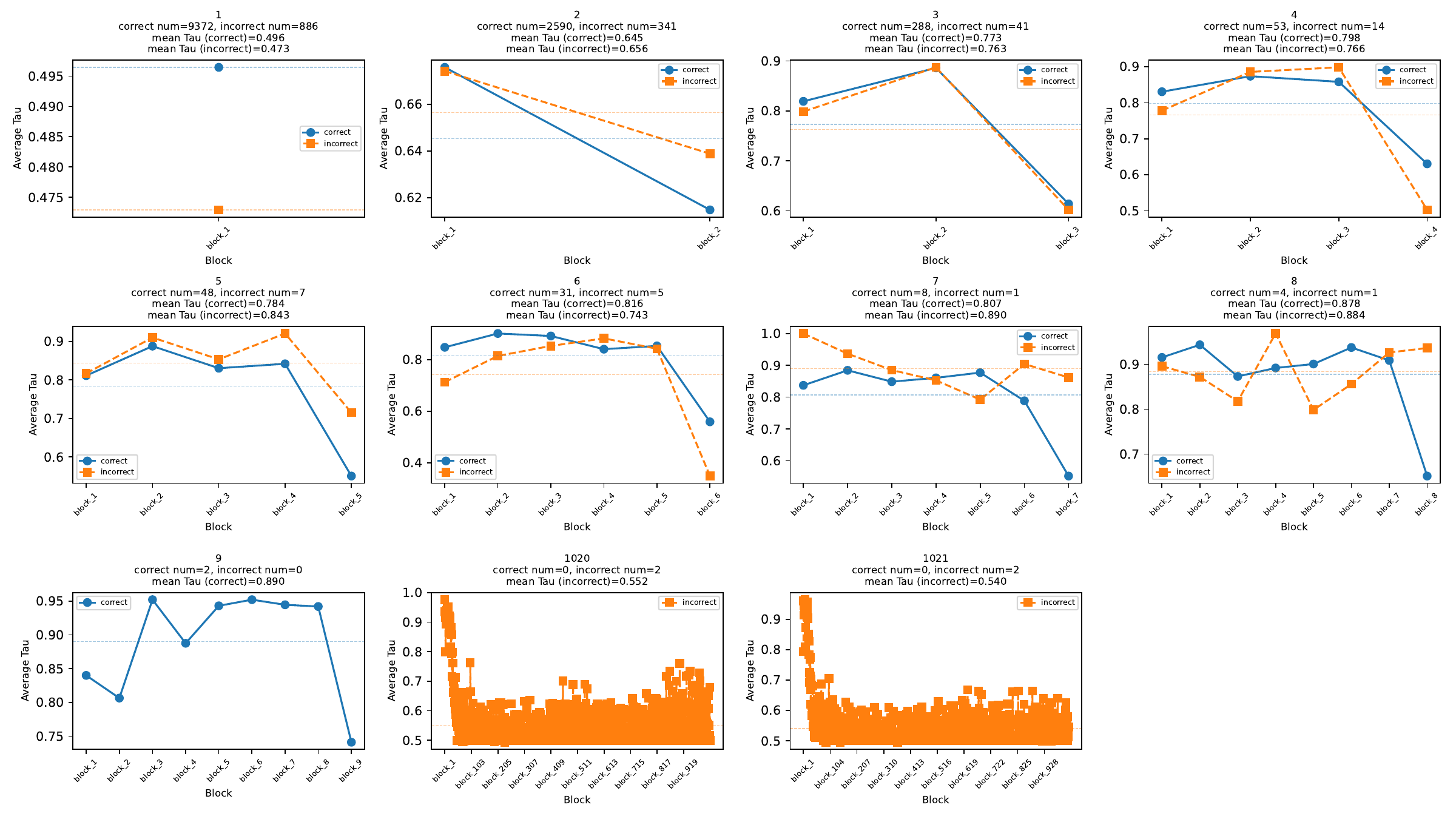} 
    \caption{Evolution of τ values for grouped samples across language understanding domains.}
    \label{app:nlu_tau}
\end{figure*}

\begin{figure*}
    \centering
    \includegraphics[width=\textwidth, height=0.95\textheight, keepaspectratio]{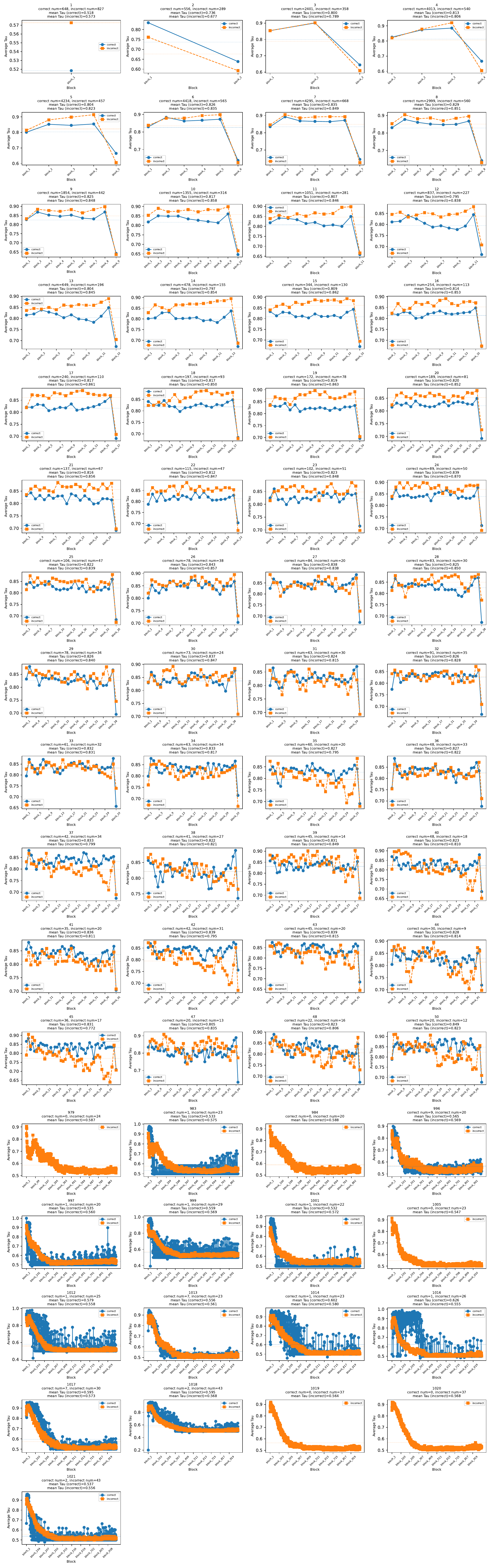}
    \caption{Evolution of τ values for grouped samples across reasoning domains.}
    \label{app:readsoning_tau}
\end{figure*}

\subsubsection{Dynamic variations of AFP values across six domains for positive and negative samples.}
\label{app:fap_block}

See \cref{app:agent_afp,app:code_afp,app:afp_know,app:afp_math,app:nlu_tau,afp_reason} for details.

\begin{figure*}
    \centering
    \includegraphics[width=\textwidth, height=0.9\textheight, keepaspectratio]{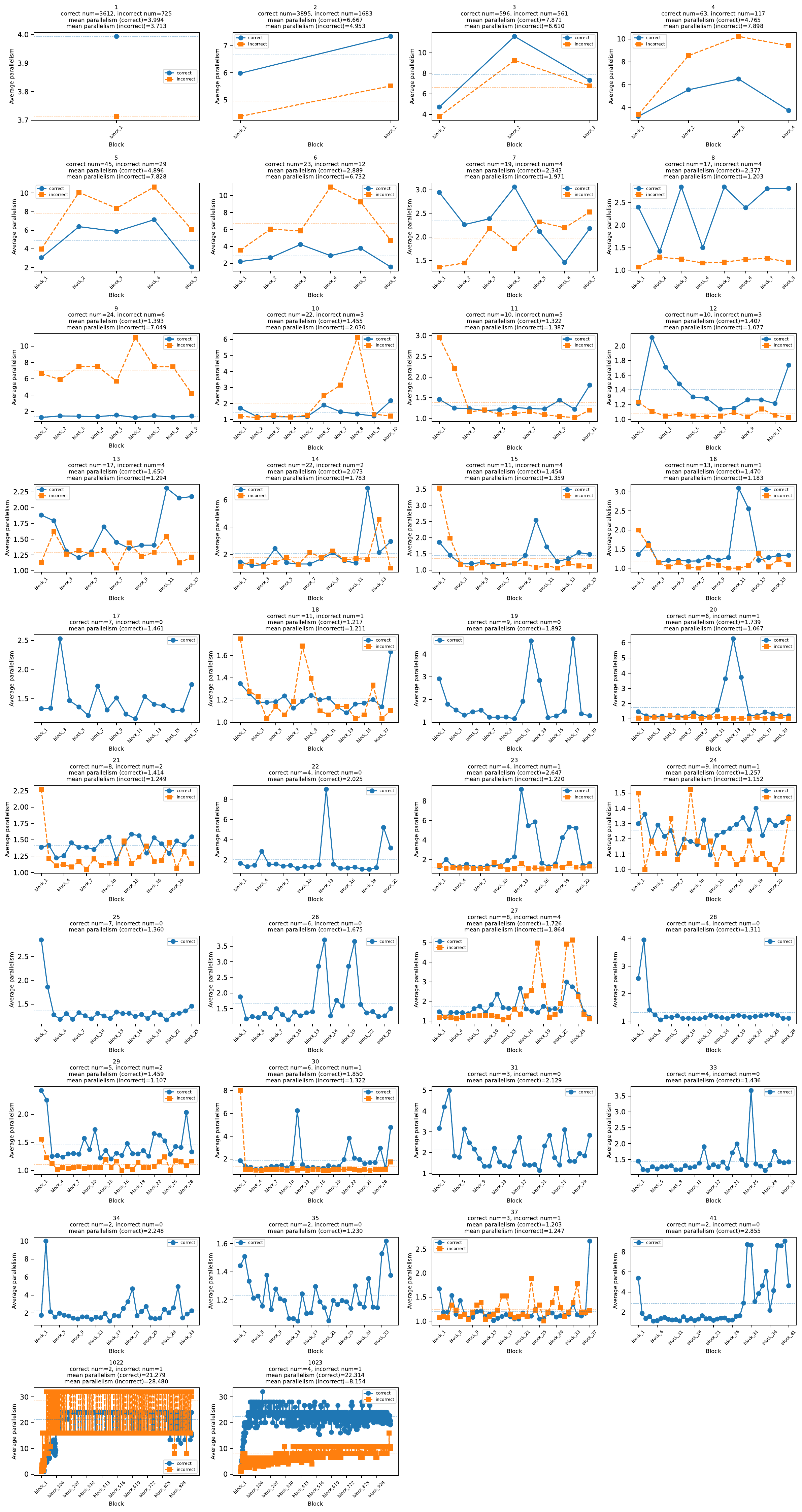}
    \caption{Evolution of afp values for grouped samples across agent domains.}
    \label{app:agent_afp}
\end{figure*}

\begin{figure*}
    \centering
    \includegraphics[width=\textwidth, height=0.9\textheight, keepaspectratio]{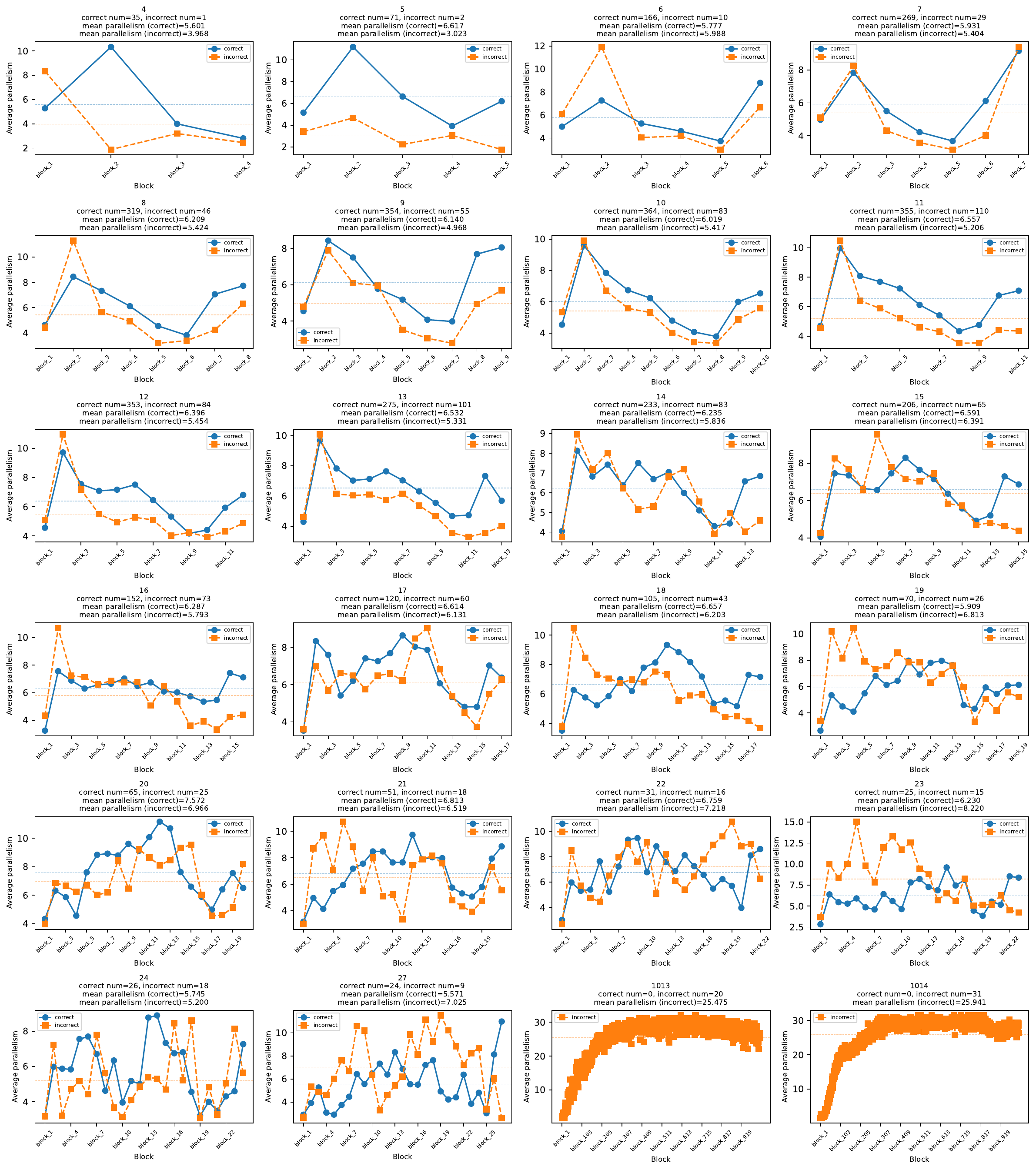}
    \caption{Evolution of afp values for grouped samples across code domains.}
    \label{app:code_afp}
\end{figure*}

\begin{figure*}
    \centering
    \includegraphics[width=\textwidth, height=0.9\textheight, keepaspectratio]{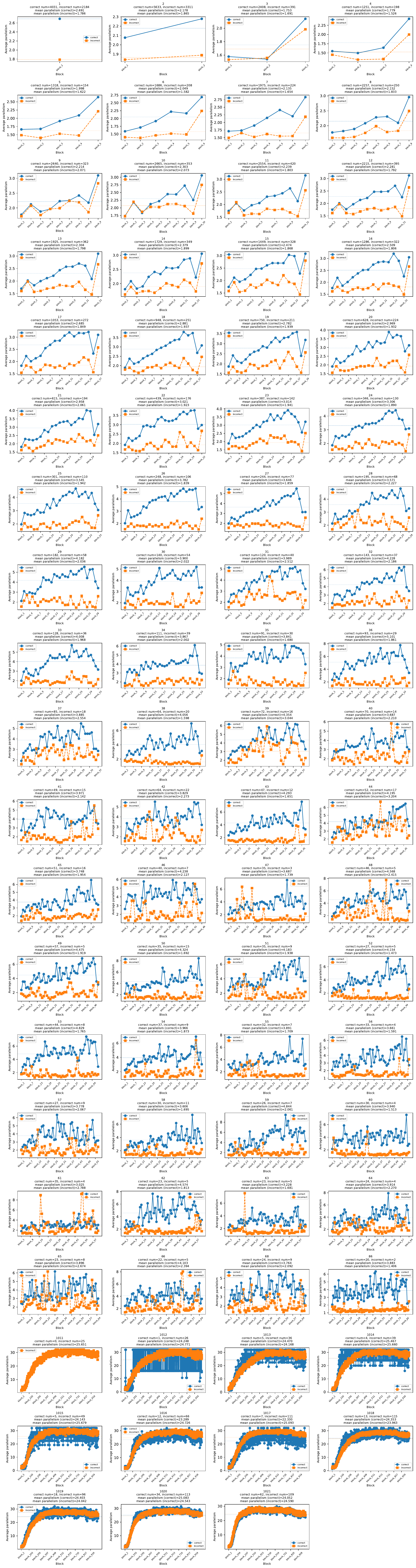}
    \caption{Evolution of afp values for grouped samples across knowledge domains.}
    \label{app:afp_know}
\end{figure*}

\begin{figure*}
    \centering
    \includegraphics[width=\textwidth, height=0.9\textheight, keepaspectratio]{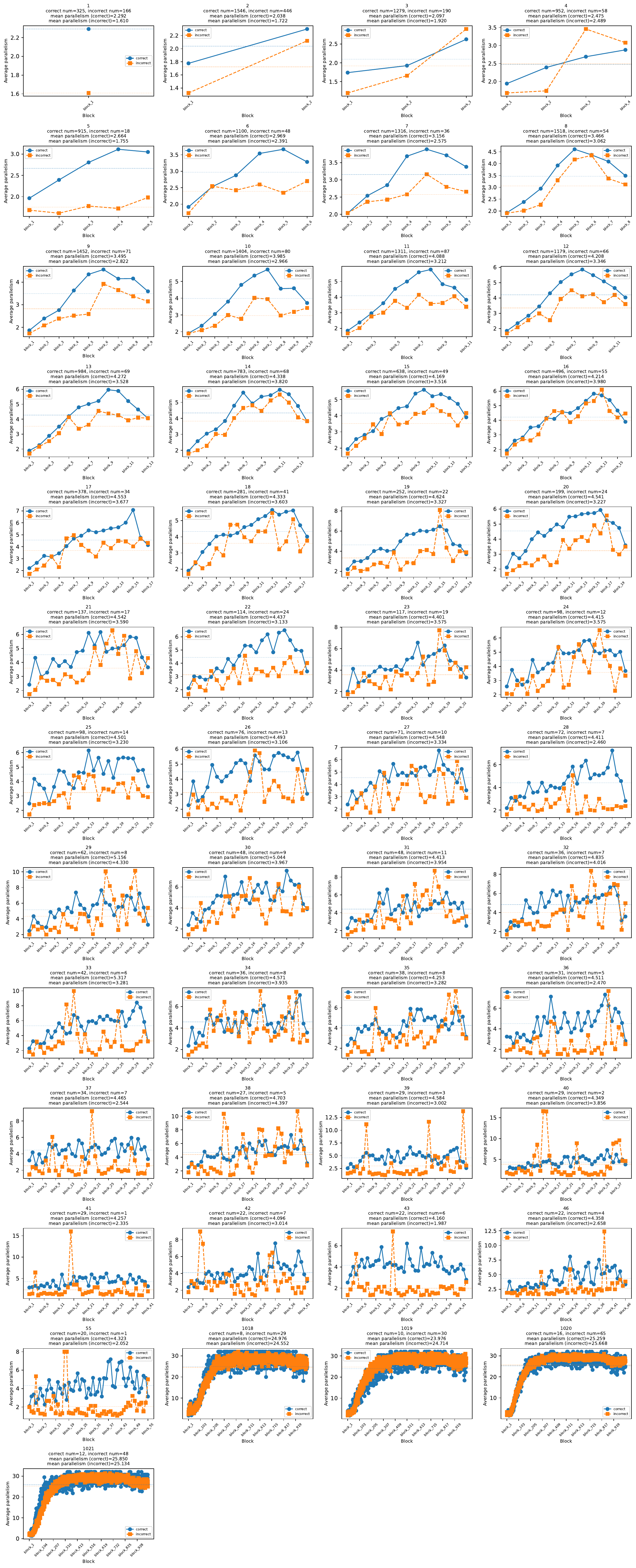}
    \caption{Evolution of afp values for grouped samples across math domains.}
    \label{app:afp_math}
\end{figure*}

\begin{figure*}
    \centering
    \includegraphics[width=\textwidth, height=0.9\textheight, keepaspectratio]{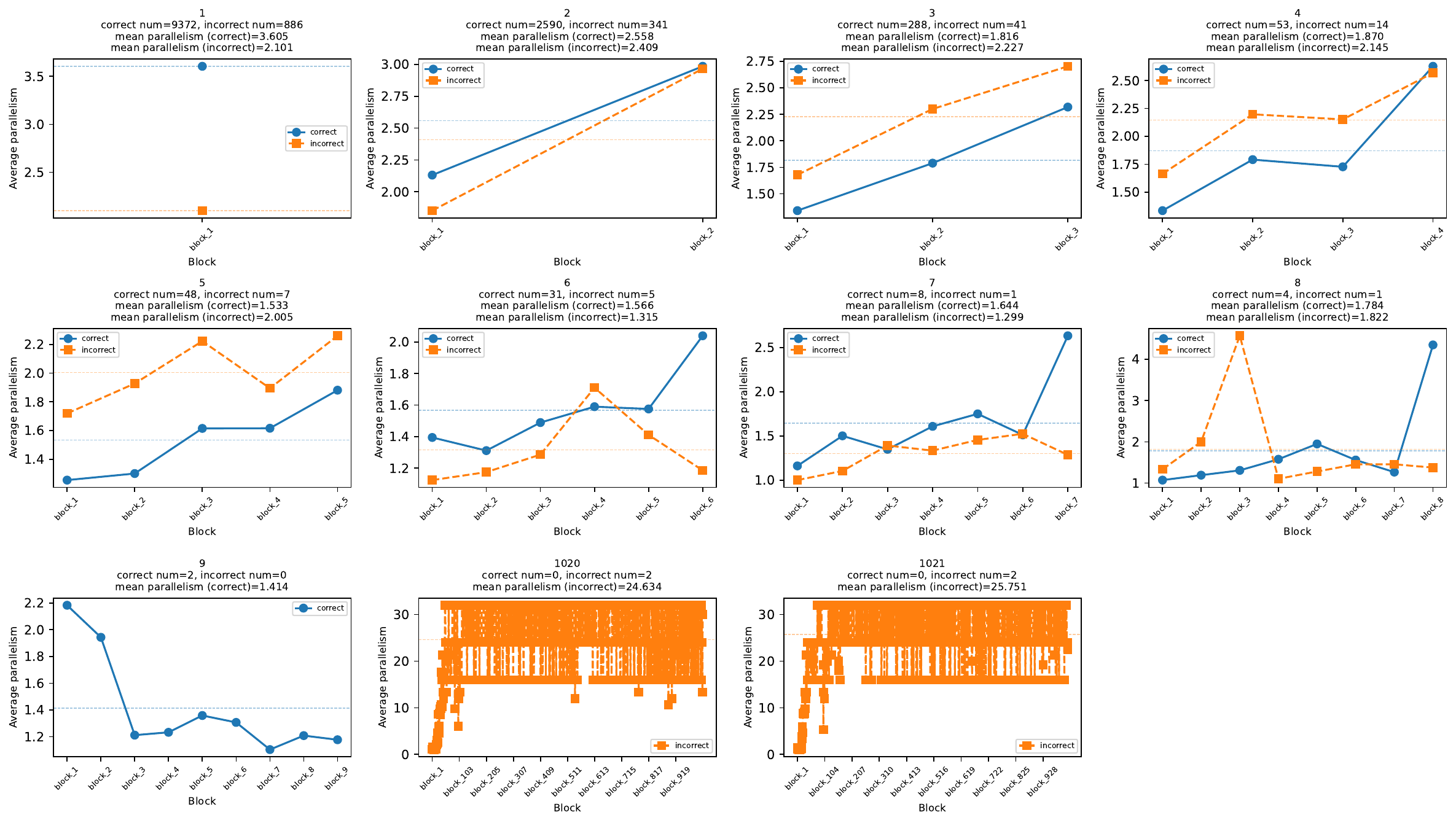}
    \caption{Evolution of afp values for grouped samples across language understanding domains.}
    \label{app:afp_nlu}
\end{figure*}

\begin{figure*}
    \centering
    \includegraphics[width=\textwidth, height=0.9\textheight, keepaspectratio]{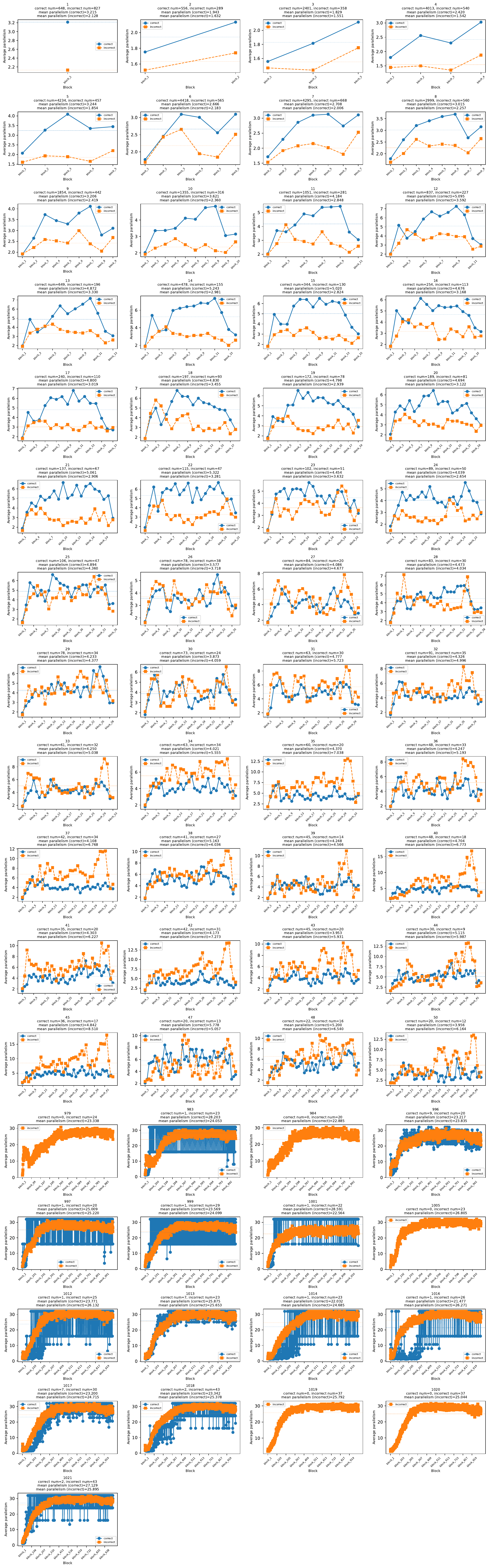}
    \caption{Evolution of afp values for grouped samples across reasoning domains.}
    \label{afp_reason}
\end{figure*}

\subsection{Details of Part-of-Speech(POS) Decoding Order}
\label{app:pos}
% \begin{table}[h]
%     \centering
%     \caption{The most frequent word combinations}
%     \label{tab:111}
%     \vspace{2mm}
    
%     % |c| 是第一列居中，|l| 是第二列左对齐，|p{8cm}| 是第三列固定宽度并自动换行
%     \begin{tabular}{|c|p{10cm}|}
%         \hline
%         \textbf{Benchmark} & \textbf{The most frequent combinations} \\ \hline
        
%         % K=2 分组
%         Agent         & ["\textbackslash n", "---", "\textbackslash n"], ["\textless\textbar endoftext\textbar\textgreater", "\textless\textbar role\_end\textless\textbar"], ["\&","\&","\textbackslash n","\&","~","\textbackslash n","\&","~","\textbackslash n","\&"] \\ \hline
%         Code          & ["\textless\textbar endoftext\textbar\textgreater", "\textless\textbar role\_end\textless\textbar"], ["```", "python"], ["\# ", "Background", ":"]  \\ \hline
%         Natural language understanding & [":", "A","\textless\textbar endoftext\textbar\textgreater", "\textless\textbar role\_end\textless\textbar"], ["refers", " to"], ["in", " the"], ["imp", "ossible"] \\ \hline
%         Knowledge     & ["The", " answer", " is"], ["\textbackslash n", "But"], ["of", " the"], ["Let", "'s", " analyze"] \\ \hline
%         Reasoning         & [".","\textbackslash n", "But"], ["1", " ", ":"], ["\textbackslash n", "\textbackslash n", "So"] \\ \hline
%         Math          & ["We"," are", " given"], ["\textless\textbar endoftext\textbar\textgreater", "\textless\textbar role\_end\textless\textbar"], [":", "\textbackslash n", "\$", "\$", "\textbackslash n"] \\ \hline
%     \end{tabular}
% \end{table}

Taking the knowledge benchmark as an example, we collected the Part-of-Speech (POS) sequences of the output content within each generation block. The results are shown in Table \ref{tab:pos_global_avg_step}. We observe that the model exhibits a distinct temporal order in its POS tags. 

First, tags typically associated with simple or independent content—such as \texttt{INTJ} (interjection), \texttt{X} (other), \texttt{NUM} (number), \texttt{PROPN} (proper noun), and \texttt{SYM} (symbol)—appear relatively early. This indicates that the model tends to initiate generation with simple tokens like numbers, names, and symbols. Second, structural cues also emerge early: \texttt{PUNCT} (punctuation) and \texttt{CCONJ} (coordinating conjunction) generally appear at lower average steps, implying that the model quickly inserts delimiters and connectors before drafting substantive content. The core substance of the sentences is then constructed during the intermediate steps, where \texttt{NOUN} (noun), \texttt{VERB} (verb), \texttt{PRON} (pronoun), and \texttt{AUX} (auxiliary) are most prevalent. Finally, tags that add supplementary details, such as \texttt{ADJ} (adjective), \texttt{ADV} (adverb), and \texttt{ADP} (adposition), tend to appear later, suggesting that the model typically appends descriptions and relational context once the main content is established. 

Notably, this sequence is not merely a reflection of tag frequency: some highly frequent tags (e.g., \texttt{NOUN} and \texttt{PUNCT}) are not the earliest to appear, while several less common tags emerge at the very beginning. This further supports the hypothesis that the model "builds the basic structure first, then populates the details."

\begin{table}[h]
\small
\centering
\caption{\textbf{Global POS tendency measured by average local step (knowledge benchmark).}For each generated block, we tag every token with a POS label and record its \emph{local step} (its position index inside the block, starting from 1).
We then compute the \textbf{Global Avg. Step} for each POS tag by averaging these local steps over all occurrences of this tag in the whole dataset.
A smaller average step means that this POS tag is more likely to appear earlier in a block, while a larger value means it tends to appear later.
\textbf{Total Count} is the total number of tokens assigned to this POS tag in our collected outputs.
For clarity, the tags are sorted by \textbf{Global Avg. Step} in ascending order.
Note that these numbers describe a position trend inside blocks, and they are not the same as the overall POS frequency in normal text.}
\label{tab:pos_global_avg_step}
\begin{tabular}{lrr}
\specialrule{0.9pt}{0pt}{2pt}
\textbf{POS Tag} & \textbf{Global Avg. Step} & \textbf{Total Count} \\
\specialrule{0.6pt}{2pt}{2pt}
INTJ  & 2.3418 & 437156 \\
X     & 2.7302 & 1135829 \\
NUM   & 3.3031 & 8419293 \\
PROPN & 3.6147 & 5188437 \\
SYM   & 3.8025 & 2816180 \\
CCONJ & 4.5384 & 1894660 \\
PUNCT & 4.5618 & 11005503 \\
AUX   & 4.7912 & 3089759 \\
NOUN  & 4.8039 & 15598091 \\
PRON  & 4.8151 & 2060850 \\
SCONJ & 4.8545 & 660041 \\
PART  & 4.9402 & 1099089 \\
DET   & 5.2555 & 3905146 \\
ADV   & 5.4416 & 2307803 \\
VERB  & 5.6032 & 3916307 \\
ADP   & 6.0675 & 2759438 \\
ADJ   & 6.2864 & 2922280 \\
\specialrule{0.9pt}{2pt}{0pt}
\end{tabular}
\end{table}

\clearpage
\onecolumn 
\subsection{Order Disruptions at Semantic Pivots}
\label{app:order_disruption}
\begin{center}
    \centering
    \includegraphics[width=0.9\textwidth, height=0.7\textheight, keepaspectratio]{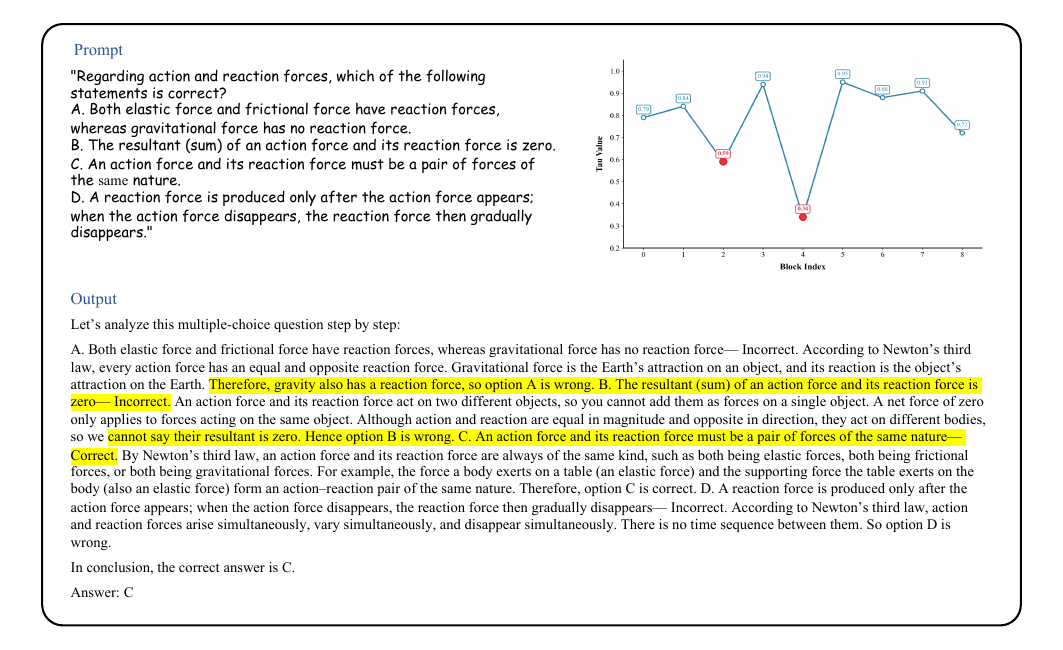}
   \captionof{figure}{\textbf{case study.}  This $\tau$ local-minimum output falls exactly at the boundary between ``closing the previous reasoning chain'' and ``opening the next one.'' The model first uses ``cannot say their resultant is zero. Hence option B is wrong.'' to simultaneously finalize B’s key rationale and its incorrect conclusion, and then immediately switches to ``C. \ldots{} Correct.'' to start a new evaluation. Because this moment involves multiple constraints---summarization, switching, structural labeling, and planning the next line of argument---the demand for parallel organization is greatest, so the curve reaches its lowest value here.}
    \label{fig:placeholder}
\end{center}

\begin{center}
    \centering
    \includegraphics[width=0.9\textwidth, height=0.7\textheight, keepaspectratio]{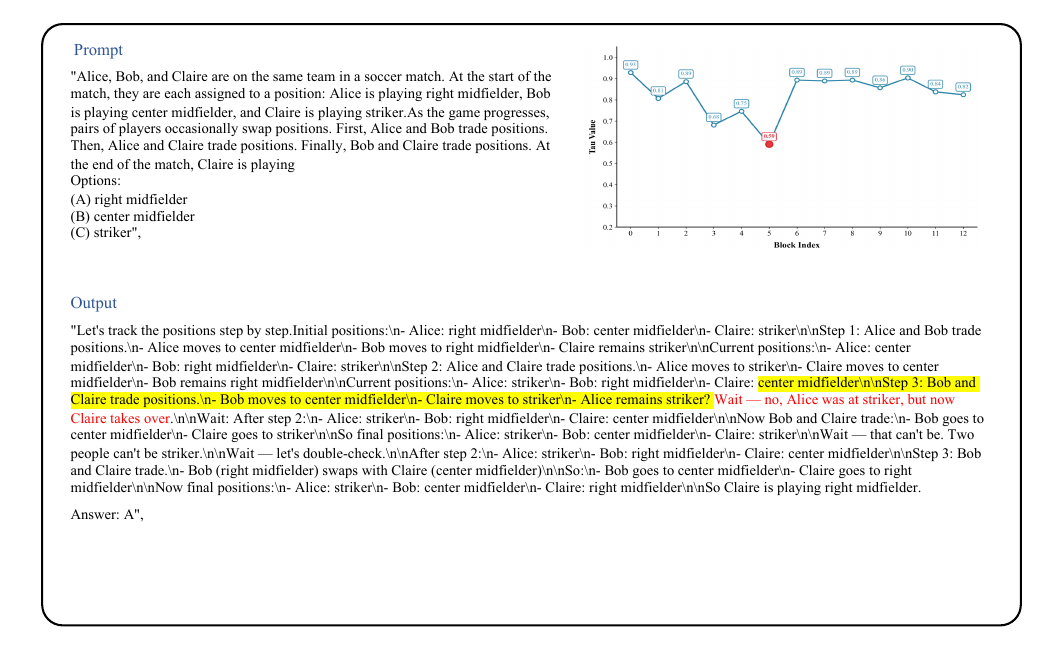}
    \captionof{figure}{\textbf{case study.} The minimum $\tau$ value occurs at the critical state-update point of the "third swap," where a single operation has to change two people’s positions at the same time. Moreover, step 3 is the final swap: right after writing these two updates, the model must immediately derive the final position assignment. Models therefore tend to initiate both "updating" and "checking whether it is reasonable / whether it will conflict" mechanisms simultaneously at this point. The later output—"Wait … / Two people can't be striker"—is precisely the explicit surfacing of that validation, and the minimum $\tau$ value often appears around the transition just before and after the check is triggered.}
    \label{fig:placeholder2}
\end{center}

\begin{center}
    \centering
    \includegraphics[width=0.9\textwidth, height=0.7\textheight, keepaspectratio]{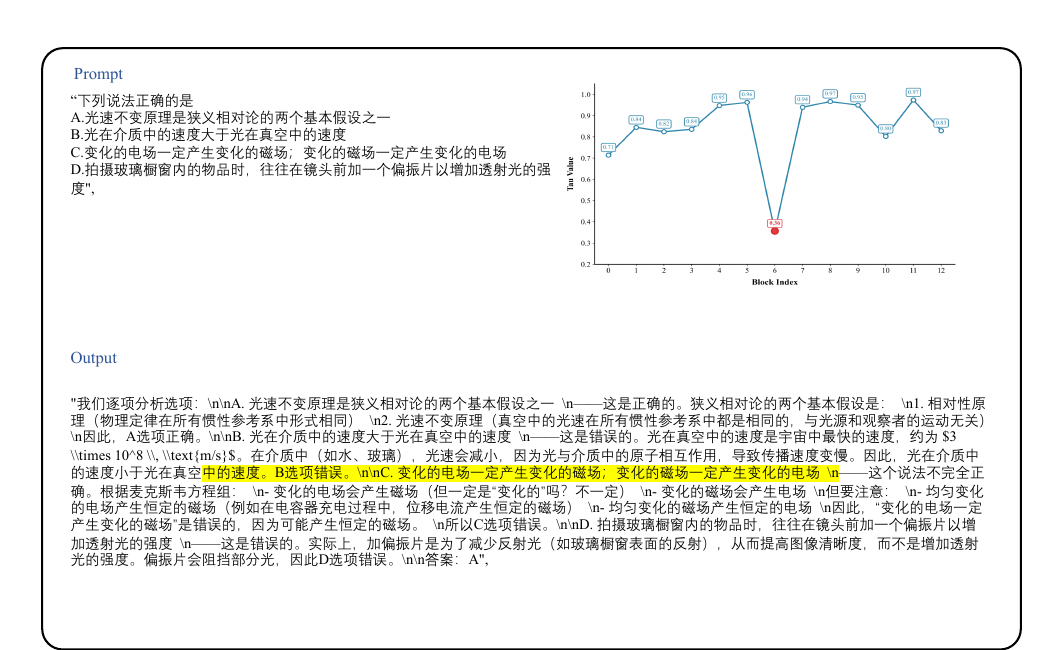}
    \captionof{figure}{\textbf{case study.}At the boundary where "B ends and C begins," the model’s output shows the highest degree of parallelism. At this point, the model must handle two things in parallel: on the one hand, it needs to finalize B's conclusion and bring the wording to a close; on the other hand, it needs to launch C's new argument framework (resetting the subject under discussion and the criteria for judgment). This kind of "state switching across paragraphs/across options" creates a stronger need for parallel organization.}
    \label{fig:placeholder}
\end{center}

\begin{center}
    \centering
    \includegraphics[width=0.9\textwidth, height=0.7\textheight, keepaspectratio]{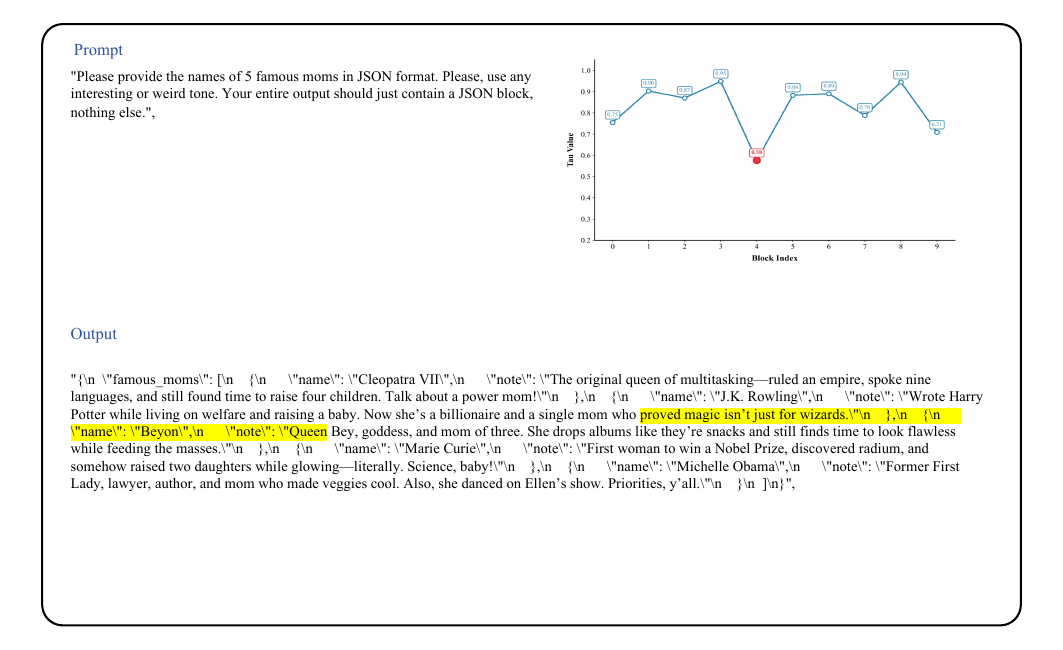}
     \captionof{figure}{\textbf{case study.} The minimum $\tau$ value occurs in this segment because it sits at the boundary between "the end of the previous item + the beginning of the next item." At the same moment, the model must handle in parallel the correctness of the JSON structure, the semantic wrap-up of the previous item, the planning of the next entry, and the entity switch—moving from Rowling’s concluding sentence to the start of Beyoncé’s entry, which involves both an entity change and stylistic continuity. Entity switching typically requires more parallel retrieval and planning, so the degree of parallelism is higher at the switch point.}
    \label{fig:placeholder}
\end{center}

\begin{center}
    \centering
    \includegraphics[width=0.9\textwidth, height=0.7\textheight, keepaspectratio]{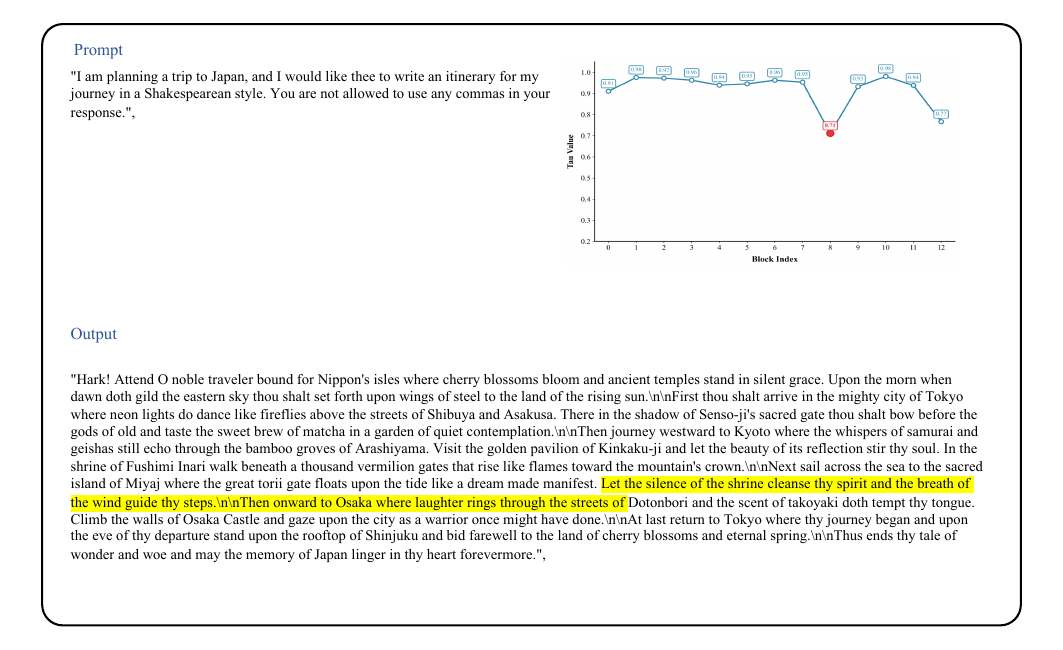}
    \captionof{figure}{\textbf{case study.} When the $\tau$ value reaches its minimum, within the same short span of text the model needs to (1) maintain a consistent writing style to deliver a summary, while (2) rapidly reset the semantic context and prepare the details to be written next. In effect, it is handling two generation goals at once—an ending and a beginning—and this forward-looking content organization manifests as a higher degree of parallel processing.}
    \label{fig:placeholder}
\end{center}

\begin{center}
    \centering
    \includegraphics[width=0.9\textwidth, height=0.7\textheight, keepaspectratio]{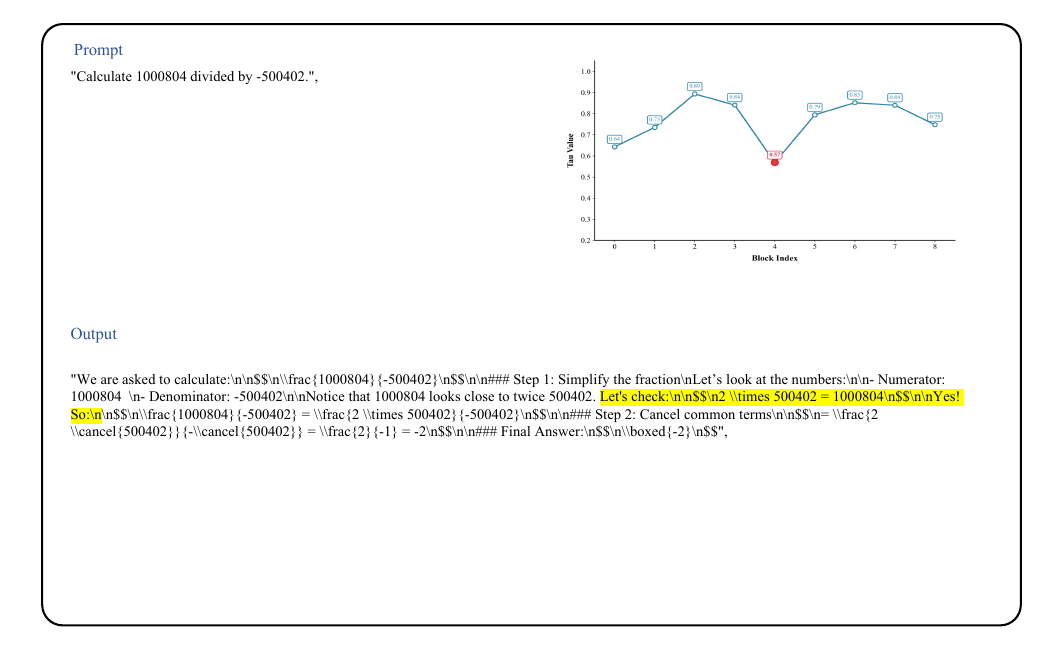}
   \captionof{figure}{\textbf{case study.} When the $\tau$ value reaches its minimum, the model must complete multiple tasks in parallel at the same moment: on the one hand, it performs and cross-checks the core computation; on the other hand, it organizes that computation into a clear proof chain using LaTeX formulas and transitional phrasing. At the same time, while outputting the equality check, it already plans the subsequent simplification route. In addition, it has to satisfy numerical correctness, consistency in symbols and typesetting, and a narrative rhythm that smoothly bridges what comes before and after. Because this point is both a checkpoint for “whether it holds” and a starting point for “how to simplify,” it functions as a key pivot in the solution, prompting the model to activate verification and simplification in parallel.}
    \label{fig:placeholder}
\end{center}
\clearpage 
\twocolumn 

\subsection{Experimental Study of Any-Order Inference in Sudoku, Molecular Design, and Protein Design}
\label{app:Sudoku}

\subsubsection{Sudoku}
\begin{strip}
    \centering
    \vspace{10pt} 
    \includegraphics[width=\textwidth, height=0.4\textheight, keepaspectratio]{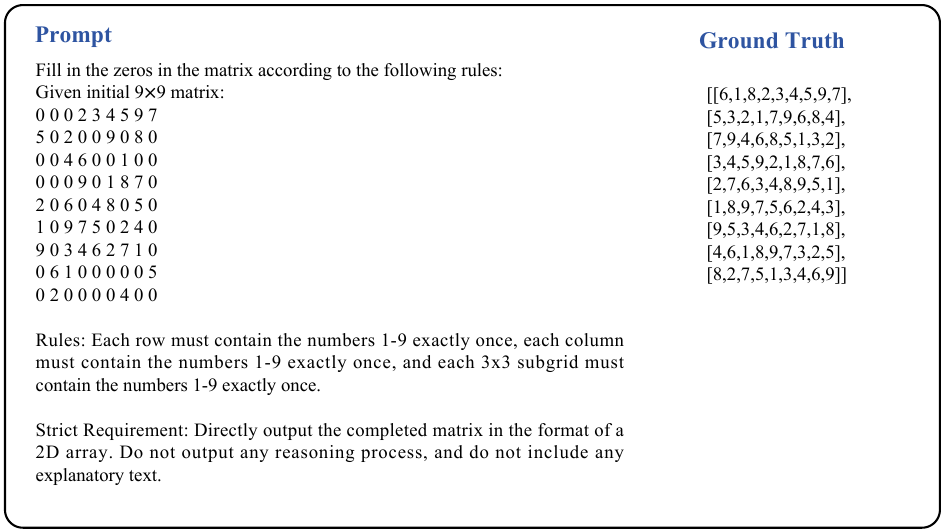} 
    \captionof{figure}{Sudoku Case.}
    \label{app:Sudoku_case}
    \vspace{10pt}
\end{strip}
\paragraph{Data Description} 
Sudoku is a representative Constraint Satisfaction Problem (CSP) that requires the model to maintain global consistency across rows, columns, and $3 \times 3$ sub-grids. We utilized a script to randomly generate 150 unique $9 \times 9$ Sudoku puzzles with varying difficulty levels. The problem format, as illustrated in Figure \ref{app:Sudoku_case}, presents a grid with missing digits (represented by zeros or specific placeholders), and the model is required to output the fully completed $9 \times 9$ matrix.

\paragraph{Experimental Setup}
To evaluate the learning efficiency and reasoning capabilities of different architectures, we partitioned the 150 puzzles into a training set of 50 samples and a test set of 100 samples. We compared two primary architectures: (1) \textbf{Autoregressive (AR) models}, including Qwen3-80B-next (Zero-shot) and Qwen3-8B (Fine-tuned); and (2) \textbf{Diffusion-based models}, including LLaDA-flash-100B (Zero-shot) and Dream-7B (Fine-tuned). All fine-tuning experiments were conducted using the models' respective official training frameworks. The evaluation metric is the accuracy score based on the number of correctly solved grids out of 100 test cases.

\paragraph{Results and Analysis}
The experimental results, summarized in Table \ref{tab:sudoku_results}, reveal several critical insights into the relationship between model architecture and logical reasoning efficiency:
\begin{table*}
\centering
\caption{Sudoku Solving Performance: Comparison between Diffusion-based and Autoregressive Models. (Training set: 50 samples, Test set: 100 samples)}
\label{tab:sudoku_results}
\begin{tabular}{lcccc}
\toprule
\textbf{Model} & \textbf{Architecture} & \textbf{Params} & \textbf{Epochs} & \textbf{Score (Acc.)} \\
\midrule
\rowcolor[HTML]{F3F3F3} 
\textit{Zero-shot Baselines} & & & & \\
LLaDA-flash-100B & Block-wise Diffusion & 100B & 0 & 78 \\
Qwen3-80B-next   & Autoregressive   & 80B  & 0 & 63 \\
\midrule
\rowcolor[HTML]{ECF4FF} 
\textit{Fine-tuned Models} & & & & \\
Dream-7B  & Full-attention Diffusion & 7B & 0 & 9 \\
                &                    &    & 2 & 31 \\
                &                    &    & 5 & 65 \\
                &                    &    & 10 & \textbf{80} \\
\cmidrule{2-5}
Qwen3-8B        & Autoregressive     & 8B & 0 & 0 \\
                &                    &    & 2 & 0 \\
                &                    &    & 10 & 15 \\
                &                    &    & 20 & 17 \\
                &                    &    & 50 & 55 \\
\bottomrule
\end{tabular}
\end{table*}

\begin{enumerate}
    \item \textbf{Architectural Superiority of Diffusion Models:} Despite having significantly fewer parameters, the fine-tuned \textbf{Dream-7B} (Diffusion) achieved a score of \textbf{80} within only 10 epochs, surpassing even the zero-shot performance of the \textbf{100B-parameter LLaDA} (78). This suggests that the full-attention mechanism in diffusion models is inherently more compatible with the bidirectional constraints of Sudoku than the causal masking used in AR models.
    
    \item \textbf{Sample Efficiency and Convergence:} Dream-7B demonstrated remarkable sample efficiency. With only 50 training examples, it reached a score of 65 at epoch 5. In contrast, \textbf{Qwen3-8B} (AR) exhibited much slower convergence, remaining at a score of 0 for the first few epochs and only reaching a score of 55 after 50 epochs. This disparity highlights that AR models struggle to "look ahead" or maintain global consistency without extensive training or explicit Chain-of-Thought (CoT) prompting.
    
    \item \textbf{Paradigm Shift in Low-Data Scenarios:} The fact that a 7B diffusion model can outperform an 80B AR model (63 score) and a 100B diffusion baseline underscores that for tasks requiring parallel constraint satisfaction, \textbf{architecture outweighs scale}. Diffusion models treat Sudoku solving as an iterative refinement process of the entire grid, whereas AR models are forced to solve it as a sequential prediction task, which is susceptible to error propagation.
\end{enumerate}

\subsubsection{Cross-math puzzle}
\begin{figure*}
    \centering
    \includegraphics[width=0.8\textwidth, height=0.8\textheight, keepaspectratio]{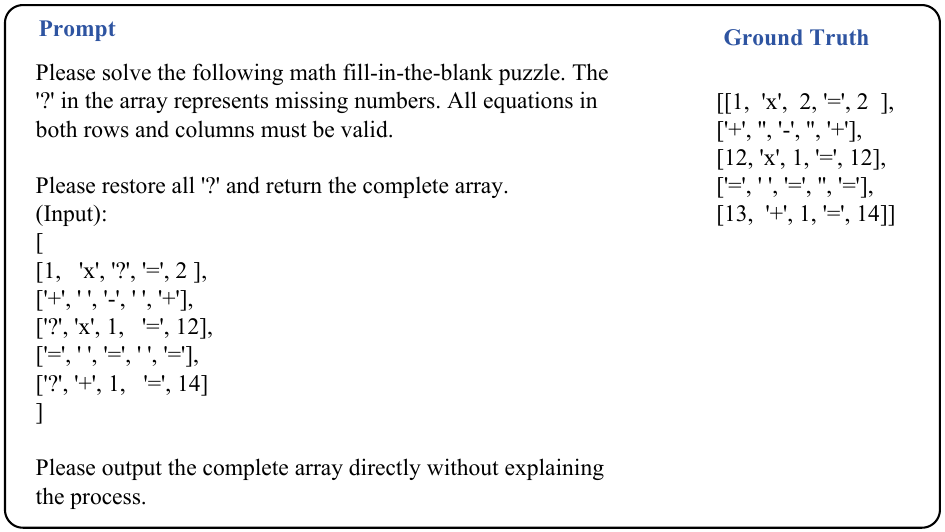} 
    \caption{\textbf{Cross-math puzzle Case.} The task requires filling placeholders to satisfy three horiWe evaluated both Diffusion-based and Autoregressive (AR) architectures in zero-shot and fine-tuned settings. For zero-shot baselines, we included massive models such as Qwen3-next-80B and LLaDA-mini-16B. For the fine-tuning experiments, we utilized 50 training samples and evaluated the models at multiple checkpoints (Epochs 0 to 50). All models were fine-tuned using their official training frameworks with standardized hyperparameters to ensure a fair comparison of architectural efficiency in low-data scenarios.zontal and three vertical equations simultaneously.}
    \label{fig:Cross-math puzzle——case}
\end{figure*}

\begin{table*}
\centering
\caption{Performance Comparison on the \textsc{Cross-math puzzle} Task. Scores represent accuracy on the 100-sample test set.}
\label{tab:math_matrix_results}
\small
\begin{tabular}{lcccc}
\toprule
\textbf{Model} & \textbf{Architecture} & \textbf{Params} & \textbf{Epochs} & \textbf{Score (Acc.)} \\
\midrule
\rowcolor[HTML]{F3F3F3} 
\textit{Zero-shot Baselines} & & & & \\
Qwen3-next-80B      & Autoregressive   & 80B  & 0 & 26.32 \\
LLaDA-mini-16B      & Diffusion        & 16B  & 0 & 15.79 \\
Trado-8B-instruct   & Diffusion   & 8B   & 0 & 14.74 \\
Qwen2.5-7B-instruct & Autoregressive   & 7B   & 0 & 6.32 \\
\midrule
\rowcolor[HTML]{ECF4FF} 
\textit{Fine-tuned Models} & & & & \\
\textbf{LLaDA-8B-instruct} & \textbf{Diffusion} & 8B & 0 & 4 \\
                           &                    &    & 2 & 36 \\
                           &                    &    & 5 & \textbf{42} \\
\cmidrule{2-5}
\textbf{Dream-7B}          & \textbf{Diffusion} & 7B & 0 & 0 \\
                           &                    &    & 2 & 18 \\
                           &                    &    & 3 & 26 \\
                           &                    &    & 5 & 40 \\
\cmidrule{2-5}
Qwen3-8B                   & Autoregressive     & 8B & 0 & 0 \\
                           &                    &    & 2 & 0 \\
                           &                    &    & 3 & 0 \\
                           &                    &    & 5 & 0 \\
                           &                    &    & 10 & 25 \\
\bottomrule
\end{tabular}
\end{table*}

% --- 表格 1: 化学公式生成 ---
\begin{table*}[htbp]
    \centering
    \caption{Performance on Chemical Formula Generation Tasks}
    \label{tab:chemical_results}
    \resizebox{\textwidth}{!}{ % 如果表格太宽，这行会自动缩放到页面宽度
    \begin{tabular}{llccccccc}
        \toprule
        \textbf{Task} & \textbf{Method} & \textbf{Exact $\uparrow$} & \textbf{BLEU $\uparrow$} & \textbf{Levenshtein $\downarrow$} & \textbf{RDK FTS $\uparrow$} & \textbf{MACCS FTS $\uparrow$} & \textbf{Morgan FTS $\uparrow$} & \textbf{Validity $\uparrow$} \\
        \midrule
        \multirow{2}{*}{Description-guided Mol Gen} & Qwen3-8B & 0.006 & 0.554 & 34.984 & 0.273 & 0.465 & 0.194 & 0.313 \\
         & DiRL-8B (bl=64) & 0.017 & 0.748 & 37.090 & 0.421 & 0.595 & 0.262 & 0.524 \\
        \midrule
        \multirow{2}{*}{Forward Reaction Pred} & Qwen3-8B & 0.195 & 0.822 & 14.131 & 0.620 & 0.739 & 0.564 & 0.793 \\
         & DiRL-8B (bl=64) & 0.309 & 0.939 & 16.670 & 0.619 & 0.780 & 0.575 & 0.999 \\
        \midrule
        \multirow{2}{*}{Reagent Prediction} & Qwen3-8B & 0.005 & 0.446 & 30.788 & 0.190 & 0.325 & 0.160 & 0.575 \\
         & DiRL-8B (bl=64) & 0.043 & 0.514 & 25.730 & 0.327 & 0.453 & 0.284 & 0.989 \\
        \midrule
        \multirow{2}{*}{Retrosynthesis} & Qwen3-8B & 0.148 & 0.876 & 15.858 & 0.594 & 0.730 & 0.564 & 0.628 \\
         & DiRL-8B (bl=64) & 0.160 & 0.905 & 21.609 & 0.600 & 0.753 & 0.549 & 0.997 \\
        \bottomrule
    \end{tabular}
    }
\end{table*}

\paragraph{Data Description}
To further evaluate the models' proficiency in integrating numerical reasoning with structural consistency, we introduce the \textsc{MathMatrix} task. This task represents a more complex Constraint Satisfaction Problem (CSP) than Sudoku, as it requires both algebraic calculation and bidirectional logical alignment.

\begin{itemize}[leftmargin=*]
    \item \textbf{Task Structure:} Each instance consists of a $5 \times 5$ grid representing a system of interlocking arithmetic equations. As shown in the examples (see Fig.~\ref{fig:Cross-math puzzle——case}), the 1st, 3rd, and 5th rows and columns constitute horizontal and vertical equations (e.g., $A \times B = C$), while the 2nd and 4th rows and columns contain the operators (e.g., $+$, $-$, $\times$) that link these equations.
    \item \textbf{Constraint Mechanism:} Unlike Sudoku, which relies on set-based exclusion logic, \textsc{MathMatrix} requires precise arithmetic operations. A valid solution must ensure that all six equations (three horizontal and three vertical) are satisfied simultaneously. This creates a high-degree of coupling: changing a single digit may invalidate both its row-wise and column-wise expressions.
    \item \textbf{Data Partitioning:} We programmatically generated 150 unique \textsc{MathMatrix} puzzles with varying arithmetic complexity. Following the same experimental protocol as the Sudoku task, we partitioned the data into a training set of 50 samples and a test set of 100 samples to assess the models' few-shot learning and generalization capabilities.
\end{itemize}

\paragraph{Experimental Setup}
We evaluated both Diffusion-based and Autoregressive (AR) architectures in zero-shot and fine-tuned settings. For zero-shot baselines, we included massive models such as Qwen3-next-80B and LLaDA-mini-16B. For the fine-tuning experiments, we utilized 50 training samples and evaluated the models at multiple checkpoints . All models were fine-tuned using their official training frameworks with standardized hyperparameters to ensure a fair comparison of architectural efficiency in low-data scenarios.

\paragraph{Results and Analysis}

The experimental results in Table \ref{tab:math_matrix_results} reveal several critical findings regarding architectural performance on highly-coupled logical tasks:

\begin{enumerate}[leftmargin=*]
    \item \textbf{Rapid Convergence of Diffusion Models:} Diffusion-based models exhibited an explosive growth in accuracy during the early stages of fine-tuning. \textbf{LLaDA-8B-instruct} improved from a zero-shot score of 4 to 36 within only 2 epochs, eventually reaching \textbf{42} at epoch 5. Similarly, \textbf{Dream-7B} reached a score of 40 at epoch 5. Remarkably, both 7B-8B diffusion models outperformed the \textbf{80B-parameter Qwen3} (26.32) after just 5 epochs of training on a tiny 50-sample dataset.
    
    \item \textbf{The Bottleneck of Causal Masking:} In contrast, the autoregressive \textbf{Qwen3-8B} failed to achieve any correct solutions (Score=0) for the first 5 epochs. This suggests that the unidirectional nature of causal masking makes it extremely difficult for the model to learn interlocking arithmetic constraints. Qwen3-8B only began to converge at epoch 10 (Score=25), still trailing behind the diffusion models that were trained for significantly fewer iterations.
    
\end{enumerate}

\subsubsection{Molecular Design and Protein Design}

\textbf{Experimental Setup and Results}. Using the Mol-instructions \cite{fang2023mol} dataset, we selected four chemical formula generation tasks and one protein design task. All models were evaluated using identical training data and evaluation protocols. The results of the chemical formula generation tasks are summarized in Table~\ref{tab:chemical_results}, while the performance for the protein design task is presented in Table~\ref{tab:protein_results}.

% --- 表格 2: 蛋白质设计 ---
\begin{table}[htbp]
    \centering
    \caption{Performance on Protein Design Task}
    \label{tab:protein_results}
    \begin{tabular}{llc}
        \toprule
        \textbf{Task} & \textbf{Model} & \textbf{NSW $\uparrow$} \\
        \midrule
        \multirow{3}{*}{Protein Design} & Qwen3-8B & 0.0108 \\
         & dream7b & 0.0861 \\
         & DiRL-8B (bl=64) & 0.1200 \\
        \bottomrule
    \end{tabular}
\end{table}

\noindent\textbf{Analysis of Results}. We compared two diffusion models (dream7b and DiRL-8B) with a state-of-the-art autoregressive (AR) model (Qwen3-8B). The results in Table~\ref{tab:chemical_results} and Table~\ref{tab:protein_results} confirm that DLMs possess a clear advantage in tasks governed by global sequence constraints. Taking the protein design task (Table~\ref{tab:protein_results}) as a representative case:

\begin{itemize}
    \item \textbf{Spatial Proximity vs. Sequence Distance:} Although proteins are represented as 1D amino acid sequences, their biological function is determined by their 3D folded structure. Two residues far apart in the 1D sequence (e.g., positions 10 and 200) may be spatially adjacent in 3D space, forming a critical active site.
    
    \item \textbf{Limitations of AR Models:} Autoregressive models generate tokens in a strict left-to-right order. When predicting the 100th residue, the model can only consider residues 1--99; it is fundamentally unable to incorporate the ``future'' constraints of the 200th residue, which has not yet been generated.
    
    \item \textbf{Advantages of Diffusion Models:} DLMs maintain a global receptive field during every denoising step. This allows the model to simultaneously consider the entire sequence context, enabling it to fix key structural ``anchors'' at both ends of a sequence before filling in the middle, or to generate a functional motif first and then ``diffuse'' outwards to create the supporting scaffold. This mechanism aligns perfectly with the physical reality of global structural constraints in biomolecules.
\end{itemize}

\subsection{Preliminary Experimental Validation of Editing Capabilities}
\label{app:editing_preliminary_results}

We implemented and evaluated the proposed editing mechanism on LLaDA2.0-mini-16B. Preliminary results demonstrate its significant potential in enhancing generation quality and accelerating inference throughput.

\subsubsection{Mechanism Design}
\begin{table*}[h]
\centering
\caption{Performance comparison of LLaDA2-mini-16B under different editing configurations.}
\label{tab:editing_results}
\resizebox{\textwidth}{!}{
\begin{tabular}{lllcccccccccccc}
\toprule
\textbf{Model} & \textbf{Thr.} & \textbf{Metric} & \textbf{Avg.} & \textbf{GPQA} & \textbf{SQ2} & \textbf{DROP} & \textbf{CruO} & \textbf{MBPP} & \textbf{HEval} & \textbf{LCB} & \textbf{GSM} & \textbf{MATH} & \textbf{OlB} & \textbf{IFEv} \\
\midrule
llada2-mini & 0.95 & Score & 71.48 & 28.79 & 85.77 & 81.62 & 70.75 & 83.84 & 85.37 & 31.50 & 92.57 & 89.00 & 59.85 & 77.26 \\
(baseline) & & AFP & 3.34 & 5.98 & 2.13 & 1.96 & 3.09 & 3.31 & 3.37 & 4.28 & 2.61 & 3.70 & 4.56 & 1.76 \\
& & TPS & -- & 336.1 & -- & -- & 589.8 & 731.9 & -- & 636.4 & -- & -- & -- & -- \\
\midrule
+ Edit & m2t:0.8 & Score & \textbf{73.27} & 48.48 & 82.66 & 81.27 & 73.38 & 78.92 & 85.98 & 24.89 & 92.80 & 93.80 & 67.85 & 75.97 \\
(Forward=2) & t2t:0.3 & AFP & 1.05 & 0.81 & 1.29 & 0.77 & 1.06 & 1.27 & 1.40 & 1.15 & 0.98 & 1.16 & 0.99 & 0.65 \\
\midrule
+ Edit & m2t:0.5 & Score & 68.89 & 32.83 & 83.93 & 78.74 & 68.62 & 78.45 & 82.32 & 27.09 & 92.04 & 89.02 & 59.85 & 64.88 \\
(Forward=1) & t2t:0.0 & AFP & \textbf{6.03} & 4.84 & 4.67 & 5.39 & 5.89 & 8.83 & 7.57 & 6.45 & 5.92 & 7.06 & 6.85 & 2.91 \\
& & TPS & -- & \textbf{954.0} & -- & -- & \textbf{1302.0} & \textbf{1161.7} & -- & \textbf{1121.5} & -- & -- & -- & -- \\
\midrule
+ Edit & m2t:0.5 & Score & 72.60 & 46.97 & 83.66 & 81.13 & 71.12 & 78.22 & 83.54 & 25.77 & 93.18 & 92.60 & 66.22 & 76.16 \\
(Forward=1) & t2t:0.3 & AFP & 3.20 & 2.17 & 3.63 & 2.34 & 3.10 & 4.27 & 4.41 & 3.72 & 3.05 & 3.77 & 3.10 & 1.62 \\
\bottomrule
\end{tabular}
}
\end{table*}
\paragraph{Training (Two-stage Loss)}
The model is trained using a hybrid objective to internalize editing capabilities:
\begin{itemize}
    \item \textbf{Stage 1 (Hybrid Denoising):} Given a prompt $x$ and a target sequence $y$ (where $y$ is the ground truth), we assume a partially decoded state where the first $t$ positions have been generated (e.g., $\hat{y}_{<t}$). The model simultaneously optimizes:
    \begin{itemize}
        \item \textbf{Mask-to-Token (M2T) Loss:} Predicting the subsequent tokens $y_{\geq t}$ given the context.
        \item \textbf{Token-to-Token (T2T) Loss:} Refining the already decoded but potentially noisy tokens $\hat{y}_{<t}$ to match the ground truth $y_{<t}$. This constitutes the \textit{Editing Loss}.
    \end{itemize}
    \item \textbf{Stage 2 (Teacher Forcing Refinement):} Utilizing teacher forcing, the model is fed slightly perturbed versions of correct tokens to practice sequence-wide refinement, ensuring robust error correction during iterative decoding.
\end{itemize}

\paragraph{Inference Strategies}
We define two primary inference modes for the editing paradigm:
\begin{itemize}
    \item \textbf{Forward=1 (Parallel Edit):} In a single forward pass, the model simultaneously predicts new tokens and refines previously generated ones (i.e., $\text{Forward}(y_{<t}, \text{[MASK]}) \rightarrow \{y'_{<t}, y_t\}$).
    \item \textbf{Forward=2 (Sequential Edit):} A two-step process where the model first decodes new tokens and then performs a dedicated second forward pass to refine the entire sequence (i.e., Step 1: $y_{<t} \rightarrow y_t$; Step 2: $\{y_{<t}, y_t\} \rightarrow y'_{1 \dots t}$).
\end{itemize}

\subsubsection{Experimental Setup}
\paragraph{Dataset} We utilized a high-quality instruction-following dataset comprising 308,000 samples across logic, code, and general tasks, distilled from state-of-the-art models such as GPT-4.
\paragraph{Infrastructure} Training was conducted on NVIDIA H200 GPUs. Inference performance was measured using the SGLang framework, which was specifically optimized for the LLaDA architecture.

\subsubsection{Results Overview}

Table~\ref{tab:editing_results} summarizes the performance across multiple benchmarks. We report the Accuracy (ACC), Average Finalization Parallelism (AFP)---a hardware-agnostic speed metric---and Tokens Per Second (TPS).

\subsubsection{Experimental Analysis}

The empirical data reveals three distinct operational regimes for the editing-enhanced MDLM:
\begin{enumerate}[label=\alph*)]
    \item \textbf{Quality-Priority Mode:} Utilizing high thresholds ($m2t=0.8, t2t=0.3$) in \textit{Forward=2} mode yields an average score of 73.27, significantly outperforming the baseline. This confirms that multi-pass refinement can recover information loss, albeit at the expense of lower AFP and inference speed.
    \item \textbf{Balanced Mode:} With moderate thresholds ($m2t=0.5, t2t=0.3$) and \textit{Forward=1} execution, the model maintains superior accuracy (72.60) compared to the baseline while preserving near-identical speed (AFP $\approx$ 3.20).
    \item \textbf{Speed-Priority Mode:} By relaxing thresholds ($m2t=0.5, t2t=0$), the model achieves nearly double the throughput of the baseline in terms of both AFP (6.03 vs. 3.34) and TPS (e.g., 1302.0 vs. 589.8 on CruO).
\end{enumerate}

\subsubsection{Theoretical Alignment}

These results provide preliminary validation for the editing theory proposed in Section 6 and Appendix A.8. Our editing approach effectively compensates for the local causal probability loss in parallel decoding discussed in Section 4.3 through iterative correction. Furthermore, our experiments demonstrate that the integration of editing capabilities can enhance parallelism and significantly increase decoding throughput.

\clearpage
\onecolumn
\subsection{Editing Potential: Theory and Proof}
\label{app:theory_proof}

% Put theorem environments in the preamble in the final version if possible.
\newtheorem{definition}{Definition}
\newtheorem{assumption}{Assumption}
\newtheorem{theorem}{Theorem}[section]
\newtheorem{lemma}[theorem]{Lemma}
\newtheorem{corollary}[theorem]{Corollary}

\subsubsection{Parallel Factorization Induces a Local Dependency Gap}

Let $\mathbf{y}=(y_1,\ldots,y_L)\in\mathcal{V}^L$ be a length-$L$ sequence and $\mathbf{x}$ be the input context.
Let $P^*(\mathbf{y}\mid \mathbf{x})$ denote the (unknown) data distribution.
In one-step parallel decoding (or one denoising/editing step in a masked diffusion language model), a common modeling choice is to predict a set of tokens $S\subseteq [L]$ in parallel under a conditional independence assumption:
\begin{equation}
Q_\theta(\mathbf{y}_S \mid \mathbf{x}, \mathbf{c})
\;=\;
\prod_{i\in S} q_\theta(y_i \mid \mathbf{x}, \mathbf{c}),
\label{eq:parallel_factorization}
\end{equation}
where $\mathbf{c}$ denotes the available conditioning signal in that step (e.g., a partially noised sequence or a previous iterate in an iterative editing procedure).

The following lemma formalizes the \emph{local dependency gap} introduced by the parallel factorization in~\eqref{eq:parallel_factorization}.

\begin{definition}[Conditional Total Correlation]
Given a random vector $\mathbf{Y}_S=\{Y_i\}_{i\in S}$ and conditioning variables $(\mathbf{X},\mathbf{C})$, the conditional total correlation is
\begin{equation}
\mathrm{TC}(\mathbf{Y}_S \mid \mathbf{X},\mathbf{C})
\;=\;
\sum_{i\in S} H(Y_i \mid \mathbf{X},\mathbf{C}) \;-\; H(\mathbf{Y}_S \mid \mathbf{X},\mathbf{C}),
\label{eq:tc_def}
\end{equation}
which is nonnegative and equals $0$ iff $\{Y_i\}_{i\in S}$ are conditionally independent given $(\mathbf{X},\mathbf{C})$.
\end{definition}

\begin{lemma}[Factorization Gap Decomposition]
\label{lem:factorization_gap}
Fix any conditioning pair $(\mathbf{x},\mathbf{c})$ and any index set $S\subseteq[L]$.
Let $P^*(\mathbf{y}_S\mid \mathbf{x},\mathbf{c})$ be the true conditional distribution of the tokens in $S$.
For any product-form approximation $\prod_{i\in S} q_\theta(y_i\mid \mathbf{x},\mathbf{c})$, we have the exact decomposition
\begin{align}
&\mathrm{KL}\!\left(
P^*(\mathbf{Y}_S \mid \mathbf{x},\mathbf{c})
\;\bigg\|\;
\prod_{i\in S} q_\theta(Y_i \mid \mathbf{x},\mathbf{c})
\right)
\label{eq:gap_decomp}
\\
&\qquad=
\underbrace{\mathrm{TC}(\mathbf{Y}_S \mid \mathbf{x},\mathbf{c})}_{\text{dependency discarded by factorization}}
\;+\;
\underbrace{\sum_{i\in S}
\mathrm{KL}\!\left(
P^*(Y_i \mid \mathbf{x},\mathbf{c})
\;\big\|\;
q_\theta(Y_i \mid \mathbf{x},\mathbf{c})
\right)}_{\text{marginal/conditional modeling error}}.
\nonumber
\end{align}
Consequently,
\begin{equation}
\mathrm{KL}\!\left(
P^*(\mathbf{Y}_S \mid \mathbf{x},\mathbf{c})
\;\bigg\|\;
\prod_{i\in S} q_\theta(Y_i \mid \mathbf{x},\mathbf{c})
\right)
\;\ge\;
\mathrm{TC}(\mathbf{Y}_S \mid \mathbf{x},\mathbf{c}).
\label{eq:gap_lower_bound_tc}
\end{equation}
\end{lemma}

\begin{proof}
Write $P_i^*(y_i)\triangleq P^*(y_i\mid \mathbf{x},\mathbf{c})$ and $q_i(y_i)\triangleq q_\theta(y_i\mid \mathbf{x},\mathbf{c})$ for brevity.
Add and subtract $\log\prod_{i\in S} P_i^*(Y_i)$ inside the KL:
\begin{align*}
\mathrm{KL}\!\left(P^*(\mathbf{Y}_S)\,\big\|\,\prod_i q_i\right)
&= \mathbb{E}_{P^*}\!\left[
\log \frac{P^*(\mathbf{Y}_S)}{\prod_i q_i(Y_i)}
\right]\\
&= \mathbb{E}_{P^*}\!\left[
\log \frac{P^*(\mathbf{Y}_S)}{\prod_i P_i^*(Y_i)}
\right]
+
\mathbb{E}_{P^*}\!\left[
\log \frac{\prod_i P_i^*(Y_i)}{\prod_i q_i(Y_i)}
\right]\\
&= \mathrm{KL}\!\left(P^*(\mathbf{Y}_S)\,\big\|\,\prod_i P_i^*\right)
+ \sum_i \mathrm{KL}\!\left(P_i^* \,\|\, q_i\right).
\end{align*}
The first term equals the conditional total correlation $\mathrm{TC}(\mathbf{Y}_S\mid \mathbf{x},\mathbf{c})$ by expanding entropies in~\eqref{eq:tc_def}, yielding~\eqref{eq:gap_decomp}.
Nonnegativity of KL gives~\eqref{eq:gap_lower_bound_tc}.
\end{proof}

\paragraph{Interpretation.}
Lemma~\ref{lem:factorization_gap} makes precise what is ``lost'' in a single parallel step:
even if each per-token predictor $q_\theta(\cdot\mid \mathbf{x},\mathbf{c})$ matches the correct marginal $P^*(\cdot\mid \mathbf{x},\mathbf{c})$, the product form in~\eqref{eq:parallel_factorization} still incurs an irreducible error equal to the conditional total correlation, i.e., the local dependency information among tokens in $S$ under the same condition.

\subsubsection{Masked Diffusion Decoding as a Thresholded Editing Markov Chain}

We now formalize the inference mechanism of a masked diffusion language model (MDLM) as an iterative editing process.
At each round, the model predicts a distribution at \emph{every} position and computes a confidence score.
A threshold rule selects a subset of positions to (re-)edit; positions not selected remain unchanged.

\paragraph{Per-position confidence and edit set.}
Let $\mathbf{Y}^{(k)}\in\mathcal V^L$ denote the current iterate at round $k$.
Given $(\mathbf{Y}^{(k)},\mathbf x)$, the model outputs per-site distributions $q_\theta(\cdot\mid \mathbf{Y}^{(k)},\mathbf x)$.
Define confidence at position $i$ by (one common choice)
\begin{equation}
s_i(\mathbf y,\mathbf x) \triangleq \max_{v\in\mathcal V} q_\theta(v\mid \mathbf y,\mathbf x),
\end{equation}
and the (re-)editing set induced by an editing threshold $\tau\in(0,1)$ as
\begin{equation}
S_\tau(\mathbf y,\mathbf x) \triangleq \left\{i\in[L]: s_i(\mathbf y,\mathbf x)\ge \tau\right\}.
\label{eq:edit_set}
\end{equation}
This matches the practical rule: each round computes all confidences and re-edits tokens whose confidence passes the threshold; all other tokens keep their previous values.

\paragraph{Thresholded editing kernel.}
We define the MDLM iterative decoding chain via the transition kernel
\begin{equation}
\mathcal K_{\theta,\tau}(\mathbf y' \mid \mathbf y, \mathbf x)
=
\prod_{i\in S_\tau(\mathbf y,\mathbf x)} q_\theta(y'_i \mid \mathbf y, \mathbf x)
\cdot
\prod_{i\notin S_\tau(\mathbf y,\mathbf x)} \mathbb I[y'_i=y_i].
\label{eq:mdlm_kernel}
\end{equation}
Within one transition, updated coordinates are conditionally independent given $(\mathbf y,\mathbf x)$, but dependence can emerge \emph{across} iterations because each site conditions on the entire previous sequence $\mathbf y$, and the edited set $S_\tau(\mathbf y,\mathbf x)$ itself depends on $\mathbf y$.
Hence iterative editing trades ``intra-step'' coupling (not modeled within the product) for ``inter-step'' coupling induced by global conditioning and state-dependent edit selection.

\subsubsection{Contraction and Convergence (Dobrushin Condition for Thresholded Editing)}

We analyze convergence via a Dobrushin-type influence bound adapted to the state-dependent (thresholded) update rule.

For $i\neq j$, define the worst-case influence coefficient
\begin{equation}
A_{ij}
=
\sup_{\mathbf{y},\tilde{\mathbf{y}}:\ \mathbf{y}_{\setminus j}=\tilde{\mathbf{y}}_{\setminus j}}
\left\|
\mu_i(\cdot \mid \mathbf{y},\mathbf{x})
-
\mu_i(\cdot \mid \tilde{\mathbf{y}},\mathbf{x})
\right\|_{\mathrm{TV}},
\label{eq:dobrushin_A}
\end{equation}
where $\mu_i(\cdot \mid \mathbf{y},\mathbf{x})$ is the one-step marginal update law at site $i$ induced by~\eqref{eq:mdlm_kernel}:
\begin{equation}
\mu_i(\cdot\mid \mathbf y,\mathbf x)=
\begin{cases}
q_\theta(\cdot\mid \mathbf y,\mathbf x), & i\in S_\tau(\mathbf y,\mathbf x),\\
\delta_{y_i}(\cdot), & i\notin S_\tau(\mathbf y,\mathbf x).
\end{cases}
\label{eq:marginal_update}
\end{equation}

\begin{assumption}[Dobrushin Uniqueness / Weak Dependence]
\label{assump:dobrushin}
Let
\begin{equation}
\alpha \;\triangleq\; \max_{i\in[L]} \sum_{j\neq i} A_{ij} \;<\; 1.
\label{eq:alpha_def}
\end{equation}
\end{assumption}

\begin{theorem}[Geometric Mixing of Thresholded Editing]
\label{thm:mixing}
Assume Assumption~\ref{assump:dobrushin} holds.
Then the Markov operator induced by $\mathcal{K}_{\theta,\tau}$ is a contraction in total variation and the chain admits a unique stationary distribution $Q_{\theta,\tau}^\infty(\cdot\mid \mathbf{x})$.
Moreover, for any initial distribution $Q^{(0)}$ over $\mathcal{V}^L$, if $Q^{(k)}$ denotes the law of $\mathbf{Y}^{(k)}$ given $\mathbf{x}$, then
\begin{equation}
\left\| Q^{(k)}(\cdot\mid \mathbf{x}) - Q_{\theta,\tau}^\infty(\cdot\mid \mathbf{x}) \right\|_{\mathrm{TV}}
\;\le\;
\alpha^k\,
\left\| Q^{(0)}(\cdot\mid \mathbf{x}) - Q_{\theta,\tau}^\infty(\cdot\mid \mathbf{x}) \right\|_{\mathrm{TV}}.
\label{eq:tv_contraction}
\end{equation}
\end{theorem}

\begin{proof}[Proof sketch]
By definition of $A_{ij}$, changing coordinate $j$ in the conditioning sequence (while holding all other coordinates fixed) can affect the next-step marginal law at coordinate $i$ by at most $A_{ij}$ in total variation, where the marginal law $\mu_i$ already accounts for both effects: whether $i$ is updated (via the thresholded set $S_\tau$) and, if updated, how its predictive distribution changes.
A standard Dobrushin coupling argument then yields a contraction of the joint transition in TV with coefficient at most $\alpha$.
Contraction implies existence and uniqueness of a stationary distribution and geometric convergence to it, proving~\eqref{eq:tv_contraction}.
\end{proof}

\subsubsection{When Does Editing Recover the Discarded Dependencies?}

Theorem~\ref{thm:mixing} guarantees convergence to a unique stationary distribution $Q_{\theta,\tau}^\infty$, but does not claim $Q_{\theta,\tau}^\infty=P^*$.
We now clarify what ``dependency recovery'' means in the MDLM setting.

\paragraph{Mechanism-level statement (non-factorized stationarity).}
A single parallel factorized step necessarily discards the dependency term $\mathrm{TC}(\mathbf{Y}_S\mid \mathbf x,\mathbf c)$ (Lemma~\ref{lem:factorization_gap}).
In contrast, the stationary distribution induced by repeated thresholded editing can be \emph{non-factorized}:
because each per-site predictor conditions on the full previous iterate $\mathbf y$ and because the edited set $S_\tau(\mathbf y,\mathbf x)$ depends on $\mathbf y$, the long-run joint law $Q_{\theta,\tau}^\infty(\cdot\mid\mathbf x)$ can exhibit nonzero multi-token dependence (and hence nonzero total correlation on subsets).
This provides the theoretical basis for ``editing potential'': dependencies discarded within one parallel step can be expressed across iterations via inter-step coupling.

\paragraph{Optional: sufficient conditions for exact recovery (strong compatibility).}
If one wants a formal guarantee that $Q_{\theta,\tau}^\infty=P^*$, one needs compatibility conditions stronger than ``per-site conditionals match'' because the kernel~\eqref{eq:mdlm_kernel} performs a \emph{state-dependent block update} with a product-form proposal.

One sufficient (but strong) condition is that for every state $\mathbf y$ and its edited set $S=S_\tau(\mathbf y,\mathbf x)$, the product of model conditionals equals the true block conditional:
\begin{assumption}[Block-conditional realizability (strong)]
\label{assump:block_realizable}
For all $\mathbf y$ and $S=S_\tau(\mathbf y,\mathbf x)$,
\begin{equation}
\prod_{i\in S} q_\theta(y_i \mid \mathbf y,\mathbf x)
\;=\;
P^*(\mathbf y_S \mid \mathbf y_{\setminus S}, \mathbf x).
\label{eq:block_realizable}
\end{equation}
\end{assumption}
Under Assumption~\ref{assump:block_realizable}, one can show $P^*(\cdot\mid \mathbf x)$ is invariant for $\mathcal K_{\theta,\tau}$, and thus (by uniqueness from Theorem~\ref{thm:mixing}) $Q_{\theta,\tau}^\infty=P^*$.
We state this as a corollary.

\begin{corollary}[Exact recovery under strong block-conditional realizability]
\label{cor:exact_recovery}
Assume Assumption~\ref{assump:dobrushin} and Assumption~\ref{assump:block_realizable} hold.
Then $P^*(\cdot\mid \mathbf x)$ is invariant for $\mathcal K_{\theta,\tau}$ and
\begin{equation}
Q_{\theta,\tau}^\infty(\cdot\mid \mathbf x)=P^*(\cdot\mid \mathbf x).
\end{equation}
\end{corollary}

\begin{proof}[Proof sketch]
Fix any $\mathbf y$ and let $S=S_\tau(\mathbf y,\mathbf x)$.
By~\eqref{eq:mdlm_kernel}, $\mathbf Y'_{\setminus S}=\mathbf y_{\setminus S}$ deterministically and $\mathbf Y'_S$ is sampled from $\prod_{i\in S} q_\theta(\cdot\mid \mathbf y,\mathbf x)$.
Under Assumption~\ref{assump:block_realizable}, this equals sampling from the true block conditional $P^*(\mathbf Y_S\mid \mathbf Y_{\setminus S}=\mathbf y_{\setminus S},\mathbf x)$.
Therefore, for $\mathbf Y\sim P^*(\cdot\mid\mathbf x)$, the one-step update preserves the law: $\mathbf Y'\sim P^*(\cdot\mid\mathbf x)$, i.e., $P^*$ is invariant.
Uniqueness of the stationary distribution under Assumption~\ref{assump:dobrushin} then implies $Q_{\theta,\tau}^\infty=P^*$.
\end{proof}

\paragraph{Takeaway for ``editing potential''.}
Lemma~\ref{lem:factorization_gap} shows a one-shot parallel step inevitably drops the conditional dependency term $\mathrm{TC}$ unless the model represents joint coupling within the step.
MDLM-style iterative thresholded editing instead repeatedly conditions on the entire previous iterate and selectively updates only confident positions, inducing an inter-step coupling mechanism.
Under weak-dependence conditions (Assumption~\ref{assump:dobrushin}), the process converges geometrically to a unique stationary distribution (Theorem~\ref{thm:mixing}), which in general is non-factorized and can therefore express cross-token dependencies.
Exact recovery of $P^*$ additionally requires strong block-conditional compatibility (Corollary~\ref{cor:exact_recovery}).

\subsubsection{Parallelism--Editing Trade-off and No-Slowdown Condition}

We now connect iterative editing to \emph{computational throughput}.
Let $m$ denote the number of sequential decoding stages (or blocks) needed by an architecture to produce a length-$L$ output.
For an autoregressive decoder, $m=L$; for a $B$-block non-autoregressive/partially-parallel decoder, $m=B$; for one-step parallel decoding, $m=1$.
We refer to $L/m$ as an effective \emph{parallelism level}.

\paragraph{Runtime model.}
Let $T_{\mathrm{step}}(m)$ be the wall-clock time of one stage (one forward pass for that stage).
We assume the natural monotonicity:
\begin{assumption}[Per-stage cost does not increase with parallelism]
\label{assump:step_cost}
$T_{\mathrm{step}}(m)$ is nonincreasing in $m$, i.e., fewer sequential stages (higher parallelism) does not make each stage slower:
\begin{equation}
m_1 \le m_2 \;\Rightarrow\; T_{\mathrm{step}}(m_1) \le T_{\mathrm{step}}(m_2).
\end{equation}
\end{assumption}

\paragraph{Editing complexity.}
Consider the thresholded editing chain with kernel $\mathcal{K}_{\theta,\tau}$ in~\eqref{eq:mdlm_kernel}.
Given a target accuracy $\delta>0$ in total variation to the stationary distribution, define the \emph{mixing time}
\begin{equation}
K(\delta) \triangleq \min\Big\{k:\ \sup_{Q^{(0)}} \|Q^{(k)}(\cdot|\mathbf{x})-Q_{\theta,\tau}^\infty(\cdot|\mathbf{x})\|_{\mathrm{TV}} \le \delta \Big\}.
\end{equation}

By Theorem~\ref{thm:mixing}, under Assumption~\ref{assump:dobrushin},
\begin{equation}
K(\delta)
\;\le\; 
\left\lceil \frac{\log\left(\sup_{Q^{(0)}}\|Q^{(0)}-Q_{\theta,\tau}^\infty\|_{\mathrm{TV}}/\delta\right)}{\log(1/\alpha^{-1})} \right\rceil
=
O\!\left(\frac{\log(1/\delta)}{1-\alpha}\right).
\label{eq:mixing_time_bound}
\end{equation}

\paragraph{End-to-end runtime.}
Define the total runtime of producing one sample with editing as
\begin{equation}
T_{\mathrm{edit}}(\delta;m) \triangleq K(\delta;m)\,T_{\mathrm{step}}(m),
\end{equation}
where $K(\delta;m)$ allows the contraction coefficient $\alpha$ (hence mixing) to depend on the degree of factorization/parallelism $m$ and on the thresholding policy (which affects the update/freeze pattern).

We compare against a baseline architecture with sequential stages $m_0$ and no editing ($K\equiv 1$), runtime
\begin{equation}
T_{\mathrm{base}} \triangleq m_0\,T_{\mathrm{step}}(m_0).
\end{equation}

\begin{theorem}[No-slowdown condition]
\label{thm:no_slowdown}
Assume Assumption~\ref{assump:step_cost} and Dobrushin contraction with coefficient $\alpha(m)<1$ for the thresholded editing kernel at parallelism level $m$.
Then achieving TV error $\delta$ to stationarity costs
\begin{equation}
T_{\mathrm{edit}}(\delta;m)
\;\le\;
T_{\mathrm{step}}(m)\cdot
O\!\left(\frac{\log(1/\delta)}{1-\alpha(m)}\right).
\end{equation}
In particular, iterative editing with higher parallelism $m<m_0$ is \emph{not slower} than the baseline whenever
\begin{equation}
\frac{\log(1/\delta)}{1-\alpha(m)}\,T_{\mathrm{step}}(m)
\;\le\; m_0\,T_{\mathrm{step}}(m_0).
\label{eq:no_slowdown_cond}
\end{equation}
\end{theorem}

\begin{proof}
The mixing-time bound follows directly from~\eqref{eq:tv_contraction} by solving for the smallest $k$ such that $\alpha(m)^k\le \delta/\sup_{Q^{(0)}}\|Q^{(0)}-Q_{\theta,\tau}^\infty\|_{\mathrm{TV}}$.
Multiplying by the per-step time yields the runtime upper bound.
Condition~\eqref{eq:no_slowdown_cond} is immediate by comparison with $T_{\mathrm{base}}$.
\end{proof}

\paragraph{Interpretation.}
Equation~\eqref{eq:no_slowdown_cond} makes explicit the trade-off:
higher parallelism reduces $T_{\mathrm{step}}(m)$, while potentially increasing the dependence strength (worsening $\alpha(m)$) and thus increasing the number of editing rounds.
As long as the deterioration in $(1-\alpha(m))^{-1}$ is dominated by the gain in per-step throughput, the edited highly-parallel model matches or exceeds the baseline speed.
Moreover, since $K(\delta)$ grows only logarithmically in $1/\delta$ under geometric contraction, modest numbers of editing rounds can suffice in practice.

\clearpage
\twocolumn

%-------------------------------------------------------
\end{document}